\documentclass[11pt]{article}
\usepackage[utf8]{inputenc}
\usepackage[authoryear,round,sort&compress]{natbib}
\usepackage[table]{xcolor}
\usepackage{tikz,amsmath,amssymb,pgfplots,enumerate,authblk,amsfonts,amsthm}
\usepackage{mathtools}
\usepackage{bbm}
\usepackage[ruled,linesnumbered]{algorithm2e}
\usepackage{subcaption}
\definecolor{linkcol}{rgb}{0,0,0.55}
\definecolor{citecol}{rgb}{0,0.45,0}
\definecolor{urlcol}{rgb}{0.55,0,0}
\usepackage[colorlinks=true,linkcolor=linkcol,citecolor=citecol,urlcolor=urlcol]{hyperref}
\usepackage{graphicx}
\usepackage[margin=1in]{geometry}
\usepackage{libertinus}
\usepackage[libertine]{newtxmath}
\usepackage[T1]{fontenc}
\usepackage{booktabs}
\usepackage{multirow}

\pgfplotsset{compat=1.18}
\usetikzlibrary{decorations.pathreplacing}
\usetikzlibrary{shapes.geometric, arrows, arrows.meta}
\usetikzlibrary{arrows,decorations.pathmorphing,backgrounds,calc,math}
\usetikzlibrary{positioning}

\tikzstyle{map}=[->,semithick]
\tikzstyle{arc}=[bend left,->,semithick]
\tikzstyle{rinclusion}=[right hook->,semithick]
\tikzstyle{linclusion}=[left hook->,semithick]
\tikzstyle{arrow} = [thick,->,>=stealth]

\newcommand{\xmark}{\ensuremath{\times}}

\def\R{{\mathbb R}}

\def\E{{\mathbb E}}

\newcommand{\D}[1]{\mathcal{D}_{{#1}}}
\newcommand{\db}{d_{\mathcal{B}}}

\DeclareMathOperator*{\diam}{diam}

\theoremstyle{plain}
\newtheorem{theorem}{Theorem}[section]
\newtheorem{remark}{Remark}[section]
\newtheorem{proposition}{Proposition}[section]
\newtheorem{corollary}{Corollary}[section]
\newtheorem{definition}{Definition}[section]
\newtheorem{lemma}{Lemma}[section]

\def\RR{{\mathbb R}}

\def\GG{{\mathbb G}}

\def\NN{{\mathbb N}}

\def\eps{{\varepsilon}}

\newcommand{\norm}[1]{\left\lVert #1 \right\rVert}

\newcommand{\LC}{\nu}                       
\newcommand{\rhonu}{\ensuremath{\lambda}}            
\newcommand{\hatrhonu}{\ensuremath{\widehat{\lambda}}}   

\colorlet{landmarkblue}{blue!65!black}
\colorlet{landmarkred}{red!75!black}
\colorlet{ballgray}{blue!12!white}
\colorlet{diagyellow}{yellow!70!orange}
\colorlet{pointgreen}{green!55!black}
\colorlet{coordorange}{orange!85!black}

\numberwithin{equation}{section}

\title{A Closed-Form Adaptive-Landmark Kernel for Certified Point-Cloud and Graph Classification}

\author[1]{Sushovan Majhi}
\affil[1]{\small Data Science, George Washington University, USA
  (\texttt{s.majhi@gwu.edu})}
\author[2]{Atish Mitra}
\affil[2]{\small Department of Mathematical Sciences, Montana Technological University, USA
  (\texttt{amitra@mtech.edu})}
\author[3]{\v{Z}iga Virk}
\affil[3]{\small Faculty of Computer and Information Science, University of Ljubljana, Slovenia
  (\texttt{ziga.virk@fri.uni-lj.si})}
\author[4]{Pramita Bagchi}
\affil[4]{\small Biostatistics and Bioinformatics, George Washington University, USA
  (\texttt{pramita.bagchi@gwu.edu})}

\date{}
\begin{document}

\maketitle

\begin{figure}[t]
\centering
\begin{tikzpicture}[
  scale=0.7,
  font=\footnotesize,
  panel/.style={font=\footnotesize\bfseries, text=black!85},
  axis/.style={-{Stealth[length=4pt,width=3.5pt]}, gray!60, line width=0.5pt},
  diag/.style={gray!45, dashed, line width=0.4pt},
  cover/.style={blue!55, fill=blue!10, fill opacity=0.30, draw opacity=0.55, line width=0.5pt},
  lmark/.style={fill=blue!70!black, opacity=0.55},
  datum/.style={fill=red!75!black},
  caption/.style={font=\scriptsize, text=gray!55!black, align=center},
  Lannot/.style={{Stealth[length=3pt]}-{Stealth[length=3pt]},
                 gray!65!black, line width=0.55pt},
  Dannot/.style={red!55!black, dashed, line width=0.5pt},
]

\def\W{5.6}      
\def\H{5.6}      


\begin{scope}[xshift=0cm]
  \draw[axis] (-0.15,0) -- (\W+0.15,0) node[right]{$b$};
  \draw[axis] (0,-0.15) -- (0,\H+0.15) node[above]{$d$};
  \draw[diag] (0,0) -- (\W,\W);
  \node[caption, rotate=45] at (3.0, 2.7) {$d{=}b$};

  \foreach \cx/\cy in {%
       0.7/2.8, 0.7/4.2, 0.7/5.6,
       2.1/4.2, 2.1/5.6,
       3.5/5.6}
    {\draw[cover] (\cx-1.05,\cy-1.05) rectangle (\cx+1.05,\cy+1.05);}
  \foreach \cx/\cy in {%
       0.7/2.8, 0.7/4.2, 0.7/5.6,
       2.1/4.2, 2.1/5.6,
       3.5/5.6}
    {\fill[lmark] (\cx,\cy) circle (1.6pt);}

  \foreach \dx/\dy in {2.45/3.55, 2.85/4.05, 3.15/3.90, 2.65/3.95, 3.00/3.65}
    {\fill[datum] (\dx,\dy) circle (1.8pt);}

  \draw[Lannot] (0.3, \W+1.05) -- (\W-0.3, \W+1.05);
  \node[anchor=south, font=\footnotesize, gray!65!black]
    at (\W/2, \W+1.05) {$L$};

  \draw[Dannot] (2.40, 3.50) rectangle (3.20, 4.10);
  \node[anchor=west, font=\scriptsize, red!55!black]
    at (3.25, 3.80) {$D$};

  \node[panel] at (\W/2, \H+2.0) {(a) Uniform grid $\GG_R$};
\end{scope}

\begin{scope}[xshift=10cm]
  \draw[axis] (-0.15,0) -- (\W+0.15,0) node[right]{$b$};
  \draw[axis] (0,-0.15) -- (0,\H+0.15) node[above]{$d$};
  \draw[diag] (0,0) -- (\W,\W);
  \node[caption, rotate=45] at (3.0, 2.7) {$d{=}b$};

  \foreach \cx/\cy/\r in {%
       2.45/3.55/0.25,
       2.85/4.05/0.30,
       3.15/3.90/0.30}
    {\draw[cover] (\cx-\r,\cy-\r) rectangle (\cx+\r,\cy+\r);}

  \foreach \cx/\cy in {2.45/3.55, 2.85/4.05, 3.15/3.90}
    {\fill[blue!75!black, opacity=0.45] (\cx,\cy) circle (3.2pt);}

  \foreach \dx/\dy in {2.45/3.55, 2.85/4.05, 3.15/3.90, 2.65/3.95, 3.00/3.65}
    {\fill[datum] (\dx,\dy) circle (1.8pt);}

  \draw[Lannot] (0.3, \W+1.05) -- (\W-0.3, \W+1.05);
  \node[anchor=south, font=\footnotesize, gray!65!black]
    at (\W/2, \W+1.05) {$L$};

  \draw[Dannot] (2.40, 3.50) rectangle (3.20, 4.10);
  \node[anchor=west, font=\scriptsize, red!55!black]
    at (3.25, 3.80) {$D$};

  \node[panel] at (\W/2, \H+2.0) {(b) Adaptive configuration $\LC$};
\end{scope}

\end{tikzpicture}
\caption{%
  \textbf{Uniform vs.\ data-adaptive landmark configurations on the
  birth--death plane.} Both panels show the same five-point data
  cluster (red) of bottleneck-diameter $D$ in a domain of extent
  $L$; here $L/D \approx 7$.
  \textbf{(a)}~PLACE's uniform grid $\GG_R^+$ covers the
  full half-plane $[0,L]^2$ with balls of fixed radius
  $\tfrac{3R}{2}$, irrespective of where data lies; most balls
  fall in empty regions.
  \textbf{(b)}~PALACE's adaptive configuration $\LC$ places
  $K$ landmarks via class-aware farthest-point sampling on
  training diagrams (here $K{=}3$ selected from the $5$
  training points) with radii from local nearest-neighbor
  spacing, concentrating coverage where data lives. The
  required budget drops from $\Theta((L/\tau)^2)$ to
  $\Theta((D/\tau)^2)$ when data diameter $D \ll L$
  (Theorem~\ref{thm:comparison}).%
}
\label{fig:uniform_vs_learned}
\end{figure}
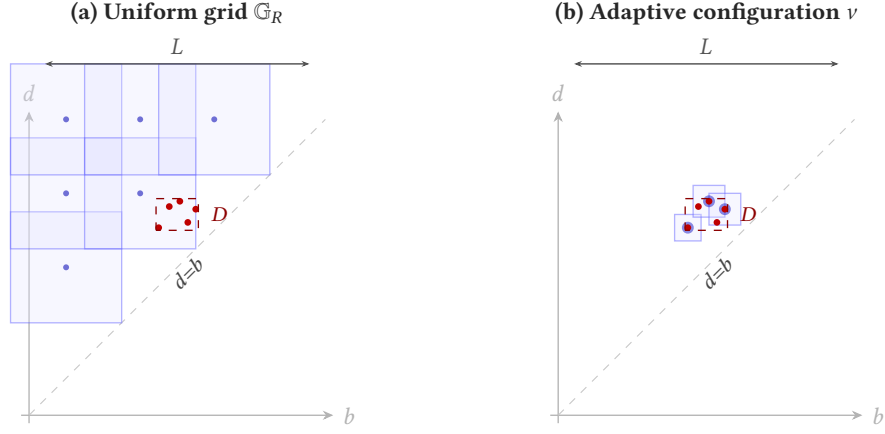

\begin{abstract}
We introduce \textbf{PALACE}
(\textbf{P}ersistence \textbf{A}daptive-\textbf{L}andmark
\textbf{A}nalytic \textbf{C}lassification \textbf{E}ngine), the
data-adaptive companion to the closed-form PLACE pipeline,
paying a small cross-validation tier on three knobs (budget,
radii, bandwidth; $\leq 5$ choices each).  The
summation embedding lifts into an RKHS via an additive landmark
kernel.  A self-contained cover-theoretic core---a Lebesgue-number
criterion on the landmark cover---yields four closed-form
guarantees.
\emph{(i)}~A structural lower distortion bound $\rhonu(\tau;\LC)$
on $\D{n}$ under cross-diagram non-interference, with a $(D/L)^2$
budget reduction over the uniform grid when diagrams concentrate
(data diameter $D$ vs.\ domain extent $L$).
\emph{(ii)}~Equal landmark weights $w_k = K^{-1/2}$ maximizing the
certificate $\rhonu$ of~(i), and farthest-point-sampling
positions $2$-approximating the optimal $k$-center covering
radius; both derived from training labels alone, no gradient
training.
\emph{(iii)}~A kernel-RKHS classification rate
$O((k{-}1)\sqrt{K}/(\gamma\sqrt{m_{\min}}))$ ($k$ classes,
$K$ landmarks) with binary necessity threshold
$m = \Omega(\sqrt K/\gamma)$ from a matching Le~Cam lower bound;
and a closed-form filtration-selection rule.  The
kernel-Mahalanobis margin $\hat\rho_{\mathrm{Mah}}$ is the
strongest closed-form ranker across the chemical-graph pool,
positive on every benchmark (mean Spearman
$\rho \approx +0.60$); the isotropic surrogate
$\hat\gamma/\sqrt{K}$ admits a closed-form selection-consistency
rate, and $\hatrhonu$ from~(i) provides an independent
data-level signal complementing the kernel-margin rankers
(positive on COX2 and PTC).
\emph{(iv)}~A per-prediction certificate, in non-asymptotic
Pinelis and asymptotic Gaussian forms, with no calibration split.
Empirically, PALACE is the strongest closed-form diagram-based
method on Orbit5k ($\mathbf{91.3 \pm 1.0\%}$, matching
Persformer's gradient-trained black-box transformer), leads every
diagram-based competitor on COX2 and MUTAG, and is competitive on
DHFR (within $1$~pp of ECP); descriptor blindness persists on NCI1
and PTC.  At $8\times$ domain
inflation, adaptive placement maintains $94\%$ while the uniform
grid collapses to chance ($25\%$ on $4$-class data).
\end{abstract}

\section{Introduction}\label{sec:intro}

Persistent homology produces a canonical topological signature of
structured data---graphs, point clouds, shapes---called the
\emph{persistence diagram}: a finite multiset of points in the
half-plane above the diagonal, augmented by a formal diagonal
point $\ast$.
Stability under perturbation is well-understood
\citep{Cohen-Steiner2007,ChazalCohenSteiner2009,ChazalDeSilva2016},
but the varying cardinality and non-Hilbertian geometry of
diagrams make them incompatible with standard machine learning.
Existing vectorizations---persistence images~\citep{Adams2017},
landscapes~\citep{Bubenik15}, kernels~\citep{Kusano2016,
Carriere2017, Reininghaus2015, LeYamada2018}, learned
weights~\citep{yusu_metric_learning}, and neural
extensions~\citep{Hofer2017, Gabrielsson2020, Carriere2020,
Reinauer2021}---all offer Lipschitz \emph{upper} bounds on embedding
distortion.  None comes with a \emph{lower} bound with explicit
constants, so there is no guarantee that bottleneck-separated diagrams
remain separated after vectorization.
Each method further carries hyperparameters---kernel bandwidth,
image resolution, landscape level count, learned weight function---%
whose selection requires held-out data, so any downstream accuracy
claim inherits the dependence on a validation split.
Despite a decade of work, there is no way to inspect a trained
persistence-diagram classifier and certify, before seeing test
data, whether its predictions will be correct.

Our companion paper~\citep{PaperI} closes these gaps on a
\emph{fixed-grid} backbone; we adopt its persistence-diagram
setup throughout (bottleneck distance, $n$-point diagram space,
top-persistence filter; see \citealp[Sec.~2]{PaperI}).
PLACE places landmarks on a uniform lattice in the birth--death
plane at $N$ geometrically spaced scales and sums a compactly
supported hat coordinate over diagram points; from training
labels alone, the construction yields a tight minimax
classification rate, a closed-form Mahalanobis-margin
descriptor-selection rule, and a per-prediction correctness
certificate. The construction is fully tuning-free, but three
\emph{residual} trade-offs limit how far it can be pushed:
\emph{(i)~Coverage.} The uniform grid covers the full birth--death
domain regardless of where diagrams concentrate, inflating
embedding capacity relative to the data support.
\emph{(ii)~Positions.} Landmark positions are combinatorial and
change the embedding dimension in discrete jumps as the grid is
refined or thinned, leaving the placement axis outside
\citet{PaperI}'s closed-form recipe.
\emph{(iii)~Linear-only analysis.} The classification rate of
\citet[Thm.~3.1]{PaperI} is stated for a linear SVM on $\Phi$;
non-linear lifts (kernel SVM, RKHS lift) are not analyzed and
cannot be deployed without forfeiting the per-prediction
certificate (which requires an analytically fixed embedding).
PALACE relaxes all three.

This paper introduces \textbf{PALACE}, the data-adaptive companion
that addresses (i)--(iii) in turn.
The fixed grid is replaced by an adaptive configuration $\LC$ of
$K$ weighted landmarks placed by class-aware farthest-point
sampling on training diagrams (Figure~\ref{fig:uniform_vs_learned};
formal definition in Section~\ref{sec:cover}).
This makes landmark positions analytic from the data and
concentrates coverage where diagrams live, dropping the required
budget from $\Theta((L/\tau)^2)$ to $\Theta((D/\tau)^2)$ when
diagrams cluster in a region of $\db$-diameter $D \ll L$
(Theorem~\ref{thm:comparison})---addressing~(i) and~(ii).
The single-point coordinate
$\varphi_{p,r}$ and the bottleneck geometry are inherited from
\citet{Mitra2024} unchanged; we sum the coordinate over diagram
points---the summation diagonalization of \citet{PaperI}---%
evaluated at the adaptive positions, then lift the resulting
embedding into an RKHS via the additive landmark kernel $k_\LC$,
addressing~(iii). The kernel lift is empirically necessary: on
the same embedding, linear classifiers leave a structural
$30$-percentage-point gap (Section~\ref{par:kernel_comparison}).
The price is a small cross-validation tier
(budget $K$, radius factor $\alpha$, bandwidth $\sigma$; $\leq 5$
choices each) replacing PLACE's tuning-free regime.

The theory mirrors PLACE's contribution list under a single
correspondence: $R \leftrightarrow \sqrt{K}$,
$\Delta \leftrightarrow 2\gamma$, where $R$ is PLACE's embedding
radius, $\Delta$ its class-mean separation, $K$ the PALACE
landmark budget, and $\gamma$ the kernel-RKHS class-mean margin.
Under this map, PLACE's grid-tied constant-floor distortion bound
becomes PALACE's configuration-intrinsic certificate
$\rhonu(\tau;\LC)$ on arbitrary admissible $\LC$ via a
self-contained non-uniform \emph{cover theory}.
PLACE's $O((k{-}1)R/(\Delta\sqrt{m_{\min}}))$ classification rate
(with $(k{-}1)$ from the OvO majority-vote reduction and
$m_{\min}$ the smallest class size) becomes the kernel-RKHS
analogue $O((k{-}1)\sqrt{K}/(\gamma\sqrt{m_{\min}}))$, with the
same Le~Cam binary lower bound and the same polynomial
necessary-vs-sufficient gap (the RKHS lift introduces no
information-theoretic loss).
PLACE's $\Delta$-based per-prediction certificate carries through
on $\R^K$ at the raw-embedding class-mean separation
$\hat\Delta_{\hat c}$ (Theorem~\ref{thm:certified}), with
identical structure.
The one piece that does \emph{not} carry over is PLACE's
nested-scale weight rule, which optimizes a different problem;
the free-configuration argmax is the equal-weight rule
$w_k = K^{-1/2}$, $\rhonu$-maximizing among $\tau$-admissible
equal-weight configurations of cardinality $K$
(Prop.~\ref{prop:optimal_config}), with admissibility itself
preventing memorization (no gradient training of weights is
needed).
Section~\ref{sec:contributions} states each result formally.

Empirically, PALACE reaches $\mathbf{91.3 \pm 1.0\%}$ on Orbit5k
with the certified landmark kernel at a triple-filtration
concatenation (Table~\ref{tab:push92}), matching
Persformer~\citep{Reinauer2021} and surpassing every other
diagram-based method, including PLACE's $87.2 \pm 0.6\%$
(linear SVM, $\ell{=}1{,}366$).
On the structurally discriminative chemical benchmarks
(COX2, DHFR, MUTAG; Table~\ref{tab:graph_comparison}), PALACE
leads every diagram-based competitor on COX2 and MUTAG and is
competitive on DHFR (within $1$~pp of ECP); it exceeds both PLACE
and PersLay on each, with accuracies
$81.7\%$/$81.0\%$/$90.9\%$.
On the heterogeneous chemical pool of
Section~\ref{sec:exp_multi_dataset}, the Mahalanobis-margin
selector $\hat\rho_{\mathrm{Mah}}$ is the only ranker positive on every
completed dataset (MUTAG, COX2, DHFR, PTC, NCI1; mean Spearman
$\rho \approx +0.60$); the certificate-as-ranker $\hatrhonu$
gives an independent positive signal on COX2 and PTC where the
trace-corrected $\widehat{\mathrm{Fisher}}_{\mathrm{ker}}$
inverts (notably COX2), confirming the two mechanisms
(kernel-margin vs.\ data-level bottleneck) are orthogonal.
At $8\times$ domain inflation on a synthetic task, adaptive
placement maintains $94\%$ while the uniform grid collapses to
the $25\%$ chance level on $4$-class data, validating the
$(D/L)^2$ budget-reduction mechanism.
Section~\ref{sec:experiments} reports the full empirical
comparison across Orbit5k, five chemical graph benchmarks, and
the controlled synthetic task; on PROTEINS, DD, IMDB-B, IMDB-M,
and NCI109~\citep{yusu_metric_learning}, the linear-SVM
baseline, nearest-centroid accuracies, and certificate firing
diagnostics are in hand, with the LK-SVM headline accuracies
deferred to a future revision (Section~\ref{sec:discussion}).

\subsection{Our Contribution and Organization}
\label{sec:contributions}

Adaptive landmark placement on persistence diagrams admits a
\emph{closed-form theory (modulo a small CV tier)} for the four
steps of the classification pipeline---embedding, optimization,
selection, deployment---replacing learned vectorizations,
gradient training of weights, held-out validation, and post-hoc
calibration with provable choices.
PALACE realizes this theory; PLACE~\citep{PaperI} is the discrete
uniform-grid special case.
The single-point coordinate $\varphi_{p,r}$ and bottleneck
geometry $\db$ are inherited from \citet{Mitra2024}; the summation
diagonalization and $\Delta$-based certificate form are inherited
from \citet{PaperI}.

The four contributions, all closed-form modulo the CV tier above,
correspond to the four pipeline steps and parallel the four
contributions of \citet{PaperI}:
\emph{(i)}~A self-contained non-uniform cover theory for arbitrary
admissible landmark configurations $\LC$, yielding a structural
lower distortion bound $\rhonu(\tau;\LC)$ via a Lebesgue-number
criterion under cross-diagram non-interference
(Theorem~\ref{thm:nu_rho}) and a $(D/L)^2$ budget reduction over
the uniform grid (Theorem~\ref{thm:comparison},
Section~\ref{sec:cover}); this generalizes the constant-floor
lower bound of \citet[contribution~(i)]{PaperI} from the grid
$\GG_R$ to arbitrary admissible $\LC$. A per-dataset audit
(Section~\ref{sec:experiments}) finds the non-interference
hypothesis of Theorem~\ref{thm:nu_rho} essentially never met on
chemical persistence diagrams; Theorem~\ref{thm:nu_rho} should
therefore be read as a structural admissibility statement, with
the empirical workhorse living at the kernel-margin level
(contribution~(iii) below).
\emph{(ii)}~Closed-form configuration choices: equal weights
$w_k = K^{-1/2}$ maximize the certificate $\rhonu$ on the
worst-case kernel-RKHS bound
(Proposition~\ref{prop:optimal_config}(i)); the effective
certificate
is sharpened by the Lebesgue number, which under uniform radii
reduces to a $k$-center covering-radius minimization
(Proposition~\ref{prop:optimal_config}(ii)), solved within a factor
of $2$ by farthest-point sampling
(Theorem~\ref{thm:fps_greedy},
Corollary~\ref{cor:fps_admissible}).  Admissibility prevents
memorization
(Definition~\ref{def:sep_radius}(i) forces $\max_k r_k \geq \tau/4$,
ruling out the degenerate $r_k \to 0$ configurations a learned
placement could otherwise reach), so no gradient training of
$w_k$ or $\{p_k\}$ is needed.  PLACE's
nested-scale rule $w_k^2 \propto (d_{k+1}^2 - d_k^2)/R_k^2$ (part
of \citealp{PaperI}'s contribution~(i)) optimizes a different
problem (nested scales with fixed support size) and does not
carry over: the free-configuration argmax is equal weights.
\emph{(iii)}~A kernel-RKHS classification rate
$O((k{-}1)\sqrt{K}/(\gamma\sqrt{m_{\min}}))$ via Theorem~\ref{thm:data_dependent},
with binary necessity threshold $m = \Omega(\sqrt K/\gamma)$ from
Theorem~\ref{thm:lower_bound}'s Le~Cam two-point construction (the
polynomial gap between necessary $\sqrt K/\gamma$ and sufficient
$K\log(k/\delta)/\gamma^2$ thresholds is the same one \citet{PaperI} leaves open;
Remark~\ref{rem:gap_pii}).  Closed-form filtration selection
(Section~\ref{sec:gamma_stat}) then provides three
$\Sigma$-treatment selectors ($\hat\gamma/\sqrt K$,
$\widehat{\mathrm{Fisher}}_{\mathrm{ker}}$,
$\hat\rho_{\mathrm{Mah}}$) covering the spherical, scalar-trace,
and full-operator regimes (Remark~\ref{rem:selector_hierarchy};
the Mahalanobis pivot parallels \citealp[Sec.~4.1]{PaperI}), with
selection-consistency theorem
(Proposition~\ref{prop:gamma_selection_consistency},
paralleling \citealp[Prop.~4.4]{PaperI}); a complementary
data-level $\hatrhonu$ recovers signal in the
bottleneck-orthogonal regime where the $\Sigma$-aware selectors
miss (Section~\ref{sec:exp_multi_dataset}).  This unifies the
analogues of \citet{PaperI}'s contributions~(ii) (classification
rate) and~(iii) (descriptor selection) into a single RKHS
framework.
\emph{(iv)}~A per-prediction correctness certificate
$r_m < \tfrac{1}{2}\hat\Delta_{\hat c}$
(Theorem~\ref{thm:certified}, Section~\ref{sec:certified}) in
non-asymptotic Pinelis and asymptotic Gaussian forms, no
calibration split; this parallels \citealp{PaperI}'s
contribution~(iv) at the raw-embedding class-mean separation
$\Delta$ on $\R^K$ (the certificate operates on the raw $\ell^2$
embedding, not the RKHS lift, with raw embedding radius
$\bar R \leq N_{\max}\tau$ replacing PI's $R$ in the constants).

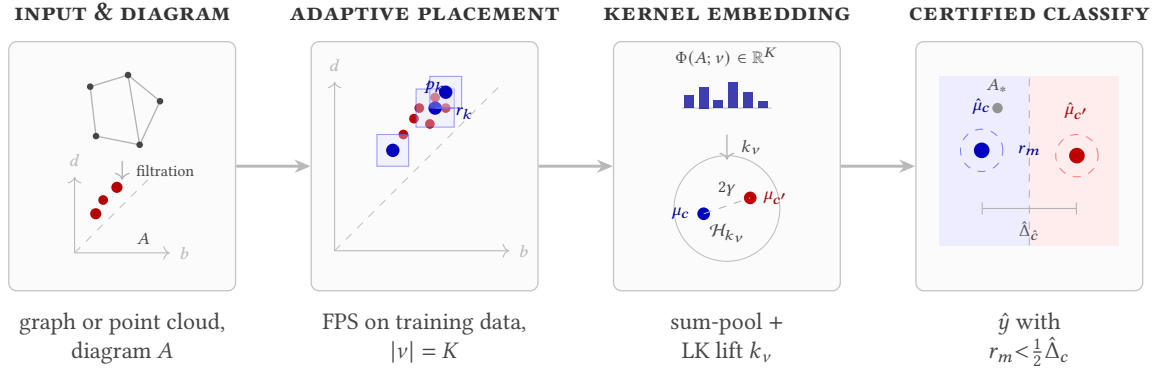
\begin{figure}[t]
\centering
\begin{tikzpicture}[
  font=\small,
  stage/.style={draw=gray!45, rounded corners=3pt, fill=gray!3,
                line width=0.5pt,
                minimum width=3.0cm, minimum height=3.3cm,
                inner sep=4pt},
  slabel/.style={font=\small\bfseries, text=black!80},
  sbelow/.style={font=\footnotesize, text=gray!55!black, align=center},
  flow/.style={-{Stealth[length=6pt,width=5pt]}, line width=1.1pt, gray!55},
]

\def\sxa{0}
\def\sxb{4}
\def\sxc{8}
\def\sxd{12}

\node[stage] (input) at (\sxa,0) {};
\node[stage] (place) at (\sxb,0) {};
\node[stage] (embed) at (\sxc,0) {};
\node[stage] (clf)   at (\sxd,0) {};

\begin{scope}[shift={(input)}, scale=1.4]
  \begin{scope}[shift={(0, 0.45)}, scale=0.55]
    \coordinate (n1) at (-0.55, 0.55);
    \coordinate (n2) at ( 0.10, 0.75);
    \coordinate (n3) at ( 0.65, 0.30);
    \coordinate (n4) at ( 0.30,-0.45);
    \coordinate (n5) at (-0.45,-0.30);
    \draw[gray!60, line width=0.5pt]
      (n1)--(n2) (n2)--(n3) (n3)--(n4)
      (n4)--(n5) (n5)--(n1) (n2)--(n4);
    \foreach \p in {n1,n2,n3,n4,n5}
      {\fill[gray!55!black] (\p) circle (1.6pt);}
  \end{scope}
  \draw[->, gray!60, line width=0.5pt] (0, 0.05) -- (0, -0.13);
  \node[font=\tiny, gray!55!black, anchor=west] at (0.04, -0.04)
    {filtration};
  \begin{scope}[shift={(0, -0.50)}]
    \draw[->, gray!55, line width=0.4pt]
      (-0.45,-0.32) -- (0.45,-0.32) node[right, font=\tiny] {$b$};
    \draw[->, gray!55, line width=0.4pt]
      (-0.45,-0.32) -- (-0.45, 0.40) node[above, font=\tiny] {$d$};
    \draw[gray!50, dashed, line width=0.3pt]
      (-0.45,-0.32) -- (0.25, 0.38);
    \fill[red!70!black] (-0.25, 0.05) circle (1.5pt);
    \fill[red!70!black] (-0.05, 0.30) circle (1.5pt);
    \fill[red!70!black] (-0.18, 0.18) circle (1.3pt);
    \node[font=\tiny, gray!55!black, anchor=west] at (0.05,-0.18)
      {$A$};
  \end{scope}
\end{scope}

\begin{scope}[shift={(place)}, scale=1.4]
  \draw[->, gray!55, line width=0.4pt]
    (-0.85,-0.80) -- (0.80,-0.80) node[right, font=\tiny] {$b$};
  \draw[->, gray!55, line width=0.4pt]
    (-0.85,-0.80) -- (-0.85, 0.80) node[above, font=\tiny] {$d$};
  \draw[gray!50, dashed, line width=0.3pt]
    (-0.85,-0.80) -- (0.75, 0.80);
  \foreach \x/\y in
    {-0.30/0.15, -0.20/0.30, -0.10/0.45, -0.05/0.55, 0.05/0.40,
      0.10/0.55, 0.10/0.65, 0.20/0.55, 0.20/0.70}
    {\fill[red!75!black] (\x, \y) circle (1.3pt);}
  \foreach \x/\y/\r in
    {-0.30/0.15/0.15, 0.10/0.55/0.18, 0.20/0.70/0.15}
    {\draw[blue!60, fill=blue!12, fill opacity=0.35, draw opacity=0.65,
       line width=0.4pt]
      (\x-\r, \y-\r) rectangle (\x+\r, \y+\r);
     \fill[blue!75!black] (\x, \y) circle (1.8pt);}
  \node[font=\tiny, blue!60!black, anchor=south]
    at (0.10, 0.60) {$p_k$};
  \draw[blue!55, line width=0.3pt] (0.10, 0.55) -- (0.28, 0.55);
  \node[font=\tiny, blue!60!black, anchor=west]
    at (0.20, 0.50) {$r_k$};
\end{scope}

\begin{scope}[shift={(embed)}, scale=1.4]
  \begin{scope}[shift={(0, 0.55)}]
    \node[font=\tiny, gray!55!black, anchor=south] at (0, 0.30)
      {$\Phi(A;\LC) \in \mathbb{R}^K$};
    \draw[gray!55, line width=0.3pt] (-0.42,0) -- (0.42,0);
    \foreach \i/\h in {0/0.12, 1/0.20, 2/0.07, 3/0.24, 4/0.15, 5/0.06}
      {\fill[blue!60!black, opacity=0.75]
        ({-0.36 + \i*0.14 - 0.045}, 0)
        rectangle ({-0.36 + \i*0.14 + 0.045}, \h);}
  \end{scope}
  \draw[->, gray!60, line width=0.5pt] (0, 0.30) -- (0, 0.06);
  \node[font=\tiny, gray!55!black, anchor=west] at (0.04, 0.18)
    {$k_\LC$};
  \begin{scope}[shift={(0, -0.40)}]
    \draw[gray!50, line width=0.4pt] (0, 0) circle (0.50);
    \fill[blue!75!black] (-0.22, -0.05) circle (1.7pt);
    \fill[red!75!black]  ( 0.22,  0.10) circle (1.7pt);
    \node[font=\tiny, blue!60!black, anchor=east]
      at (-0.25, -0.05) {$\mu_c$};
    \node[font=\tiny, red!65!black, anchor=west]
      at (0.25, 0.10) {$\mu_{c'}$};
    \draw[gray!60, line width=0.4pt, dashed]
      (-0.22, -0.05) -- (0.22, 0.10);
    \node[font=\tiny, gray!50!black, anchor=south]
      at (0, 0.02) {$2\gamma$};
    \node[font=\tiny, gray!55!black, anchor=south]
      at (0, -0.42) {$\mathcal{H}_{k_\LC}$};
  \end{scope}
\end{scope}

\begin{scope}[shift={(clf)}, scale=1.4]
  \fill[blue!7] (-0.85,-0.75) rectangle (0.0, 0.85);
  \fill[red!7]  ( 0.0,-0.75) rectangle (0.85, 0.85);
  \draw[gray!60, dashed, line width=0.4pt]
    (0.0,-0.75) -- (0.0, 0.85);
  \fill[blue!75!black] (-0.45, 0.15) circle (2.1pt);
  \fill[red!75!black]  ( 0.45, 0.10) circle (2.1pt);
  \draw[blue!50, dashed, line width=0.4pt]
    (-0.45, 0.15) circle (0.20);
  \draw[red!50,  dashed, line width=0.4pt]
    ( 0.45, 0.10) circle (0.20);
  \node[font=\tiny, blue!60!black, anchor=south]
    at (-0.45, 0.40) {$\hat\mu_c$};
  \node[font=\tiny, red!65!black, anchor=south]
    at ( 0.45, 0.35) {$\hat\mu_{c'}$};
  \node[font=\tiny, blue!55!black, anchor=west]
    at (-0.20, 0.15) {$r_m$};
  \draw[gray!60, line width=0.4pt, |-|]
    (-0.45,-0.40) -- (0.45,-0.40);
  \node[font=\tiny, gray!60!black, anchor=north]
    at (0,-0.42) {$\hat\Delta_{\hat c}$};
  \fill[gray!85] (-0.30, 0.55) circle (1.4pt);
  \node[font=\tiny, gray!60!black, anchor=south]
    at (-0.30, 0.60) {$A_*$};
\end{scope}

\node[slabel, above=4pt of input] {\textsc{input \& diagram}};
\node[slabel, above=4pt of place] {\textsc{adaptive placement}};
\node[slabel, above=4pt of embed] {\textsc{kernel embedding}};
\node[slabel, above=4pt of clf]   {\textsc{certified classify}};

\node[sbelow, below=4pt of input.south] {graph or point cloud,\\diagram $A$};
\node[sbelow, below=4pt of place.south] {FPS on training data,\\$|\LC|=K$};
\node[sbelow, below=4pt of embed.south] {sum-pool $+$\\LK lift $k_\LC$};
\node[sbelow, below=4pt of clf.south]   {$\hat y$ with\\$r_m{<}\tfrac{1}{2}\hat\Delta_c$};

\draw[flow] (input.east) -- (place.west);
\draw[flow] (place.east) -- (embed.west);
\draw[flow] (embed.east) -- (clf.west);

\end{tikzpicture}
\caption{\textbf{The PALACE pipeline.}
A graph or point cloud is converted to a persistence diagram $A$
through a filtration. Class-aware farthest-point sampling on
training diagrams fixes the landmark configuration
$\LC = \{(p_k, r_k, w_k)\}_{k=1}^K$. The diagram is
sum-pooled into $\Phi(A; \LC) \in \mathbb{R}^K$
and lifted to the RKHS $\mathcal{H}_{k_\LC}$ via the
additive landmark kernel $k_\LC$. The kernel-SVM
prediction $\hat y$ is audited per input by the certificate
$r_m < \tfrac{1}{2}\hat\Delta_c$
(Theorem~\ref{thm:certified}).
The middle two stages are PALACE-specific; stages 1 and 4 share
form with PLACE~\citep{PaperI}.}
\label{fig:pipeline}
\end{figure}

Section~\ref{sec:cover} develops the cover theory, the summation
embedding $\Phi$, the certificate $\rhonu(\tau;\LC)$, and the budget
reduction; Section~\ref{sec:ot} the kernel-RKHS framework and the
selection statistic; Section~\ref{sec:certified} the certified
nearest-centroid classifier; Section~\ref{sec:experiments} the
experiments; Section~\ref{sec:discussion} limitations.

\subsection{Related Work}\label{sec:related}

For a survey of persistence diagram vectorizations
(landscapes, persistence images, kernels, learned weightings,
neural extensions), certified machine learning (conformal
prediction, selective classification, learning with rejection),
and topological data analysis for classification
(diagram-based, neural-augmented, Euler-characteristic
methods), we refer the reader to \citet[Section~1.2]{PaperI}; PALACE
inherits PLACE's positioning relative to those literatures and
adds two new contact points discussed below.
Table~\ref{tab:palace_vs_prior} extends the feature matrix of
\citet[Table~1]{PaperI} with two columns specific to PALACE
(\emph{Adaptive}, \emph{Kernel/RKHS}); PLACE and PALACE remain
the only methods with both an explicit lower-distortion bound
and a per-prediction certificate, and PALACE adds adaptive
placement and the RKHS lift in exchange for a small CV tier.

\begin{table}[h]
\centering
\caption{Persistence-diagram vectorizations.
  \textbf{Lipschitz:} upper stability $\|\Phi(\cdot){-}\Phi(\cdot)\| \leq c_+\,\db$.
  \textbf{Lower dist.:} explicit constant $c_-$ in
  $c_-\,\db \leq \|\Phi(\cdot){-}\Phi(\cdot)\|$;
  ``config-intrinsic'' for PALACE means the constant depends only
  on $\LC$, not a specific grid.
  \textbf{Adaptive:} landmark positions/radii fixed analytically
  from training data (not a grid, not a learned optimizer).
  \textbf{Kernel/RKHS:} explicit positive-definite kernel with
  RKHS-level analysis; PALACE additionally provides non-degeneracy
  on $\tau$-separated pairs (Cor.~\ref{cor:nondegen}).
  \textbf{No-CV:} embedding hyperparameters fixed analytically
  (no held-out validation).
  \textbf{Cert.:} correctness certificate
  (metric / classification / per-prediction).
  PALACE trades No-CV for Adaptive vs.\ PLACE.\label{tab:palace_vs_prior}}
\small
\setlength{\tabcolsep}{4pt}
\resizebox{\textwidth}{!}{%
\begin{tabular}{lcccccc}
\toprule
\textbf{Method} & \textbf{Lipschitz} & \textbf{Lower dist.}
  & \textbf{Adaptive} & \textbf{Kernel/RKHS}
  & \textbf{No-CV} & \textbf{Cert.} \\
\midrule
Landscapes \citeyearpar{Bubenik15}
  & \checkmark &---& ---
  & {\scriptsize implicit}
  & \xmark~{\scriptsize(levels)} & \xmark \\
Persistence images \citeyearpar{Adams2017}
  & \checkmark &---& ---
  & \checkmark
  & \xmark~{\scriptsize($\sigma$, grid, weight)} & \xmark \\
SW\,/\,PSS kernels \citeyearpar{Kusano2016,Carriere2017}
  & \checkmark &---& ---
  & \checkmark
  & \xmark~{\scriptsize(bandwidth)} & \xmark \\
WKPI \citeyearpar{yusu_metric_learning}
  & \checkmark &---& \xmark~{\scriptsize(learned $w$)}
  & \checkmark
  & \xmark~{\scriptsize(learned $w$)} & \xmark \\
PersLay\,/\,Persformer \citeyearpar{Carriere2020,Reinauer2021}
  & {\scriptsize learned} &---& \xmark~{\scriptsize(learned)}
  & \xmark
  & \xmark~{\scriptsize(end-to-end)} & \xmark \\
Mitra--Virk $n$-fold \citeyearpar{Mitra2024}
  & \checkmark
  & \checkmark~{\scriptsize $\rho_-$ on $\D{n}$}
  & \xmark~{\scriptsize($NM^n$ grid)} & \xmark
  & \checkmark
  & {\scriptsize metric only} \\
Uniform PLACE \citeyearpar{PaperI}
  & \checkmark
  & \checkmark~{\scriptsize $\lambda(\nu)$ on $\D{n}$, constant floor}
  & \xmark~{\scriptsize(grid $\GG_R$)} & \xmark~{\scriptsize(linear SVM)}
  & \checkmark & {\scriptsize classification} \\
\rowcolor{blue!6}
\textbf{PALACE (this work)}
  & \checkmark
  & \checkmark~{\scriptsize $\rhonu$ on $\D{n}$, config-intrinsic}
  & \checkmark~{\scriptsize(FPS on data)}
  & \checkmark~{\scriptsize($k_\LC$, Cor.~\ref{cor:nondegen})}
  & \xmark~{\scriptsize($K, \alpha, \sigma$)}
  & {\scriptsize per-prediction} \\
\bottomrule
\end{tabular}%
}
\end{table}

\paragraph{Landmarks and coresets.}
Landmark-based embeddings are classical in manifold learning and
metric-space approximation (e.g., $k$-medoids embedding,
landmark MDS~\citep{DeSilvaTenenbaum2004}).
In kernel methods, Nystr\"om approximation~\citep{WilliamsSeeger2001,
DrineasMahoney2005} uses a random or adaptive landmark subset to
approximate the gram; it produces dimension-reduced embeddings
but no lower distortion bounds in the bottleneck metric on
$\D{n}$, in contrast to Theorem~\ref{thm:nu_rho}.
PALACE's cover-theoretic analysis via Lebesgue numbers is closer
in spirit to the coreset literature for geometric
clustering~\citep{AgarwalHarPeled2005,FeldmanLangberg2011}, with
the Lebesgue-number criterion playing the role of a coreset's
coverage radius.

\paragraph{Kernel-RKHS minimax theory.}
PLACE's classification rate $O(kR/(\Delta\sqrt{m_{\min}}))$ uses the
metric-SVM machinery of Vapnik~\citep{Vapnik1998} and
Mohri--Rostamizadeh--Talwalkar~\citep{MohriRostamizadehTalwalkar2018},
with a matching Le~Cam lower bound
(\citealp{LeCam1973,Yu1997,Tsybakov2009}) for bounded Hilbert-space
inputs.  Theorems~\ref{thm:data_dependent}--\ref{thm:lower_bound}
lift this to the RKHS induced by the PALACE landmark kernel,
establishing the kernel analogue
$O((k{-}1)\sqrt{K}/(\gamma\sqrt{m_{\min}}))$ with matching binary
lower bound (tight up to the OvO $(k{-}1)$ factor): $\sqrt{K}$
plays the role of the input norm bound, and the kernel margin
$\gamma$ replaces $\Delta/2$.  The lower bound is, to our
knowledge, the first matching minimax rate for the landmark-kernel
construction on persistence diagrams.

\section{Non-Uniform Cover Theory}\label{sec:cover}

This section develops PALACE's intrinsic distortion certificate
$\rhonu(\tau;\LC)$ for arbitrary admissible landmark configurations.
The summation embedding $\Phi$ is constructed in
Section~\ref{sec:embed} from a per-landmark coordinate function
$\varphi_{p,r}$;
Section~\ref{sec:rhonu} establishes the lower distortion bound
$\|\Phi(A;\LC) - \Phi(B;\LC)\|_{\ell^2} \geq \rhonu(\tau;\LC)$ on
$\tau$-separated cross-class pairs, under
$\tau$-admissibility (a Lebesgue-number criterion on the cover,
Definition~\ref{def:sep_radius}) and a non-interference
condition inherited from~\citet[Prop.~2.1(b)]{PaperI}
(Definition~\ref{def:noninterference_pII}; the top-$N_{\max}$
persistence filter controls part of this risk in practice, and
Remark~\ref{rem:top_Nmax} reports the empirical audit). The
adaptive placement reduces the landmark budget by $(D/L)^2$ over the
uniform grid when data concentrates with diameter $D \ll L$
(Theorem~\ref{thm:comparison}).

Within the broader theory $\rhonu$ is an \emph{admissibility
guarantee}: $\tau$-separated diagrams remain separated in $\ell^2$
(and in $\mathcal{H}_{k_\LC}$), the precondition under which the
classification quantities $\Delta$ and $\gamma$ of
Section~\ref{sec:ot} are nondegenerate.  The classification error
bound is stated in the kernel margin $\gamma$
(Theorem~\ref{thm:data_dependent}); the per-prediction certificate
is stated in the raw class-mean separation $\Delta$ on $\R^K$
(Section~\ref{sec:certified}, matching the framework
of~\citealp{PaperI}).  $\rhonu$ enters both via the bridges
$\gamma \geq \tfrac{1}{2}(\kappa\rhonu - 2D_{\max})$ and
$\Delta \geq \rhonu - 2\bar D_{\max}$
(Proposition~\ref{prop:lambda_sep_pII}).

A \emph{landmark configuration} is a finite set
\begin{equation}\label{eq:config_def}
  \LC \;=\; \bigl\{(p_k, r_k, w_k)\bigr\}_{k=1}^K,
  \qquad p_k \in \D{1},\ r_k > 0,\ w_k > 0,\ \sum_{k=1}^K w_k^2 = 1,
\end{equation}
generalizing the conventional offset grid (fixed positions,
uniform support radius $\tfrac{3R}{2}$) to adaptive positions,
radii, and weights.  The Mitra--Virk grid~\citep{Mitra2024}
achieves bounded cover multiplicity ($\leq 4$) at every point
above the diagonal; the Lebesgue-number criterion below extends
the admissibility test to non-uniform configurations.

\subsection{Coordinate Functions and the Embedding}\label{sec:embed}

The PALACE embedding is built from a single primitive: a per-landmark
pyramid evaluated at diagram points.  We give the construction in two
steps---first on $\D{1}$, then extended to $\D{n}$---and record the
Lipschitz property used throughout the cover-theoretic analysis.

For $(p, r) \in \D{1} \times (0,\infty)$ and $x \in \D{1}$, the
\emph{coordinate function}
\begin{equation}\label{eq:coord_func}
  \varphi_{p,r}(x) \;:=\; \max\{r - \db(p, x),\; 0\}
\end{equation}
is the piecewise-linear cap with peak $r$ at $x = p$ and support
on the closed ball $\overline{B}(p, r) \subset \D{1}$.
Geometrically, $\varphi_{p,r}$ is a \emph{pyramid} over
$\overline{B}(p, r)$: apex of height $r$ at
$x = p$, decaying linearly to zero at the boundary, and zero
outside. This generalizes the Mitra--Virk coordinate
$\varphi_{R,p}$~\citep{Mitra2024}, whose support ball has uniform
radius $\tfrac{3R}{2}$, to landmark-specific radii $r_k$.

The lift to $n$-point diagrams is by summation over diagram points.
For $A \in \D{n}$, the \emph{sum-pool coordinate} extends the
single-point cap function additively:
\begin{equation}\label{eq:sum_coord}
  \varphi_{p, r}(A)
  \;:=\; \sum_{a \in A} \varphi_{p, r}(a)
  \;=\; \sum_{a \in A} \max\{r - \db(p, a),\; 0\}.
\end{equation}
The symbol $\varphi_{p,r}$ is overloaded by argument type: a single
point $a \in \D{1}$ returns the pyramid height
of~\eqref{eq:coord_func}; a diagram $A \in \D{n}$ returns the
sum-pool above. The structural property used throughout the
cover-theoretic analysis is the per-point Lipschitz behavior of
$\varphi_{p,r}$.

\begin{lemma}[Bottleneck-Lipschitz coordinate]\label{lem:lip_coord}
For every $(p, r) \in \D{1} \times (0,\infty)$,
\[
  |\varphi_{p,r}(x) - \varphi_{p,r}(y)| \;\leq\; \db(x, y)
  \qquad \text{for all } x, y \in \D{1}.
\]
\end{lemma}

\begin{proof}
$x \mapsto \db(p, x)$ is $1$-Lipschitz on $\D{1}$ by the triangle
inequality for $\db$, and $t \mapsto \max(r - t, 0)$ is $1$-Lipschitz
on $\R$. Composition preserves the Lipschitz constant.
\end{proof}

The embedding aggregates sum-pool coordinates over the configuration
with per-landmark weights.

\begin{definition}[Summation landmark embedding]\label{def:sum_embed}
Given $\LC = \{(p_k, r_k, w_k)\}_{k=1}^K$ as in~\eqref{eq:config_def},
the \emph{summation landmark embedding} is the map
$\Phi(\,\cdot\,;\LC) : \D{n} \to \R^K$ with coordinates
\begin{equation}\label{eq:sum_embed}
  \Phi_k(A; \LC) \;:=\; w_k\,\varphi_{p_k, r_k}(A),
  \qquad k = 1, \ldots, K.
\end{equation}
\end{definition}

The summation form generalizes the per-scale block
$\Phi_R(A) = \bigl(\sum_{a \in A} \varphi_{R,p}(a)\bigr)_{p \in \GG_R^+}$
of PLACE's multi-scale embedding~\citep{PaperI} from a fixed
parity-constrained grid to arbitrary data-adaptive
configurations. Each coordinate costs one bottleneck
evaluation per (landmark, diagram point) pair, total $O(K \cdot |A|)$
linear in diagram cardinality.

Lemma~\ref{lem:lip_coord} extends from single points to diagrams via
bijective matching: for any $A, B \in \D{n}$ of cardinality at most
$N_{\max}$ and an optimal $\sigma : A \to B$ realizing $\db(A,B)$,
\[
  |\varphi_{p_k, r_k}(A) - \varphi_{p_k, r_k}(B)|
  \;\leq\; \sum_i |\varphi_{p_k, r_k}(a_i) - \varphi_{p_k, r_k}(b_{\sigma(i)})|
  \;\leq\; N_{\max}\,\db(A, B).
\]
Squaring, summing over $k$, and using $\sum_k w_k^2 = 1$ yields the
unconditional upper distortion bound
\begin{equation}\label{eq:upper_distortion}
  \norm{\Phi(A;\LC) - \Phi(B;\LC)}_{\ell^2}
  \;\leq\; N_{\max}\,\db(A, B),
\end{equation}
matching \citet[Eq.~(2.5)]{PaperI}.  Equivalently,
$N_{\max}^{-1}\Phi$ is $1$-Lipschitz; we keep the unnormalized
sum-pool form to preserve the semantics inherited
from~\citet{PaperI} and to keep downstream certificate constants
($\bar R \leq N_{\max}\tau$, Section~\ref{sec:certified}) parallel.
The complementary \emph{lower} distortion bound---non-trivial and
governed by the geometry of the landmark cover plus a
non-interference condition---is the subject of the next
subsection.

\subsection{The Non-Uniform Distortion Certificate}\label{sec:rhonu}

Throughout this subsection and the next, fix a data support
$\mathcal{R} \subset \D{1} \cap [0,L]^2$ of $\db$-diameter
$D := \diam_{\db}(\mathcal{R}) \leq L$.
The distortion certificate depends on a single geometric quantity
of the landmark cover: its (classical) Lebesgue number $\lambda_0$
\citep[Lebesgue's covering lemma; see][Thm.~27.5]{Munkres2000},
made precise for our cover in
Definition~\ref{def:sep_radius}.
Picture the family of cap functions $\{\varphi_{p_k, r_k}\}$ as a
landscape of overlapping pyramids---one of peak height~$r_k$ at each
landmark~$p_k$, decaying linearly to zero at the ball boundary. Then
$\lambda_0$ is the lowest point of the upper envelope of these
pyramids over the data support: every $x \in \mathcal{R}$ lies under
some pyramid that still rises to at least~$\lambda_0$ above it, and
the admissibility conditions below force this minimum height to be
commensurate with the separation scale~$\tau$
(Figure~\ref{fig:admissibility}(a)).

\begin{definition}[$\tau$-admissibility]\label{def:sep_radius}
Let $\LC = \{(p_k, r_k, w_k)\}_{k=1}^K$ be a configuration whose
balls cover the data support,
$\mathcal{R} \subseteq \bigcup_k \overline{B}(p_k, r_k)$.
Define its \emph{Lebesgue number}
\begin{equation}\label{eq:lebesgue}
  \lambda_0(\LC) \;:=\;
  \inf_{x \in \mathcal{R}}\; \max_k \;\varphi_{p_k, r_k}(x).
\end{equation}
Equivalently, $\lambda_0(\LC)$ is the largest $\rho \geq 0$ for which
the shrunk cover $\{\overline{B}(p_k, r_k - \rho)\}_k$ still contains
$\mathcal{R}$ (the maximum uniform shrinkage of the cover balls that
preserves coverage of $\mathcal{R}$).
The configuration is \emph{$\tau$-admissible} at separation scale
$\tau > 0$ if
\begin{enumerate}[(i)]
\item $\lambda_0(\LC) \geq \tau/4$,
\item $\max_k r_k \leq (\tau + \lambda_0)/2$.
\end{enumerate}
\end{definition}

Two immediate consequences of Definition~\ref{def:sep_radius} are
recorded for downstream use. First, conditions (i) and (ii) jointly
bound $\max_k r_k \leq \tau$: combining
$\max_k r_k \leq (\tau + \lambda_0)/2$ with $\lambda_0 \leq \max_k r_k$
(immediate from~\eqref{eq:lebesgue}) gives
$\max_k r_k \leq (\tau + \max_k r_k)/2$, i.e., $\max_k r_k \leq \tau$,
hence $\lambda_0 \leq \tau$.
Second, the upper-envelope winner at any $x \in \mathcal{R}$ has
$r_{k^\star} \geq \tau/4$: the maximum in~\eqref{eq:lebesgue} at
$x$ is attained at some $k^\star$ with
$\varphi_{p_{k^\star}, r_{k^\star}}(x) \geq \lambda_0$, hence
$r_{k^\star} \geq \lambda_0 \geq \tau/4$. Hence
$\{k : r_k \geq \tau/4\} \neq \emptyset$, and
the minimum-weight quantity
$w_{\min}^{\geq}(\tau; \LC) := \min\{w_k : r_k \geq \tau/4\}$
appearing in Theorem~\ref{thm:nu_rho} is well-defined.

$\tau$-admissibility is operationally cheap to check.  The
Lebesgue number $\lambda_0$ is computable in $O(nmK)$ time from
the training diagrams ($m$ diagrams of at most $n$ points each,
$K$ landmarks); for the
square uniform grid at spacing $R$ with ball radius $3R/2$
(Theorem~\ref{thm:comparison}'s uniform model),
$\lambda_0(\GG_R) = R$ and conditions~(i) and~(ii) yield the
admissibility window $R \in [\tau/4, \tau/2]$.  For non-uniform
configurations produced by FPS
(Theorem~\ref{thm:fps_greedy}), both conditions can be verified
directly once landmarks are placed; if either fails, $\tau$ is
rescaled to the largest value compatible with the cover.

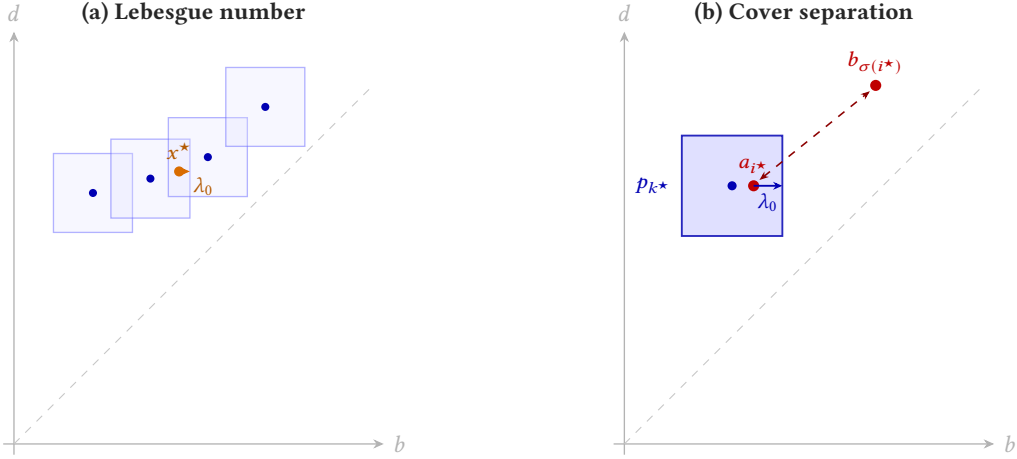
\begin{figure}[t]
\centering
\begin{tikzpicture}[
  scale=0.95,
  font=\footnotesize,
  panel/.style={font=\footnotesize\bfseries, text=black!85},
  axis/.style={-{Stealth[length=4pt,width=3.5pt]}, gray!60, line width=0.5pt},
  diag/.style={gray!45, dashed, line width=0.4pt},
  cover/.style={blue!55, fill=blue!10, fill opacity=0.30, draw opacity=0.55, line width=0.5pt},
  active/.style={blue!70!black, fill=blue!22, fill opacity=0.55, draw opacity=0.85, line width=0.7pt},
  lmark/.style={fill=blue!70!black},
  datum/.style={fill=red!75!black},
  worst/.style={fill=orange!85!black},
  arrow/.style={-{Stealth[length=3.5pt,width=3pt]}, line width=0.7pt},
  lab/.style={font=\scriptsize, inner sep=1.2pt},
]
\def\W{5.0}
\def\H{5.6}

\begin{scope}[xshift=0cm]
  \draw[axis] (-0.15,0) -- (\W+0.15,0) node[right]{$b$};
  \draw[axis] (0,-0.15) -- (0,\H+0.15) node[above]{$d$};
  \draw[diag] (0,0) -- (\W,\W);

  \foreach \cx/\cy/\rad in {%
       1.1/3.5/0.55,
       1.9/3.7/0.55,
       2.7/4.0/0.55,
       3.5/4.7/0.55}
    {\draw[cover] (\cx-\rad,\cy-\rad) rectangle (\cx+\rad,\cy+\rad);
     \fill[lmark] (\cx,\cy) circle (1.6pt);}

  \fill[worst] (2.30,3.80) circle (2.0pt);
  \node[lab, text=orange!70!black, anchor=south] at (2.30,3.92) {$x^\star$};
  \draw[arrow, orange!85!black] (2.30,3.80) -- (2.45,3.80);
  \node[lab, text=orange!70!black, anchor=north west] at (2.45,3.78) {$\lambda_0$};

  \node[panel] at (\W/2, \H+0.4) {(a) Lebesgue number};
\end{scope}

\begin{scope}[xshift=8.5cm]
  \draw[axis] (-0.15,0) -- (\W+0.15,0) node[right]{$b$};
  \draw[axis] (0,-0.15) -- (0,\H+0.15) node[above]{$d$};
  \draw[diag] (0,0) -- (\W,\W);

  \draw[active] (0.8,2.9) rectangle (2.2,4.3);
  \fill[lmark] (1.5,3.6) circle (1.8pt);
  \node[lab, text=blue!70!black, anchor=east] at (0.7,3.6) {$p_{k^\star}$};

  \fill[datum] (1.8,3.6) circle (2.2pt);
  \node[lab, text=red!75!black, anchor=south] at (1.8,3.72) {$a_{i^\star}$};
  \draw[arrow, blue!70!black] (1.8,3.6) -- (2.2,3.6);
  \node[lab, text=blue!70!black, anchor=north] at (2.0,3.55) {$\lambda_0$};

  \fill[datum] (3.5,5.0) circle (2.2pt);
  \node[lab, text=red!75!black, anchor=south] at (3.5,5.10) {$b_{\sigma(i^\star)}$};

  \draw[{Stealth[length=3.5pt]}-{Stealth[length=3.5pt]}, red!55!black,
        line width=0.6pt, dashed]
    (1.88,3.68) -- (3.42,4.92);

  \node[panel] at (\W/2, \H+0.4) {(b) Cover separation};
\end{scope}

\end{tikzpicture}
\caption{%
  \textbf{$\tau$-admissibility (Def.~\ref{def:sep_radius}).}
  \textbf{(a)}~The Lebesgue number is the lowest point of the upper
  envelope of the cap-function pyramids over the data support:
  $\lambda_0 = \inf_{x\in\mathcal{R}} \max_k \varphi_{p_k, r_k}(x)$.
  Highlighted: a point $x^\star$ realizing this minimum.
  Condition~(i) requires $\lambda_0 \geq \tau/4$; condition~(ii)
  bounds radii $\max_k r_k \leq (\tau+\lambda_0)/2$.
  \textbf{(b)}~The single-point step in
  Theorem~\ref{thm:nu_rho}'s proof.
  For a $\tau$-separated pair $(A, B)$, fix the worst-matched pair
  $(a_{i^\star}, b_{\sigma(i^\star)})$.  The Lebesgue number applied
  at $a_{i^\star} \in \D{1}$ yields some pyramid
  $\varphi_{p_{k^\star}, r_{k^\star}}$ with
  $\varphi_{p_{k^\star}, r_{k^\star}}(a_{i^\star}) \geq \lambda_0$;
  admissibility~(ii) then places $b_{\sigma(i^\star)}$ outside that
  pyramid's support, so
  $\varphi_{p_{k^\star}, r_{k^\star}}(b_{\sigma(i^\star)}) = 0$.
  Cross-class non-interference shows the remaining $b$-points lie
  outside the ball as well, so $\Phi_{k^\star}(B;\LC) = 0$;
  $a$-side contributions are non-negative.
}
\label{fig:admissibility}
\end{figure}

\begin{definition}[Non-interference]\label{def:noninterference_pII}
A pair $(A, B) \in \D{n} \times \D{n}$ satisfies
\emph{non-interference} if either $n = 1$ (vacuous), or $n \geq 2$
and there exists an optimal matching $\sigma : A \to B$ realizing
$\db(A, B)$ such that every cross-diagram pair
$(a_i, b_{\sigma(j)})$ with $i \neq j$ is strictly farther than
three times the bottleneck distance:
\begin{equation}\label{eq:noninterference_pII}
  \min_{i \neq j} \db(a_i, b_{\sigma(j)}) \;>\; 3\, \db(A, B).
\end{equation}
\end{definition}

Definition~\ref{def:noninterference_pII} mirrors
\citet[Definition~2.1]{PaperI}; the partial-matching extensions
needed when $\db$ projects some points to the diagonal are
straightforward bookkeeping and do not appear in the experiments.
A sufficient condition for non-interference, expressed purely in
within-class diagram geometry, decouples this hypothesis from any
cross-pair check.

\begin{proposition}[Non-interference from scale separation]\label{prop:ni_from_geometry}
Suppose that for every cross-class pair $(A, B)$ arising in the
classification, the within-diagram minimum separation satisfies
\begin{equation}\label{eq:within_scale_sep}
  \min_{i \neq j} \db(a_i, a_j) \;>\; 4\, \db(A, B).
\end{equation}
Then the non-interference
condition~\eqref{eq:noninterference_pII} holds for $(A, B)$.
\end{proposition}

\begin{proof}
For $i \neq j$, using $\db(a_j, b_{\sigma(j)}) \leq \db(A, B)$ and
the triangle inequality,
\[
  \db(a_i, b_{\sigma(j)})
  \;\geq\; \db(a_i, a_j) - \db(a_j, b_{\sigma(j)})
  \;>\; 4\,\db(A, B) - \db(A, B)
  \;=\; 3\,\db(A, B).
\]
The analogous within-$B$ condition
$\min_{i \neq j} \db(b_i, b_j) > 4\,\db(A,B)$ also implies
non-interference by the same argument routed through $b_j$;
either condition alone is sufficient.
\end{proof}

\begin{remark}[Empirical scope of non-interference]\label{rem:top_Nmax}
Condition~\eqref{eq:within_scale_sep} requires within-diagram
features to sit at a scale strictly larger than the cross-class
separation scale.  The top-$N_{\max}$ persistence
filter~\citep{PaperI}, which retains only the $N_{\max}$ points
with largest persistence $\db(a, \Delta)$, controls part of this
risk by discarding low-persistence features clustered near the
diagonal.  On the chemical graph benchmarks of
Section~\ref{sec:experiments}, however, the filter is not tight
enough to make non-interference hold pointwise: an audit on four
datasets at the per-dataset headline filtration finds essentially
$0\%$ of cross-class pairs satisfying
condition~\eqref{eq:noninterference_pII}, with median cross-ratios
$\min_{i \neq j}\db(a_i, b_{\sigma(j)})/\db(A,B)$ at or near zero
(audit reported in Section~\ref{sec:experiments}).  The hypothesis is therefore
structural; remarkably, the \emph{conclusion} of
Theorem~\ref{thm:nu_rho} (the certificate
$\|\Phi(A) - \Phi(B)\|_{\ell^2} \geq \rhonu(\tau;\LC)$) holds on
$99.9$--$100\%$ of cross-class pairs in the same audit
(Table~\ref{tab:certificate_bound_audit}), with median embedded
distance $3$--$14\times$ the certificate---non-interference is
sufficient but not necessary for the bound, and the proof is
overcautious on these diagrams.  The working classification
machinery in Section~\ref{sec:ot} operates at the kernel-margin
level $\gamma > 0$
(Theorem~\ref{thm:data_dependent}), independent of pairwise
non-interference.
\end{remark}

We now establish the main result: an explicit lower distortion
bound for admissible configurations. At its core,
Theorem~\ref{thm:nu_rho} is a quantitative
\emph{Lebesgue-number lemma}\footnote{Henri L\'eon Lebesgue
(1875--1941): his \emph{Lebesgue-number lemma}
\citep[see][Thm.~27.5]{Munkres2000} guarantees that every open
cover of a compact metric space admits a positive uniform scale at
which every ball lies in some cover element.  The non-uniform
distortion certificate below makes that scale explicit for the
landmark cover $\LC$.  We dedicate this work to his memory.}
for the landmark cover $\LC$.

\begin{theorem}[Non-uniform distortion certificate]\label{thm:nu_rho}
Let $\LC = \{(p_k, r_k, w_k)\}_{k=1}^K$ be a $\tau$-admissible
configuration.  For any $A, B \in \D{n}$ with all points of
$A$ and $B$ in $\mathcal{R}$, $\db(A,B) \geq \tau$, and the
non-interference condition~\eqref{eq:noninterference_pII}
holding,
\begin{equation}\label{eq:nu_rho}
  \norm{\Phi(A;\LC) - \Phi(B;\LC)}_{\ell^2} \;\geq\; \rhonu(\tau;\LC),
\end{equation}
where the \emph{non-uniform distortion certificate} is
\begin{equation}\label{eq:nu_rho_formula}
  \rhonu(\tau;\LC) \;:=\; \tfrac{1}{4}\,\tau \cdot w_{\min}^{\geq}(\tau; \LC),
  \qquad
  w_{\min}^{\geq}(\tau; \LC) := \min_{k:\, r_k \geq \tau/4} w_k.
\end{equation}
In particular, for equal weights $w_k = K^{-1/2}$ and all
$r_k \geq \tau/4$, the certificate is
$\rhonu(\tau; \LC) = \tau / (4\sqrt{K})$.
\end{theorem}

\begin{proof}
The strategy is to exhibit a single \emph{witnessing coordinate}
$k^\star$ that contributes the certificate value to the $\ell^2$
distance: at the worst-matched cross-class pair
$(a_{i^\star}, b_{\sigma(i^\star)})$, the Lebesgue number forces
some pyramid $\varphi_{p_{k^\star}, r_{k^\star}}$ to rise above
$\lambda_0$ at $a_{i^\star}$, while admissibility~(ii) plus
non-interference force every $b_{\sigma(j)}$ to lie outside this
pyramid's support; the $k^\star$-th coordinate alone then witnesses
the floor $w_{k^\star}\,\lambda_0 \geq w_{\min}^{\geq}\,\tau/4$.

By Definition~\ref{def:noninterference_pII}, fix the optimal matching
$\sigma : A \to B$ realizing $\db(A, B)$ for
which~\eqref{eq:noninterference_pII} holds (vacuous and trivial for
$n = 1$). Let $i^\star$ be an index of a worst-matched pair, so
$\db(a_{i^\star}, b_{\sigma(i^\star)}) = \db(A, B) \geq \tau$.

By Definition~\ref{def:sep_radius}, $\lambda_0(\LC) \geq \tau/4$, so
applied at the single point $a_{i^\star} \in \mathcal{R}$ there
exists $k^\star$ with
\begin{equation}\label{eq:witnessing_coord}
  \varphi_{p_{k^\star}, r_{k^\star}}(a_{i^\star})
  \;=\; r_{k^\star} - \db(p_{k^\star}, a_{i^\star})
  \;\geq\; \lambda_0(\LC).
\end{equation}
In particular $r_{k^\star} \geq \lambda_0 \geq \tau/4$, so
$k^\star \in \{k : r_k \geq \tau/4\}$ and
$w_{k^\star} \geq w_{\min}^{\geq}(\tau; \LC)$.

The triangle inequality and admissibility's upper-radius bound
$r_{k^\star} \leq (\tau + \lambda_0)/2$ give
\[
  \db(p_{k^\star}, b_{\sigma(i^\star)})
  \;\geq\; \db(a_{i^\star}, b_{\sigma(i^\star)}) - \db(p_{k^\star}, a_{i^\star})
  \;\geq\; \tau - (r_{k^\star} - \lambda_0)
  \;\geq\; \tau - \tfrac{\tau - \lambda_0}{2}
  \;=\; \tfrac{\tau + \lambda_0}{2} \;\geq\; r_{k^\star},
\]
so $\varphi_{p_{k^\star}, r_{k^\star}}(b_{\sigma(i^\star)}) = 0$.
For $j \neq i^\star$, if
$\db(p_{k^\star}, b_{\sigma(j)}) \leq r_{k^\star}$, the triangle
inequality and the same upper-radius bound would give
\[
  \db(a_{i^\star}, b_{\sigma(j)})
  \;\leq\; \db(a_{i^\star}, p_{k^\star}) + \db(p_{k^\star}, b_{\sigma(j)})
  \;\leq\; (r_{k^\star} - \lambda_0) + r_{k^\star}
  \;=\; 2r_{k^\star} - \lambda_0 \;\leq\; \tau,
\]
contradicting~\eqref{eq:noninterference_pII} since
$\db(a_{i^\star}, b_{\sigma(j)}) > 3\,\db(A, B) \geq 3\tau > \tau$.
Hence $\varphi_{p_{k^\star}, r_{k^\star}}(b_{\sigma(j)}) = 0$ for
every $j$, and $\Phi_{k^\star}(B; \LC) = 0$.
By non-negativity of $\varphi_{p_{k^\star}, r_{k^\star}}$,
\[
  \Phi_{k^\star}(A; \LC)
  \;=\; w_{k^\star} \sum_{a \in A}\varphi_{p_{k^\star}, r_{k^\star}}(a)
  \;\geq\; w_{k^\star}\,\varphi_{p_{k^\star}, r_{k^\star}}(a_{i^\star})
  \;\geq\; w_{k^\star}\,\lambda_0.
\]
Using $\lambda_0 \geq \tau/4$ and
$w_{k^\star} \geq w_{\min}^{\geq}(\tau; \LC)$, the $k^\star$-th
coordinate alone contributes at least
$w_{\min}^{\geq}(\tau; \LC) \cdot \tau/4 = \rhonu(\tau; \LC)$
to the $\ell^2$ distance.
\end{proof}

\subsection{Budget Comparison and Optimal Configuration}\label{sec:budget}

Section~\ref{sec:rhonu} established the cover-level certificate
$\rhonu(\tau;\LC)$ as a structural admissibility statement.  This
subsection builds the configuration: how many landmarks suffice
(Theorem~\ref{thm:comparison}), the weight and position choices
that maximize the worst-case and effective certificates
(Proposition~\ref{prop:optimal_config}), and the algorithmic
realization via farthest-point sampling
(Theorem~\ref{thm:fps_greedy}, Corollary~\ref{cor:fps_admissible}).
The data-adaptive advantage manifests as a $(D/L)^2$ budget
reduction over the uniform grid when diagrams concentrate.

\begin{theorem}[Budget reduction]\label{thm:comparison}
For any separation scale $\tau > 0$, there is a $\tau$-admissible
configuration whose balls cover $\mathcal{R}$ with at most
$K_{\mathrm{adapt}} \leq (4D/\tau)^2$ landmarks. Every uniform-grid
construction covering $[0,L]^2$ at admissible spacing
$R \in [\tau/4, \tau/2]$ uses
$K_{\mathrm{unif}} = (L/R)^2 \geq 4(L/\tau)^2$ landmarks (PLACE's
parity-restricted offset variant of \citealp{Mitra2024} keeps the
same $\Theta((L/\tau)^2)$ scaling, halved by the parity sieve).
Consequently
\begin{equation}\label{eq:budget_ratio}
  \frac{K_{\mathrm{adapt}}}{K_{\mathrm{unif}}} \;\leq\;
  \frac{(4D/\tau)^2}{4(L/\tau)^2} \;=\; 4\,\frac{D^2}{L^2}.
\end{equation}
\end{theorem}

\begin{proof}
\emph{Adaptive.}\; Take a maximal $(\tau/4)$-separated set
$\{p_1, \ldots, p_K\} \subset \mathcal{R}$, which is also a
$(\tau/4)$-net of $\mathcal{R}$: by maximality, every
$x \in \mathcal{R}$ lies within $\db$-distance $\tau/4$ of some
$p_k$ (otherwise $x$ could be added without violating separation).
Equip each landmark with radius $r_k = \tau/2$. The net property
gives $\varphi_{p_k, r_k}(x) \geq \tau/2 - \tau/4 = \tau/4$ for
some $k$, hence $\lambda_0(\LC) \geq \tau/4$ and admissibility~(i).
Admissibility~(ii) is also met:
$\max_k r_k = \tau/2 \leq (\tau + \tau/4)/2 = 5\tau/8$.
The separated property bounds the cardinality: standard packing
in $(\D{1}, \db)$\footnote{The data-relevant region of
$(\D{1}, \db)$ has doubling dimension $d_0 = 2$: away from the
diagonal, $\db$ coincides with $\ell^\infty$ on $\R^2$, which has
doubling constant $4$ and hence doubling dimension
$\log_2 4 = 2$ \citep[Ch.~10]{Heinonen2001}.}
gives $K \leq (D/(\tau/4))^{2} = (4D/\tau)^{2}$ (the doubling
constant $4$ on $\R^2$ with the $\ell^\infty$ metric is the
elementary observation that an $\ell^\infty$-ball of radius
$2r$ is covered by $4$ axis-aligned $\ell^\infty$-balls of
radius $r$).

\emph{Uniform.}\; Take the square grid
$\GG_R = R\cdot \mathbb{Z}^2 \cap [0,L]^2$ with spacing $R$ and
ball radius $r = 3R/2$ (an $\ell^\infty$ tiling with $4\times$
overlap; the offset variant of \citet{Mitra2024} differs only in
the half-plane truncation and shares the same Lebesgue-number
scaling). By direct computation, $\lambda_0(\GG_R) = R$ (the pyramid
height at any cell-corner point covered by four adjacent balls).
Admissibility~(i) $\lambda_0 \geq \tau/4$ requires
$R \geq \tau/4$; admissibility~(ii) $3R/2 \leq (\tau + R)/2$
requires $R \leq \tau/2$. The feasible window
$R \in [\tau/4,\, \tau/2]$ gives $|\GG_R| = (L/R)^2$ landmarks,
and $|\GG_R| \geq 4(L/\tau)^2$ since $R \leq \tau/2$.
\end{proof}

For data concentrated in $\mathcal{R}$ with $D \ll L$, the bounds
of Theorem~\ref{thm:comparison} open a budget window in which the
adaptive configuration $\LC^*$ (equal weights, radii
$r_k = \tau/2$, $|\LC^*| \leq (4D/\tau)^2$ landmarks) is
$\tau$-admissible and Theorem~\ref{thm:nu_rho} certifies
$\rhonu(\tau;\LC^*) = \tau/(4\sqrt{|\LC^*|}) > 0$, while any
single-scale uniform grid $\GG_R$ with spacing $R > \tau/2$
(equivalently $|\GG_R| < 4(L/\tau)^2$ landmarks) violates
admissibility~(ii)---since
$r = 3R/2 > (\tau + R)/2 = (\tau + \lambda_0)/2$---and falls
outside Theorem~\ref{thm:nu_rho}'s hypothesis.

Two qualifications follow.  The single-scale restriction in the
comparison above is essential:
\citet{PaperI}'s multi-scale uniform construction combines $N$
scales $R_1 < \cdots < R_N$ and achieves positive $\rhonu$ at any
$K$ by spanning scales both above and below $\tau$. PALACE's
adaptive $\LC^*$ matches this at the smaller single-scale budget
$K = (4D/\tau)^2$ without multi-scale machinery.
Second, in the both-admissible regime $K \geq 16(L/\tau)^2$, a
$K$-landmark uniform grid and a $K$-landmark adaptive
configuration (equal weights, both $\tau$-admissible) deliver
identical worst-case certificates $\tau/(4\sqrt{K})$---the formula
depends on the cover only through $\tau \cdot w_{\min}^{\geq}$,
and equal weights flatten any positional advantage. The
substantive $L/D$-scaling advantage of adaptive placement
therefore manifests at the admissibility boundary
(Theorem~\ref{thm:comparison}), not as a power-law
gain inside the both-admissible window.

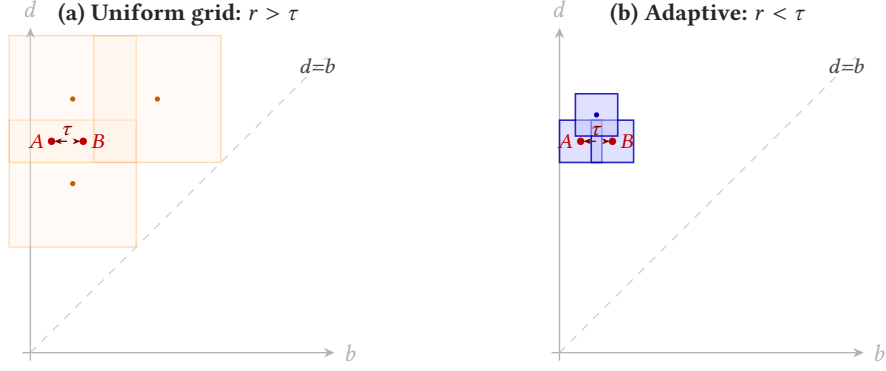
\begin{figure}[t]
\centering
\begin{tikzpicture}[
  scale=0.7,
  font=\footnotesize,
  panel/.style={font=\footnotesize\bfseries, text=black!85},
  axis/.style={-{Stealth[length=4pt,width=3.5pt]}, gray!60, line width=0.5pt},
  diag/.style={gray!45, dashed, line width=0.4pt},
  ucover/.style={orange!55, fill=orange!12, fill opacity=0.32, draw opacity=0.55, line width=0.5pt},
  acover/.style={blue!70!black, fill=blue!22, fill opacity=0.55, draw opacity=0.85, line width=0.6pt},
  lmark/.style={fill=blue!70!black},
  ulmark/.style={fill=orange!75!black},
  datum/.style={fill=red!75!black},
  caption/.style={font=\scriptsize, text=gray!55!black, align=center},
  arrow/.style={-{Stealth[length=3.5pt,width=3pt]}, line width=0.7pt},
  lab/.style={font=\scriptsize, inner sep=1.2pt},
]
\def\W{5.6}
\def\H{6.0}

\begin{scope}[xshift=0cm]
  \draw[axis] (-0.15,0) -- (\W+0.15,0) node[right]{$b$};
  \draw[axis] (0,-0.15) -- (0,\H+0.15) node[above]{$d$};
  \draw[diag] (0,0) -- (\W,\W);
  \node[caption, anchor=south west] at (\W-0.7, \W-0.5) {$d{=}b$};

  \foreach \cx/\cy in {0.8/3.2, 0.8/4.8, 2.4/4.8}
    {\draw[ucover] (\cx-1.2,\cy-1.2) rectangle (\cx+1.2,\cy+1.2);
     \fill[ulmark] (\cx,\cy) circle (1.4pt);}

  \fill[datum] (0.4,4.0) circle (2.0pt);
  \node[lab, text=red!75!black, anchor=east] at (0.32,4.0) {$A$};
  \fill[datum] (1.0,4.0) circle (2.0pt);
  \node[lab, text=red!75!black, anchor=west] at (1.08,4.0) {$B$};
  \draw[{Stealth[length=3pt]}-{Stealth[length=3pt]}, red!50!black,
        line width=0.5pt, dashed] (0.45,4.0) -- (0.95,4.0);
  \node[lab, text=red!50!black, anchor=south] at (0.7,4.04) {$\tau$};

  \node[panel] at (\W/2, \H+0.4) {(a) Uniform grid: $r > \tau$};
\end{scope}

\begin{scope}[xshift=10cm]
  \draw[axis] (-0.15,0) -- (\W+0.15,0) node[right]{$b$};
  \draw[axis] (0,-0.15) -- (0,\H+0.15) node[above]{$d$};
  \draw[diag] (0,0) -- (\W,\W);
  \node[caption, anchor=south west] at (\W-0.7, \W-0.5) {$d{=}b$};

  \foreach \cx/\cy in {0.4/4.0, 1.0/4.0, 0.7/4.5}
    {\draw[acover] (\cx-0.4,\cy-0.4) rectangle (\cx+0.4,\cy+0.4);
     \fill[lmark] (\cx,\cy) circle (1.4pt);}

  \fill[datum] (0.4,4.0) circle (2.0pt);
  \node[lab, text=red!75!black, anchor=east] at (0.32,4.0) {$A$};
  \fill[datum] (1.0,4.0) circle (2.0pt);
  \node[lab, text=red!75!black, anchor=west] at (1.08,4.0) {$B$};
  \draw[{Stealth[length=3pt]}-{Stealth[length=3pt]}, red!50!black,
        line width=0.5pt, dashed] (0.45,4.0) -- (0.95,4.0);
  \node[lab, text=red!50!black, anchor=south] at (0.7,4.04) {$\tau$};

  \node[panel] at (\W/2, \H+0.4) {(b) Adaptive: $r < \tau$};
\end{scope}

\end{tikzpicture}
\caption{%
  \textbf{Budget-regime distinction (Theorem~\ref{thm:comparison}).}
  \textbf{(a)}~Single-scale uniform grid with $r > \tau$ (here
  $r = 2\tau$ for visual clarity): two $\tau$-separated diagrams
  $A, B$ at the same $d$-height share the same containing balls
  and sit at equal distances within them, while landmarks not
  containing $A, B$ contribute $\varphi_k(A) = \varphi_k(B) = 0$.
  Either way the coordinate gap vanishes.
  \textbf{(b)}~Adaptive placement with smaller radii (here
  $r = 2\tau/3$) and positions chosen so that $A, B$ fall in
  distinct balls: Theorem~\ref{thm:nu_rho} certifies
  $\rhonu(\tau;\LC^*) \geq \tau/(4\sqrt{|\LC^*|}) > 0$.
}
\label{fig:comparison}
\end{figure}

We now identify the configuration choices that maximize
$\rhonu(\tau;\LC)$, then refine to the data-dependent
\emph{effective} certificate via the Lebesgue number.

\begin{definition}[Effective certificate]\label{def:rhonu_eff}
Given an admissible configuration $\LC$, the \emph{effective
certificate} is
\begin{equation}\label{eq:rhonu_eff}
  \rhonu^{\mathrm{eff}}(\LC) \;:=\; \lambda_0(\LC) \cdot
  w_{\min}^{\geq}(\tau;\LC),
\end{equation}
sharpening Theorem~\ref{thm:nu_rho}'s worst-case
$\rhonu(\tau;\LC) = (\tau/4)\,w_{\min}^{\geq}$ by using the actual
Lebesgue number $\lambda_0(\LC) \geq \tau/4$ at the
witnessing-coordinate step~\eqref{eq:witnessing_coord} of
Theorem~\ref{thm:nu_rho}'s proof rather than the admissibility
floor $\tau/4$.
\end{definition}

\begin{proposition}[Optimal weights and effective certificate]\label{prop:optimal_config}
Fix a support $\mathcal{R} \subset \D{1} \cap [0,L]^2$ on which
diagrams are observed. Among $\tau$-admissible configurations
$\LC$ of cardinality $K$ with $r_k \geq \tau/4$ for every $k$:
\begin{enumerate}[(i)]
\item equal weights $w_k = K^{-1/2}$ maximize the certificate
  $\rhonu(\tau;\LC)$ of Theorem~\ref{thm:nu_rho}, giving
  $\tau/(4\sqrt{K})$ independent of positions $\{p_k\}$ and radii
  $\{r_k\}$;

\item the effective certificate
  $\rhonu^{\mathrm{eff}}(\LC)$ of
  Definition~\ref{def:rhonu_eff} equals
  $\lambda_0(\LC)\,K^{-1/2}$ at equal weights.  For uniform
  radii $r_k = r$, $\lambda_0(\LC) = r - \delta_K(\LC)$, where
  $\delta_K(\LC) := \sup_{x \in \mathcal{R}} \min_k \db(x, p_k)$
  is the covering radius of $\mathcal{R}$ by the landmarks.
  Maximizing $\rhonu^{\mathrm{eff}}$ over positions therefore
  reduces to minimizing $\delta_K$ on $\mathcal{R}$.
\end{enumerate}
\end{proposition}

\begin{proof}
\textbf{(i)} From Theorem~\ref{thm:nu_rho},
$\rhonu(\tau;\LC) = (\tau/4)\cdot \min_{k:\,r_k \geq \tau/4} w_k$.
Subject to $\sum_k w_k^2 = 1$ on $K$ landmarks, the symmetric
max-min bound
$K \cdot (\min_k w_k)^2 \leq \sum_k w_k^2 = 1$ gives
$\min_k w_k \leq K^{-1/2}$, with equality iff all weights equal:
$w_k = K^{-1/2}$. This yields $\rhonu = \tau/(4\sqrt{K})$. With
equal weights and every $r_k \geq \tau/4$, this value has no
further dependence on positions or radii, so it is invariant
across all admissible equal-weight configurations.

\textbf{(ii)} The witnessing-coordinate
step~\eqref{eq:witnessing_coord} of
Theorem~\ref{thm:nu_rho}'s proof gives
$\Phi_{k^\star}(A) \geq w_{k^\star}\,\lambda_0(\LC)$ at any
$a_{i^\star} \in \mathcal{R}$; the theorem then weakens this to
$w_{k^\star}\,(\tau/4)$ via admissibility~(i). Keeping the
actual $\lambda_0$ yields
$\rhonu^{\mathrm{eff}} = \lambda_0 \cdot w_{\min}^{\geq}$, equal
to $\lambda_0\,K^{-1/2}$ at equal weights. For uniform radii
$r_k = r$, $\max_k \varphi_{p_k, r}(x) = r - \min_k \db(x, p_k)$
on $\mathcal{R}$, so
$\lambda_0(\LC) = \inf_{x \in \mathcal{R}}(r - \min_k \db(x,p_k))
= r - \delta_K(\LC)$. Minimizing $\delta_K$ therefore maximizes
$\rhonu^{\mathrm{eff}}$.
\end{proof}

The quantity
$\delta_K(\LC) = \sup_{x \in \mathcal{R}} \min_k \db(x, p_k)$
in part~(ii) is the \emph{covering radius} of $\mathcal{R}$ by
the landmarks $\{p_k\}$---the worst data-point's $\db$-distance
to its nearest landmark. Minimizing $\delta_K$ over positions
$\{p_k\}$ is the classical $k$-center problem on $\mathcal{R}$,
NP-hard in general; the next theorem shows that farthest-point
sampling solves it within a factor of $2$.

\begin{theorem}[FPS as a $k$-center $2$-approximation]
\label{thm:fps_greedy}
Take $\mathcal{R}$ to be a finite set of diagram points
(the training-diagram-point support), and let $K \leq |\mathcal{R}|$.
Let
$\delta^{\ast}_K := \min_{|P| = K,\, P \subset \mathcal{R}}\,
  \max_{x \in \mathcal{R}}\, \min_{p \in P} \db(x, p)$
be the optimal $k$-center covering radius of $\mathcal{R}$.
Farthest-point sampling, the greedy selection
\begin{equation}\label{eq:fps_rule}
  p_{t+1} \;=\; \arg\max_{q \in \mathcal{R}}\,
  \min_{k \leq t} \db(q, p_k)
  \qquad (t = 1, \ldots, K-1)
\end{equation}
from an arbitrary seed $p_1 \in \mathcal{R}$, produces
$P = \{p_1, \ldots, p_K\}$ with covering radius
\begin{equation}\label{eq:fps_2approx}
  \delta_K(P) \;:=\; \max_{x \in \mathcal{R}}\,
  \min_{p \in P} \db(x, p)
  \;\leq\; 2\,\delta^{\ast}_K.
\end{equation}
\end{theorem}

\begin{proof}[Proof~\citep{Gonzalez1985}]
Suppose an optimal $K$-set
$P^{\ast} = \{c^{\ast}_1, \ldots, c^{\ast}_K\}$ achieves covering
radius $\delta^{\ast}_K$.  Define the FPS \emph{insertion distance}
$D_t := \min_{k < t} \db(p_t, p_k)$ ($t \geq 2$); the rule
\eqref{eq:fps_rule} picks $p_{t+1}$ to maximize $D_{t+1}$, and
$D_t$ is non-increasing in $t$ (each new candidate has more
existing landmarks to be close to).  Run~\eqref{eq:fps_rule} for
one additional step ($t = K$) to obtain a candidate $p_{K+1}$;
$D_{K+1} = \delta_K(P)$ by~\eqref{eq:fps_rule}, and by monotonicity
of $D_t$ every pair $(p_i, p_j)$ with $i < j \leq K{+}1$ satisfies
$\db(p_i, p_j) \geq D_j \geq D_{K+1} = \delta_K(P)$.
Pigeonhole assigns two of the $K{+}1$ points to the same optimal
cluster $\{s : \arg\min_k \db(s, c^{\ast}_k) = j\}$, so they are
within $2\delta^{\ast}_K$ of each other by triangle inequality
through $c^{\ast}_j$. Hence $\delta_K(P) \leq 2\delta^{\ast}_K$.
\end{proof}

Combining Theorem~\ref{thm:fps_greedy} with
Proposition~\ref{prop:optimal_config}(ii) yields a concrete
$\tau$-admissible FPS configuration as soon as the optimal
covering radius is small enough relative to $\tau$.

\begin{corollary}[Admissible cover from FPS]
\label{cor:fps_admissible}
Fix $\tau > 0$ and let $\mathcal{R}$ be the (finite) training-diagram
point set of Theorem~\ref{thm:fps_greedy}.
Suppose the budget $K$ is large enough that the optimal $k$-center
radius satisfies $\delta^{\ast}_K \leq \tau/8$ (which holds for
$K \geq (4D/\tau)^2$ by the packing argument in
Theorem~\ref{thm:comparison}'s proof). Equip the FPS-placed
landmarks $P = \{p_1, \ldots, p_K\}$ with uniform radii
$r_k = \tau/2$ and equal weights $w_k = K^{-1/2}$.  Then
$\LC^\star = (P, \tau/2, K^{-1/2})$ is $\tau$-admissible
(Definition~\ref{def:sep_radius}) with Lebesgue number
\begin{equation}\label{eq:fps_lambda}
  \lambda_0(\LC^\star) \;=\; \tau/2 - \delta_K(P)
  \;\geq\; \tau/2 - 2\delta^{\ast}_K \;\geq\; \tau/4,
\end{equation}
and Theorem~\ref{thm:nu_rho} gives the worst-case certificate
$\rhonu(\tau; \LC^\star) = \tau/(4\sqrt{K})$.
\end{corollary}

\begin{proof}
$\delta^{\ast}_K \leq \tau/8$ combined with
Theorem~\ref{thm:fps_greedy} yields
$\delta_K(P) \leq 2\delta^{\ast}_K \leq \tau/4$.
Proposition~\ref{prop:optimal_config}(ii) at uniform radius
$r = \tau/2$ gives
$\lambda_0(\LC^\star) = \tau/2 - \delta_K(P) \geq \tau/4$,
satisfying admissibility~(i); admissibility~(ii)
$\max_k r_k = \tau/2 \leq (\tau + \tau/4)/2 = 5\tau/8$ also holds.
The certificate value follows from
Proposition~\ref{prop:optimal_config}(i).
\end{proof}

Where FPS contributes content beyond existence is the
\emph{effective} certificate of Proposition~\ref{prop:optimal_config}(ii): at
uniform radius $r = \tau/2$, $\lambda_0(\LC) = \tau/2 - \delta_K(P)$,
and the $2$-approximation $\delta_K(P) \leq 2\delta^{\ast}_K$ gives
\[
\rhonu^{\mathrm{eff}}(\LC^\star)
= \lambda_0(\LC^\star)/\sqrt{K}
\;\geq\;
\bigl(\tau/2 - 2\delta^{\ast}_K\bigr)\big/\sqrt{K},
\]
within a factor of $2$ of the best $\rhonu^{\mathrm{eff}}$
achievable at $K$ landmarks.  The worst-case certificate
$\tau/(4\sqrt{K})$ itself is not FPS-specific
(Proposition~\ref{prop:optimal_config}(i)): every admissible
equal-weight $K$-landmark configuration attains it, and
Corollary~\ref{cor:fps_admissible}'s constructive content is the
guarantee that FPS produces such a configuration.

\begin{remark}[Practical radius choice]\label{rem:fps_practical}
The pipeline of Section~\ref{sec:experiments} uses \emph{scaled
nearest-neighbor} radii $r_k = \alpha \cdot d_{\mathrm{NN}}(p_k)$
with $d_{\mathrm{NN}}(p_k) := \min_{j \neq k} \db(p_k, p_j)$,
clipped to $[\tau/2, 4\tau]$, in place of the uniform
$r_k = \tau/2$ of Corollary~\ref{cor:fps_admissible}.  The upper
clip can violate admissibility~(ii),
$\max_k r_k \leq (\tau + \lambda_0)/2$, in which case
Theorem~\ref{thm:nu_rho} no longer applies.  The empirical
guarantees of the pipeline flow through the kernel margin
$\gamma > 0$
(Theorem~\ref{thm:data_dependent}), which is independent of
admissibility and $\rhonu$, so the loss of admissibility~(ii) does
not affect the empirical results.
\end{remark}

This closes the Section~\ref{sec:cover} construction: the FPS-placed
equal-weight configuration $\LC^\star$ is $\tau$-admissible with
no further optimization, attains both the worst-case and
$2$-approximate effective certificates as above, and inherits the
$(D/L)^2$ budget reduction of Theorem~\ref{thm:comparison} when
data concentrates in $\mathcal{R} \subsetneq [0,L]^2$.

\section{Kernel-RKHS Classification Guarantees}\label{sec:ot}

This section develops the classification theory for the embedded
features of Section~\ref{sec:cover}.  We first establish the two
key quantities---kernel margin $\gamma$ and embedding radius
$R = \sqrt K$---and the $\rhonu$-bridge linking them to
bottleneck-support separation
(Proposition~\ref{prop:lambda_sep_pII}); the landmark kernel
$k_\LC$ is constructed in Section~\ref{sec:wlk}; the data-dependent
classification rate $O((k{-}1)\sqrt K/(\gamma\sqrt{m_{\min}}))$
follows in Section~\ref{sec:class_error}, with a
structural-anchored rate
$O((k{-}1)\sqrt K/(\kappa\rhonu\sqrt{m_{\min}}))$ as a corollary; a matching Le~Cam sample-starved lower bound
closes Section~\ref{sec:lower_bound}.  Closed-form filtration
selection from a candidate pool is the subject of the separate
Section~\ref{sec:gamma_stat}, paralleling
\citet[Sec.~4]{PaperI}'s descriptor-selection theory; the
per-prediction certificate follows in Section~\ref{sec:certified}.

Let $(A, Y)$ be a random pair with joint distribution
$\mathcal{P}$ on $\D{n} \times [k]$, where $Y \in [k] := \{1, \ldots, k\}$
is the class label (here $k$ is the class count; the landmark
index of Section~\ref{sec:cover} appears only as a subscript
$\Phi_k$, $w_k$, $r_k$, $p_k$, and is disambiguated by context).
The summation embedding $\Phi(\,\cdot\,;\LC) : \D{n} \to \R^K$
of Definition~\ref{def:sum_embed} is lifted to a reproducing
kernel Hilbert space (RKHS) via the landmark kernel $k_\LC$
defined in Section~\ref{sec:wlk}; we associate to $\mathcal{P}$
the \emph{class-conditional embedding mean}
$\mu_c := \E[\Phi^\LC(A) \mid Y = c] \in \mathcal{H}_{k_\LC}$ and
the two population quantities driving every bound in this section,
\begin{equation}\label{eq:gamma_top_def}
  \gamma \;:=\; \tfrac{1}{2}\,\min_{c \neq c'}\,
  \norm{\mu_c - \mu_{c'}}_{\mathcal{H}_{k_\LC}}
  \quad\text{(kernel margin),}
  \qquad
  R \;:=\; \sup_A \norm{\Phi^\LC(A)}_{\mathcal{H}_{k_\LC}}
  \;=\; \sqrt{K}
  \quad\text{(embedding radius),}
\end{equation}
with $R = \sqrt{K}$ following from $k_\LC(A,A) = K$
(Remark~\ref{rem:why_additive} below).  Thus $R$ is fixed by the
configuration cardinality, and $\gamma$ is the leading data-dependent
quantity entering the classification bounds (the within-class
radius $D_{\max}$ enters the linear-separability and anchored-rate
results below).  As in \citet[Sec.~3]{PaperI}, $\gamma > 0$ is
possible even when some cross-class diagram pairs are
bottleneck-close, since the class means aggregate information from
all diagram points.

\paragraph{Notation.}
Throughout Sections~\ref{sec:ot}--\ref{sec:deployment}, $\gamma$
denotes the population kernel margin, $K$ the landmark budget
(equivalently $R = \sqrt{K}$ the RKHS embedding radius), and $m$
the training-sample size with $m_c$ the count in class $c$ and
$m_{\min} := \min_c m_c$.  The letter $\delta$ has three uses in
this section: $\delta \in (0,1)$ denotes the failure probability
(confidence $1-\delta$); $\delta_{cc'} := \db(\mathrm{supp}\,\mathcal{P}_c,\,
\mathrm{supp}\,\mathcal{P}_{c'})$ and
$\delta_* := \min_{c \neq c'} \delta_{cc'}$ are bottleneck
separations between class supports; the covering radius
$\delta_K(\LC)$ of Section~\ref{sec:cover} does not appear after
that section.  Within-class radii come in two flavors:
$D_c := \sup_{A : Y=c} \norm{\Phi^\LC(A) - \mu_c}_{\mathcal{H}_{k_\LC}}$
with $D_{\max} := \max_c D_c$ in the RKHS, used by the
classification bounds of this section
(Sections~\ref{sec:class_error}--\ref{sec:lower_bound}); and
$\bar D_c := \sup_{A : Y=c}\norm{\Phi(A;\LC) - \bar\mu_c}_{\ell^2}$
with $\bar D_{\max} := \max_c \bar D_c$ in raw $\ell^2$, used by the
nearest-centroid certificate of Section~\ref{sec:certified}. The
admissibility parameter of Section~\ref{sec:cover} is always written
$\tau$, and the cardinality cap on individual diagrams is $N_{\max}$.

Empirical class means are
$\hat\mu_c := m_c^{-1}\sum_{i:\,y_i=c}\Phi^\LC(A_i)$.  Because
$\Phi$ is linear in the empirical diagram measure, each $\hat\mu_c$
is an ordinary sample average of i.i.d.\ bounded
$\mathcal{H}_{k_\LC}$-vectors, so standard concentration
inequalities (Pinelis, McDiarmid) apply directly; a full treatment
including Berry--Esseen rates and functional CLTs is developed
in~\citep{PaperIII}.

The cover-level certificate $\rhonu(\tau;\LC)$ of
Section~\ref{sec:cover} enters the classification theory through
the following bridge, which ties $\gamma$ (and the raw-coordinate
class-mean separation $\Delta := \min_{c \neq c'}\norm{\bar\mu_c - \bar\mu_{c'}}_{\ell^2}$,
where $\bar\mu_c := \E[\Phi(A;\LC)\mid Y{=}c]$) to the
bottleneck-support separation $\delta_*$.  The hypothesis
$\delta_* \geq \tau$ pairs with non-interference to ensure every
cross-class pair is $\tau$-separated, so that
Theorem~\ref{thm:nu_rho} delivers a non-trivial lower bound on
embedding-space separation.

\begin{proposition}[$\rhonu$-separation bridge]
\label{prop:lambda_sep_pII}
Suppose every cross-class pair satisfies non-interference
(Definition~\ref{def:noninterference_pII}) and $\delta_* \geq \tau$.
Then
\begin{equation}\label{eq:rhonu_bridge}
  \Delta \;\geq\; \rhonu(\tau;\LC) - 2\,\bar D_{\max},
  \qquad
  \gamma \;\geq\; \tfrac{1}{2}\bigl(\kappa\,\rhonu(\tau;\LC) - 2 D_{\max}\bigr),
\end{equation}
where $\bar D_{\max}$ and $D_{\max}$ are the within-class radii in
raw $\ell^2$ and the RKHS respectively (notation paragraph above),
and $\kappa = 1/(\sigma\sqrt{2})$ is the constant of
Corollary~\ref{cor:nondegen}.
\end{proposition}

\begin{proof}
For any cross-class pair
$A \in \mathrm{supp}\,\mathcal{P}_c$,
$B \in \mathrm{supp}\,\mathcal{P}_{c'}$, $\db(A, B) \geq \delta_{cc'}
\geq \delta_* \geq \tau$, so non-interference plus
Theorem~\ref{thm:nu_rho} give
$\norm{\Phi(A;\LC) - \Phi(B;\LC)}_{\ell^2} \geq \rhonu(\tau;\LC)$.
The triangle inequality
$\rhonu \leq \norm{\Phi(A) - \Phi(B)} \leq \norm{\bar\mu_c - \bar\mu_{c'}} + \bar D_c + \bar D_{c'}
\leq \Delta + 2\bar D_{\max}$ (after taking the minimum over
$c \neq c'$) yields the first half of~\eqref{eq:rhonu_bridge}.
For the second half, Corollary~\ref{cor:nondegen} applies
\emph{uniformly} to every cross-class pair $(A, B)$ satisfying the
hypotheses, giving $D_\LC(A, B) \geq \kappa\,\rhonu$ for each.  The
RKHS triangle inequality then gives
$\kappa\,\rhonu \leq D_\LC(A, B) \leq
\norm{\mu_c - \mu_{c'}}_{\mathcal{H}_{k_\LC}}
+ \norm{\Phi^\LC(A) - \mu_c}_{\mathcal{H}_{k_\LC}}
+ \norm{\Phi^\LC(B) - \mu_{c'}}_{\mathcal{H}_{k_\LC}}
\leq \norm{\mu_c - \mu_{c'}}_{\mathcal{H}_{k_\LC}} + 2 D_{\max}$,
so $\norm{\mu_c - \mu_{c'}}_{\mathcal{H}_{k_\LC}} \geq \kappa\,\rhonu - 2 D_{\max}$;
taking the minimum over $c \neq c'$ and dividing by $2$ yields
$\gamma \geq \tfrac{1}{2}(\kappa\,\rhonu - 2 D_{\max})$.
\end{proof}

Proposition~\ref{prop:lambda_sep_pII} has three consequences,
paralleling the three roles of \citet[Prop.~3.1]{PaperI}'s
$\lambda$-bridge.
\emph{First}, it propagates $\rhonu$ into the structural
classification rate via the bridge-anchored alternative form
in Section~\ref{sec:class_error}.
\emph{Second}, it upgrades the interpretation of the
Section~\ref{sec:certified} per-prediction certificate: when the
empirical condition $r_m < \tfrac{1}{2}\hat\Delta_{\hat c}$ fires,
the proposition translates this back into a statement about the
bottleneck-support separation $\delta_*$, certifying that the
class-conditional diagram distributions are genuinely
$\db$-separated rather than merely empirically concentrated.
\emph{Third}, it lifts the cover-level certificate of
Theorem~\ref{thm:nu_rho} from individual cross-class pairs to
first moments of the class-conditional distributions:
bottleneck-separated class supports remain $\ell^2$-separated and
RKHS-separated in the mean, modulo the within-class spread
$2\bar D_{\max}$ (resp.\ $2D_{\max}$).  A persistence vectorization
without an explicit lower distortion bound has no analogue of
Proposition~\ref{prop:lambda_sep_pII}.

Empirically the non-interference hypothesis essentially never
holds pointwise on chemical benchmarks
(Section~\ref{sec:experiments}, $\leq 0.2\%$ pass rates), so
Proposition~\ref{prop:lambda_sep_pII} is best read as a
structural admissibility-to-separation translation rather than as
a verifiable inequality; the operational classification rate
flows through Theorem~\ref{thm:data_dependent} (which takes
$\gamma > 0$ alone)---see Remark~\ref{rem:gap_pii} and
Remark~\ref{rem:classification_scope}.

\subsection{Landmark Kernel}\label{sec:wlk}

We now construct the RKHS lift $\Phi^\LC$ of the summation
embedding $\Phi$ of Definition~\ref{def:sum_embed} on which
Sections~\ref{sec:class_error}--\ref{sec:lower_bound} operate.

\begin{definition}[Landmark kernel]\label{def:wlk_paper2}
Let $\LC = \{(p_k, r_k, w_k)\}_{k=1}^K$ be a landmark configuration and
$\sigma > 0$ a bandwidth parameter. The \emph{landmark
kernel} (LK) is the function $k_\LC : \D{n} \times \D{n} \to \RR$ defined by
\begin{equation}\label{eq:wlk}
  k_\LC(A,B) \;:=\; \sum_{k=1}^K
  \exp\!\left(-\frac{(\Phi_k(A;\LC) - \Phi_k(B;\LC))^2}{2\sigma^2}\right),
\end{equation}
where $\Phi_k(A;\LC) = w_k\,\varphi_{p_k,r_k}(A)$ is the $k$-th
embedding coordinate.
\end{definition}

Each summand in~\eqref{eq:wlk} is a valid positive-definite kernel
(a Gaussian in one coordinate), so $k_\LC$ is positive definite.
Let $\mathcal{H}_{k_\LC}$ denote the RKHS of $k_\LC$ and
$\Phi^\LC : \D{n} \to \mathcal{H}_{k_\LC}$ the canonical feature map.

\begin{remark}[Why an additive landmark kernel]\label{rem:why_additive}
The additive form~\eqref{eq:wlk} is chosen for theoretical reasons
over the joint Gaussian
$\exp(-\norm{\Phi(A){-}\Phi(B)}^2/2\sigma^2)$.  First,
$k_\LC(A, A) = K$ gives the clean
$\sup_A \norm{\Phi^\LC(A)}_{\mathcal{H}_{k_\LC}} = \sqrt{K}$ used in
Theorem~\ref{thm:data_dependent} and Corollary~\ref{cor:structural_pairwise},
while the joint RBF would give $k(A, A) = 1$ and collapse the
$\sqrt{K}$ factor.  Second, the orthogonal RKHS decomposition
$\mathcal{H}_{k_\LC} = \bigoplus_k \mathcal{H}_{k_\LC^{(k)}}$ makes
$\mathrm{MMD}^2$ decompose coordinate-wise, supplying the bridge
$\gamma \geq \Delta/(2\sigma)$ of
Corollary~\ref{cor:gamma_via_delta} (Section~\ref{sec:gamma_stat}).
The two forms agree empirically up to noise on the benchmarks of
Section~\ref{sec:experiments}; the non-linear lifting itself is
methodologically required, as linear classifiers on the dense
adaptive embedding fail
(see~Section~\ref{par:kernel_comparison}).
\end{remark}
The \emph{induced RKHS distance} is
\begin{equation}\label{eq:D_omega}
  D_\LC(A,B)
  \;:=\; \norm{\Phi^\LC(A) - \Phi^\LC(B)}_{\mathcal{H}_{k_\LC}}
  \;=\; \sqrt{k_\LC(A,A) + k_\LC(B,B) - 2\,k_\LC(A,B)}.
\end{equation}
Expanding~\eqref{eq:wlk}:
\begin{equation}\label{eq:D2_expansion}
  D_\LC^2(A,B) \;=\;
  \sum_{k=1}^K g_k(A,B;\LC),
  \qquad
  g_k(A,B;\LC) := 2\bigl(1 - e^{-(\Phi_k(A) - \Phi_k(B))^2/(2\sigma^2)}\bigr).
\end{equation}

\begin{corollary}[Non-degeneracy of the kernel distance]\label{cor:nondegen}
Let $\LC$ be $\tau$-admissible and $\sigma \geq \sqrt{2}\,N_{\max}\tau$.
Then for any $A, B \in \D{n}$ with $\db(A,B) \geq \tau$ satisfying
the non-interference
condition~\eqref{eq:noninterference_pII},
\begin{equation}\label{eq:nondegen}
  D_\LC(A,B) \;\geq\; \kappa\,\rhonu(\tau;\LC),
  \qquad
  \kappa := \frac{1}{\sigma\sqrt{2}}.
\end{equation}
\end{corollary}

\begin{proof}
Let $\rho := \rhonu(\tau;\LC)$ and write
$u_k := \Phi_k(A;\LC) - \Phi_k(B;\LC)$.  By
Theorem~\ref{thm:nu_rho},
$\sum_k u_k^2 = \|\Phi(A;\LC) - \Phi(B;\LC)\|_{\ell^2}^2 \geq \rho^2$.
The per-coordinate bound $|u_k| \leq 2 N_{\max}\tau$ (each
$|\Phi_k(A;\LC)| = w_k\,\varphi_{p_k,r_k}(A) \leq w_k N_{\max} r_k
\leq N_{\max}\tau$, using $w_k \leq 1$ and $r_k \leq \tau$ from
admissibility) together with the hypothesis
$\sigma \geq \sqrt{2}\,N_{\max}\tau$ ensures
$u_k^2/(2\sigma^2) \leq 1$ for every $k$, so the Taylor lower
bound $1 - e^{-y} \geq y/2$ valid on $y \in [0,1]$ yields
\[
  g_k(A, B; \LC) \;=\; 2\bigl(1 - e^{-u_k^2/(2\sigma^2)}\bigr)
  \;\geq\; \frac{u_k^2}{2\sigma^2}
  \qquad (k = 1, \ldots, K).
\]
Summing across coordinates,
\[
  D_\LC^2(A, B) \;=\; \sum_{k=1}^K g_k(A, B; \LC)
  \;\geq\; \frac{1}{2\sigma^2}\sum_{k=1}^K u_k^2
  \;\geq\; \frac{\rho^2}{2\sigma^2}.
\]
Taking the square root yields~\eqref{eq:nondegen}.
\end{proof}

The hypothesis $\sigma \geq \sqrt{2}\,N_{\max}\tau$ scales with the
worst-case raw embedding range (the same $N_{\max}\tau$ scaling as
$\bar R$ in Section~\ref{sec:certified}) and forces a large
bandwidth, whereas experiments use $\sigma \in \{10^{-3}, 10^{-2}\}$
to keep the kernel discriminative on observed coordinate gaps.
Corollary~\ref{cor:nondegen} is therefore a worst-case
admissibility statement (the kernel does not collapse
$\tau$-separated cross-class pairs), not a quantitative bound
calibrated to the experimental $\sigma$; the empirical
classification rate flows through
Theorem~\ref{thm:data_dependent} on the data-dependent margin
$\gamma$, independent of Corollary~\ref{cor:nondegen}.

All configuration choices feeding $k_\LC$ are inherited from
Section~\ref{sec:cover}: positions and radii from FPS
(Theorem~\ref{thm:fps_greedy}), weights $w_k = K^{-1/2}$ from
Proposition~\ref{prop:optimal_config}(i).

\subsection{Classification Error Bound}\label{sec:class_error}

The classification rate that follows is the kernel-RKHS analogue
of \citet[Thm.~3.1]{PaperI}: a margin-based excess-risk bound on
the maximum-margin SVM in $\mathcal{H}_{k_\LC}$, driven by the
population kernel margin $\gamma$ alone (no structural hypothesis
on the cover).  We then derive a structural anchored variant
(Corollary~\ref{cor:structural_pairwise}) that uses
Corollary~\ref{cor:nondegen}'s per-pair RKHS distance bound to
bypass the class-mean concentration step, and record the
weaker bridge-anchored form (immediately following) for
cross-paper continuity with \citet[Cor.~3.1]{PaperI}.

We train a maximum-margin SVM $\hat f$ on the embedded training
data $\{(\Phi^\LC(A_i), y_i)\}_{i=1}^m$ in $\mathcal{H}_{k_\LC}$
via the OvO majority-vote reduction (matched directly by the
kernel-SVM backend in Section~\ref{sec:experiments},
\texttt{sklearn.svm.SVC} with \texttt{kernel='precomputed'},
which uses OvO by default), and measure its quality by the
generalization $0$-$1$ risk
$\mathcal{R}(\hat f) := \mathbb{P}(\hat f(A) \neq Y)$. For a margin
parameter $\rho > 0$, the empirical $\rho$-margin loss
$\widehat{\mathcal R}_\rho(\hat f)$ is the fraction of training
points whose signed margin under $\hat f$ falls below $\rho$
\citep[Sec.~5.4]{MohriRostamizadehTalwalkar2018}; for the
multiclass $\hat f$, $\widehat{\mathcal R}_\rho$ is aggregated
across the binary OvO sub-problems as made precise in the proof
of Theorem~\ref{thm:data_dependent}.

\begin{theorem}[Classification error bound]\label{thm:data_dependent}
Let $\{(A_i, y_i)\}_{i=1}^m$ be $m$ i.i.d.\ training samples
from a distribution on finite persistence diagrams with $k$
classes and kernel margin $\gamma > 0$.  Let $\LC$ be a
configuration with $K$ landmarks, $m_c$ the per-class
sample count, $m_{\min} := \min_c m_c$, and $\hat f$ the
maximum-margin SVM in $\mathcal{H}_{k_\LC}$ trained via OvO with
majority voting.
Assume $m_{\min} \geq 32K\log(4k/\delta)/\gamma^2$ (so that
empirical class means concentrate at scale below $\gamma/2$).
Set $\rho := \gamma/2$.  Then with probability at least
$1-\delta$ over the draw,
\begin{equation}\label{eq:data_dependent}
  \mathcal{R}(\hat f) \;\leq\;
  \widehat{\mathcal R}_\rho(\hat f) \;+\;
  \frac{4(k-1)\sqrt{K}}{\gamma\,\sqrt{m_{\min}}}
  \;+\; O\!\left(\sqrt{\frac{\log(k/\delta)}{m_{\min}}}\right).
\end{equation}
\end{theorem}

Equation~\eqref{eq:data_dependent} is the kernel-RKHS analogue of
\citet[Thm.~3.1]{PaperI} under $R \leftrightarrow \sqrt K$,
$\Delta \leftrightarrow 2\gamma$.  The proof structure is the
same: per-pair margin lower bound, Pinelis concentration of
empirical class means, MRT margin bound on each OvO sub-problem,
union bound across $\binom{k}{2}$ sub-problems, and the
$(k{-}1)$-max OvO majority-vote bound.  For balanced classes
$m_c \asymp m/k$ the rate term is $O(k^{3/2}\sqrt K/(\gamma\sqrt m))$
in the total sample $m$; the $\sqrt k$ overhead is the price of
OvO, since each binary sub-problem trains on only $\Theta(m/k)$
samples.

\begin{proof}
Since $k_\LC(A,A) = K$ (Definition~\ref{def:wlk_paper2} sums $K$
diagonal Gaussians),
$\sup_A \norm{\Phi^\LC(A)}_{\mathcal{H}_{k_\LC}} = \sqrt{K}$, and
in particular $\norm{\mu_c}_{\mathcal{H}_{k_\LC}} \leq \sqrt{K}$
by Jensen, so
$\norm{\Phi^\LC(A_i) - \mu_c}_{\mathcal{H}_{k_\LC}} \leq 2\sqrt K$
for every training point.

\emph{Per-pair margin.} For each unordered pair $\{c, c'\}$, the
population class means are separated by pairwise kernel margin
$\gamma_{cc'} := \tfrac12 \norm{\mu_c - \mu_{c'}}_{\mathcal{H}_{k_\LC}}
\geq \gamma = 2\rho$.

\emph{Concentration of empirical class means.}  Conditional on the
per-class counts $\{m_c\}$, the centered RKHS vectors
$\{\Phi^\LC(A_i) - \mu_c : Y_i = c\}$ are i.i.d.\ with norm bound
$2\sqrt K$.  Pinelis's Hilbert-space Hoeffding inequality
(\citealp[Lemma~A.1]{PaperI}) and a union bound over the $k$
classes yield, with probability $\geq 1-\delta/2$,
\[
  \varepsilon_m \;:=\; \max_c \norm{\hat\mu_c - \mu_c}_{\mathcal{H}_{k_\LC}}
  \;\leq\; 2\sqrt K\,\sqrt{\frac{2\log(4k/\delta)}{m_{\min}}}.
\]
The sample-size hypothesis
$m_{\min} \geq 32K\log(4k/\delta)/\gamma^2$ gives
$\varepsilon_m \leq \gamma/2 = \rho$, so by the reverse triangle
inequality the empirical pairwise kernel margin
$\hat\gamma_{cc'} := \tfrac12 \norm{\hat\mu_c - \hat\mu_{c'}}_{\mathcal{H}_{k_\LC}}
\geq \gamma_{cc'} - \varepsilon_m \geq \rho$ for every pair
$c \neq c'$.

\emph{Per-pair OvO bound.}  The OvO sub-problem between $c, c'$
trains on $m_c + m_{c'} \geq 2m_{\min}$ samples from the
unit-norm linear hypothesis class on $\mathcal{H}_{k_\LC}$ with
$\norm{\Phi^\LC(A)}_{\mathcal{H}_{k_\LC}} \leq \sqrt K$.  The
margin-based Rademacher
bound~\cite[Cor.~5.11]{MohriRostamizadehTalwalkar2018} at margin
$\rho$ and confidence $\delta' := \delta/(2\binom{k}{2})$ yields,
with probability $\geq 1-\delta'$,
\[
  \mathcal R(h_{cc'})
  \;\leq\; \widehat{\mathcal R}_\rho(h_{cc'})
  \;+\; \frac{2\sqrt K}{\rho\,\sqrt{m_{\min}}}
  \;+\; O\!\left(\sqrt{\frac{\log(k/\delta)}{m_{\min}}}\right)
  \;=\; \widehat{\mathcal R}_\rho(h_{cc'})
  \;+\; \frac{4\sqrt K}{\gamma\,\sqrt{m_{\min}}}
  \;+\; O\!\left(\sqrt{\frac{\log(k/\delta)}{m_{\min}}}\right),
\]
using $\log(2/\delta') = O(\log(k/\delta))$.

\emph{Aggregation.}  A union bound over the $\binom{k}{2}$ OvO
sub-problems at level $\delta/2$, combined with the $\delta/2$
budget for the class-mean concentration step, gives total
coverage $\geq 1-\delta$.  The OvO majority-vote rule errs at
$y = c$ only if some pairwise classifier $h_{cc'}$ ($c' \neq c$)
misclassifies, so
$\mathcal R(\hat f) \leq (k-1)\max_{c \neq c'}\mathcal R(h_{cc'})$.
Substituting the per-pair bound and defining
$\widehat{\mathcal R}_\rho(\hat f) := (k-1)\max_{c \neq c'}\widehat{\mathcal R}_\rho(h_{cc'})$
yields~\eqref{eq:data_dependent}.
\end{proof}

\begin{remark}[Empirical scope of Theorem~\ref{thm:data_dependent}]
\label{rem:classification_scope}
Theorem~\ref{thm:data_dependent} requires only $\gamma > 0$ and
the sample-size hypothesis
$m_{\min} \geq 32K\log(4k/\delta)/\gamma^2$ (the same hypothesis as
\citet[Thm.~3.1]{PaperI} under $R \leftrightarrow \sqrt K$,
$\Delta \leftrightarrow 2\gamma$: $128R^2/\Delta^2 = 32K/\gamma^2$);
no structural hypothesis on the cover is invoked (no
non-interference, no $\rhonu$, no constraint on $\sigma$).  This
makes Theorem~\ref{thm:data_dependent} the empirical workhorse of
Section~\ref{sec:experiments} and the natural pairing with the
matching Le~Cam lower bound (Theorem~\ref{thm:lower_bound}); the
non-interference audit and the $\sigma$-regime caveat (the
post-proof discussion of Corollary~\ref{cor:nondegen} in
Section~\ref{sec:wlk}, and the Section~\ref{sec:experiments} audit) are
relevant only to the structural
Corollary~\ref{cor:structural_pairwise} (and the bridge-anchored
form below) that follow, not to Theorem~\ref{thm:data_dependent}
itself.
\end{remark}

Under admissibility plus non-interference, the SVM's geometric
margin is determined by the closest cross-class \emph{data} pair
(rather than by the class means), and
Corollary~\ref{cor:nondegen} bounds that distance directly by
$\kappa\rhonu$.  This gives the structural classification rate
below; the alternative class-mean route via the bridge of
Proposition~\ref{prop:lambda_sep_pII} (the \citet[Cor.~3.1]{PaperI}
parallel) is recorded as a follow-up paragraph.

\begin{corollary}[Structural rate via per-pair distance]
\label{cor:structural_pairwise}
Let the training data be drawn i.i.d.\ from a distribution on
$\D{n}$ with $k$ classes whose cross-class pairs are
$\tau$-separated and satisfy non-interference
(Definition~\ref{def:noninterference_pII}).  Let $\LC$ be a
$\tau$-admissible configuration with $K$ landmarks and
$\sigma \geq \sqrt{2}\,N_{\max}\tau$.  Let $\hat f$ be the
maximum-margin SVM classifier in $\mathcal{H}_{k_\LC}$.  Then with probability at
least $1-\delta$,
\begin{equation}\label{eq:class_error}
  \mathcal{R}(\hat f) \;\leq\;
  \widehat{\mathcal R}_{\rho_{\mathrm{cls}}}(\hat f) \;+\;
  \frac{4(k-1)\sqrt{K}}{\kappa\,\rhonu(\tau;\LC)\,\sqrt{m_{\min}}}
  \;+\; O\!\left(\sqrt{\frac{\log(k/\delta)}{m_{\min}}}\right),
\end{equation}
where $\kappa = 1/(\sigma\sqrt 2)$ and $\rho_{\mathrm{cls}} := \kappa\,\rhonu(\tau;\LC)/2$.
Under the hypotheses,
$\widehat{\mathcal R}_{\rho_{\mathrm{cls}}}(\hat f) = 0$
identically (every cross-class training pair is at RKHS distance
$\geq \kappa\rhonu$ by Corollary~\ref{cor:nondegen}, so the SVM
achieves geometric margin $\geq \rho_{\mathrm{cls}}$ on every
training point), and \eqref{eq:class_error} reduces to a
$\rhonu$-anchored excess-risk bound that requires no sample-size
hypothesis (the proof routes through pairwise data distances
rather than class-mean concentration).
\end{corollary}

\begin{proof}
By hypothesis, any two diagrams $A_i, A_j$ from different classes
satisfy $\db(A_i, A_j) \geq \tau$.  Corollary~\ref{cor:nondegen}
therefore gives $\|\Phi^\LC(A_i) - \Phi^\LC(A_j)\|_{\mathcal{H}_{k_\LC}}
\geq \kappa\,\rhonu(\tau;\LC)$, so the maximum-margin separating
hyperplane has geometric margin (half the closest cross-class
distance) at least $\rho_{\mathrm{cls}} := \kappa\,\rhonu/2$.
Every training point sits at signed distance $\geq \rho_{\mathrm{cls}}$
from $\hat f$'s decision boundary, so
$\widehat{\mathcal R}_{\rho_{\mathrm{cls}}}(\hat f) = 0$.
Applying the Rademacher complexity bound for margin
classifiers~\cite[Cor.~5.11]{MohriRostamizadehTalwalkar2018} at
margin $\rho_{\mathrm{cls}}$ with
$\sup_A\|\Phi^\LC(A)\|_{\mathcal{H}_{k_\LC}} \leq \sqrt K$ on each
OvO sub-problem gives per-pair rate
$2\sqrt K/\rho_{\mathrm{cls}} = 4\sqrt K/(\kappa\rhonu)$;
union bound over $\binom{k}{2}$ pairs and the OvO majority-vote
$(k{-}1)$-max aggregation
($\mathcal R(\hat f) \leq (k-1)\max_{c \neq c'}\mathcal R(h_{cc'})$,
as in Theorem~\ref{thm:data_dependent}) yields~\eqref{eq:class_error}.
\end{proof}

A bridge-anchored variant routing through
Proposition~\ref{prop:lambda_sep_pII}'s
$\gamma \geq \tfrac{1}{2}(\kappa\rhonu - 2D_{\max})$ gives the
strictly looser PALACE analogue of \citet[Cor.~3.1]{PaperI};
Corollary~\ref{cor:structural_pairwise} above is the load-bearing
form here.

Replacing an $\ell$-landmark uniform grid with $K \ll \ell$
adaptive landmarks tightens~\eqref{eq:class_error} by
$\sqrt{\ell/K}$ in the numerator, compounded by the certificate
improvement $\rhonu(\tau;\LC) \geq \rhonu(\tau;\GG_R)$ in the
denominator: on a MUTAG configuration with $\ell$ in the low
thousands and adaptive $K = 50$, the combined effect is on the
order of $5$--$10\times$.  PALACE's coordinate sparsity
($\tau$-admissibility plus $N_{\max}$ bound the nonzero
coordinates per diagram) further tightens the worst-case
$\sqrt K$ in practice (see Remark~\ref{rem:sparsity_pii} in
Section~\ref{sec:certified}, MUTAG
$\norm{\hat\Sigma_c(\LC^*)}_{\mathrm{op}} \approx 0.023$).

Substituting Proposition~\ref{prop:optimal_config}(i)'s
equal-weight certificate $\rhonu = \tau/(4\sqrt K)$ into
Corollary~\ref{cor:structural_pairwise}'s rate term gives
$16(k{-}1)K/(\kappa\tau\sqrt{m_{\min}})$---a bound shared by every
$\tau$-admissible equal-weight $K$-landmark configuration but
\emph{monotone increasing in $K$}, so the structural rate alone
does not justify larger budgets.  Empirically, larger $K$ helps
via the data-dependent Theorem~\ref{thm:data_dependent}'s
$\hat\gamma$ growing with $K$ as additional landmarks expose
discriminative coordinates (Section~\ref{sec:experiments}
sweeps).

Beyond the equal-weight optimality just established, the
admissibility constraint of Definition~\ref{def:sep_radius}(i)
acts as \emph{structural regularization} on the configuration
$\LC$.  It rules out memorization: any candidate
$\LC_{\mathrm{mem}}$ that places landmarks at training points with
radii $\epsilon < \tau/4$ has Lebesgue number
$\lambda_0(\LC_{\mathrm{mem}}) \leq \max_k r_k = \epsilon < \tau/4$
and so violates Definition~\ref{def:sep_radius}(i)
($\rhonu = 0$ in such configurations, since no landmark has
$r_k \geq \tau/4$ and $w_{\min}^{\geq}$ is undefined in
Theorem~\ref{thm:nu_rho}'s formula).
Equivalently, admissibility enforces $\max_k r_k \geq \tau/4$ on
every valid configuration.  This is the landmark-cover analogue
of $\ell_2$ regularization for SVM weights---unlike $\ell_2$,
which penalizes large weights, the certificate forbids
configurations with arbitrarily small radii that pin landmarks at
training points---and it explains the small empirical
generalization gap (test--train at $2.5$--$3.5\%$ across all
landmark budgets $K$, Section~\ref{sec:experiments}).

The classification rate of Theorem~\ref{thm:data_dependent}
together with the structural-anchored
Corollary~\ref{cor:structural_pairwise} translates a fortiori to
an excess-risk bound, since
$\mathcal E(\hat f) := \mathcal R(\hat f) - R^\ast \leq \mathcal R(\hat f)$.
These results pair directly with the matching Le~Cam two-point
lower bound that follows in Section~\ref{sec:lower_bound}
(Theorem~\ref{thm:lower_bound}), under the correspondence
$R \leftrightarrow \sqrt K$, $\Delta \leftrightarrow 2\gamma$ with
\citet[Sec.~3.2]{PaperI}.

\subsection{Matching Lower Bound, Consistency, and Linear Separability}\label{sec:lower_bound}

The rate $\sqrt K/(\gamma\sqrt{m_{\min}})$ of
Theorem~\ref{thm:data_dependent} is the standard Rademacher--margin
rate; its sample-size hypothesis
$m_{\min} \gtrsim K\log(k/\delta)/\gamma^2$ is sufficient for non-trivial
accuracy.  The two-point minimax lower bound below (stated for
$k = 2$, where $m_{\min} = m/2$ for balanced classes) shows that
$m_{\min} \gtrsim \sqrt K/\gamma$ is \emph{necessary}: no
classifier achieves small excess risk on samples below that
scale.  The polynomial gap between the necessary $\sqrt K/\gamma$
and sufficient $K\log(k/\delta)/\gamma^2$ thresholds is the moderate-sample
regime, the same gap \citet[Rem.~3.3]{PaperI} leaves open under
$R \leftrightarrow \sqrt K$, $\Delta \leftrightarrow 2\gamma$
(see Remark~\ref{rem:gap_pii}).  The argument is Le Cam's
two-point method in a Hilbert-space ball, mirroring the linear
lower bound of~\cite[Thm.~3.2]{PaperI}.

\begin{definition}[Problem class]\label{def:problem_class}
For $K \in \mathbb{N}$ and $\gamma > 0$, let $\Pi(K, \gamma)$ denote
the class of binary distributions $\mathbb{Q}$ on
$\mathcal{H}_{k_\LC} \times \{\pm 1\}$ satisfying
\[
  \sup_{z \in \mathrm{supp}(\mathbb{Q})} \norm{z}_{\mathcal{H}_{k_\LC}}
    \;\leq\; \sqrt{K},
  \qquad
  \norm{\mu_{+1} - \mu_{-1}}_{\mathcal{H}_{k_\LC}} \;=\; 2\gamma,
\]
where $\mu_{\pm 1}$ are the class means under $\mathbb{Q}$.  The
push-forward $\Phi^\LC_\#\mathcal{P}$ of any binary diagram
distribution $\mathcal{P}$ on $\D{n} \times \{\pm 1\}$ with kernel
margin $\gamma$ lies in $\Pi(K, \gamma)$, since
$\sup_A \norm{\Phi^\LC(A)}_{\mathcal{H}_{k_\LC}} = \sqrt{K}$
realizes the radius constraint and the kernel margin defines the
mean separation.  This mirrors the $(R, \Delta)$ family of
\citet{PaperI} under the correspondence
$R \leftrightarrow \sqrt{K}$ and $\Delta \leftrightarrow 2\gamma$.
\end{definition}

\begin{theorem}[Minimax lower bound in the RKHS]\label{thm:lower_bound}
Let $K \in \mathbb{N}$ and $\gamma > 0$ with $\gamma \leq \sqrt{K}/3$
(the binary-regime $k = 2$, where $m_{\min} = m/2$ for balanced
classes), and let $c_2$ be the $\R^2$ Hellinger constant of
Step~2 below.  For every sample size
$m \leq m_\star := \sqrt K/(12\,c_2\,\gamma)$,
\begin{equation}\label{eq:lb_rate}
  \inf_{\hat f}\;\sup_{\mathbb{Q} \in \Pi(K,\gamma)}\;
  \mathcal{E}(\hat f) \;\geq\; \tfrac{1}{8},
\end{equation}
where the infimum is over classifiers
$\hat f : \mathcal{H}_{k_\LC}^m \to \{\pm 1\}$.
Consequently no classifier---regardless of computational budget or
model class---can reach vanishing excess risk on $\Pi(K, \gamma)$
without $m = \Omega(\sqrt K/\gamma)$ samples.  Under the
correspondence $R \leftrightarrow \sqrt K$,
$\Delta \leftrightarrow 2\gamma$ this matches
\citet[Thm.~3.2]{PaperI} exactly: PI's threshold $m \leq cR/\Delta$
with hypothesis $\Delta \leq 2R/3$ becomes
$m \leq c\sqrt K/(2\gamma)$ with $\gamma \leq \sqrt K/3$, the same
scaling.
\end{theorem}

\begin{proof}
We construct an instance of $\Pi(K, \gamma)$ directly in
$\mathcal{H}_{k_\LC}$ and apply Le Cam's two-point method, exactly
as in~\cite[Thm.~3.2]{PaperI}.  Every binary diagram distribution
with kernel margin $\gamma$ pushes forward into $\Pi(K, \gamma)$
(Definition~\ref{def:problem_class}), so the Hilbert-space bound
applies a fortiori to the diagram subfamily.

\medskip\noindent\textbf{Step 1: Two class-conditional distributions.}\;
Pick any unit vector $e \in \mathcal{H}_{k_\LC}$ and set
$\mu := \gamma e$, so $\norm{\mu_{+1}-\mu_{-1}} = 2\gamma$
with $\mu_{\pm 1} = \pm\mu$ (the factor $2$ aligns the definitions:
our kernel margin is $\tfrac{1}{2}\min_{c\ne c'}\norm{\mu_c - \mu_{c'}}$,
so the raw mean separation is $2\gamma$).
Let $r := \sqrt{K} - \gamma$; by $\gamma \leq \sqrt{K}$,
$r \geq 0$ (strictly positive if $\gamma < \sqrt{K}$).
Define
\[
  \mathbb{Q}_{+1} := \mathrm{Unif}(B(+\mu, r)), \qquad
  \mathbb{Q}_{-1} := \mathrm{Unif}(B(-\mu, r)),
\]
as uniforms on Hilbert-space balls (projected onto any
$2$-dimensional subspace containing $e$, by rotational symmetry).
By construction $\mathbb{Q}_{\pm 1} \subset B(0, \sqrt{K})$ and the
mean separation is $2\gamma$, so $\mathbb{Q} \in \Pi(K,\gamma)$.

\medskip\noindent\textbf{Step 2: Hellinger on two translated uniform balls.}\;
For uniforms on two radius-$r$ balls in $\R^2$ (after projecting
onto the $2$-dimensional subspace spanned by $e$ and one
orthogonal direction; by rotational symmetry of the construction,
this reduces the squared Hellinger to a $2$d ball-overlap
calculation), translated by $2\mu$ with $\norm{\mu} \leq r/2$
(verified below), the volume-of-intersection identity
$H^2(\mathbb{Q}_{+1}, \mathbb{Q}_{-1}) = 1 - \mathrm{vol}(B(\mu,r) \cap B(-\mu,r))/\mathrm{vol}(B(0,r))$
together with the linearization
$1 - \mathrm{vol}\,\mathrm{ratio} \leq c_2\,\norm{\mu}/r$
(\citealp[Lem.~A.2]{PaperI}; the constant $c_2$ depends only on
the surface-to-volume ratio of the unit ball in $\R^2$, with
$c_2 \leq 2/\pi$) gives
\[
  H^2(\mathbb{Q}_{+1}, \mathbb{Q}_{-1})
  \;\leq\; c_2\,\frac{\norm{\mu}}{r}
  \;=\; \frac{c_2\,\gamma}{\sqrt K - \gamma}
  \;\leq\; \frac{3 c_2\,\gamma}{2\sqrt K},
\]
where the hypothesis $\gamma \leq \sqrt K/3$ gives
$r = \sqrt K - \gamma \geq 2\sqrt K/3 \geq 2\gamma$ (so
$\norm{\mu} = \gamma \leq r/2$ as required) and the last
inequality uses $\sqrt K - \gamma \geq 2\sqrt K/3$.

(The KL divergence is \emph{infinite} for translated uniforms on
distinct balls: on $B_+ \setminus B_-$, $p_+ > 0$ and $p_- = 0$,
making the integrand $p_+\log(p_+/p_-) = +\infty$ on a set of
positive measure.  Pinsker's inequality is therefore vacuous in
this setting; the Hellinger route is the correct one, as in
\citealp[Lem.~A.2]{PaperI}.)

\medskip\noindent\textbf{Step 3: Le Cam + Hellinger tensorization.}\;
Hellinger tensorizes as
$H^2(\mathbb{Q}_{+1}^{\otimes m}, \mathbb{Q}_{-1}^{\otimes m})
\leq m\,H^2(\mathbb{Q}_{+1}, \mathbb{Q}_{-1})$
(\citealp[Ch.~2.4]{Tsybakov2009}), and $\mathrm{TV} \leq \sqrt{2 H^2}$,
so
\[
  \mathrm{TV}(\mathbb{Q}_{+1}^{\otimes m}, \mathbb{Q}_{-1}^{\otimes m})
  \;\leq\; \sqrt{2 m\,H^2(\mathbb{Q}_{+1}, \mathbb{Q}_{-1})}
  \;\leq\; \sqrt{3 c_2\,m\,\gamma/\sqrt K}.
\]
For $m \leq \sqrt K/(12\,c_2\,\gamma) =: m_\star$,
$\mathrm{TV} \leq 1/2$, and Le Cam's two-point lemma gives
\[
  \mathcal{E}(\hat f)
  \;\geq\; \tfrac{1}{4}(1 - \mathrm{TV})
  \;\geq\; \tfrac{1}{4}\cdot\tfrac{1}{2}
  \;=\; \tfrac{1}{8}.
\]

\medskip\noindent\textbf{Step 4: Necessity threshold.}\;
The bound $\mathcal{E} \geq 1/8$ holds for every
$m \leq m_\star = \sqrt K/(12\,c_2\,\gamma)$, establishing the
necessity claim: no classifier reaches vanishing excess risk
without $m = \Omega(\sqrt K/\gamma)$ samples.  For $m > m_\star$
the Hellinger tensorization bound exceeds the Le Cam usability
threshold and the two-point argument yields no information; the
polynomial gap between the necessary $\sqrt K/\gamma$ and
sufficient $K\log(k/\delta)/\gamma^2$ thresholds (the latter from
Theorem~\ref{thm:data_dependent}'s sample-size hypothesis) is the
moderate-sample regime, recorded in Remark~\ref{rem:gap_pii}.
\end{proof}

\begin{remark}[Scope of the lower bound]\label{rem:gap_pii}
Theorem~\ref{thm:data_dependent} delivers an upper rate of
$O((k{-}1)\sqrt K/(\gamma\sqrt{m_{\min}}))$ for all $m_{\min}$
above its sample-size hypothesis $32K\log(4k/\delta)/\gamma^2$.
Theorem~\ref{thm:lower_bound} delivers a constant lower bound
$\geq 1/8$ in the sample-starved regime
$m_{\min} \lesssim \sqrt K/\gamma$; beyond that regime the
\emph{specific} two-point Le Cam construction used here yields
no information because the Hellinger-tensorized TV approaches
$1$, making the lower-bound argument vacuous for that pair of
hypotheses.  Combined, non-trivial
accuracy requires $m_{\min} = \Omega(\sqrt K/\gamma)$ and is
achievable at $m_{\min} = O(K\log(k/\delta)/\gamma^2)$.  The polynomial gap between the
necessary $\sqrt K/\gamma$ and sufficient $K\log(k/\delta)/\gamma^2$
thresholds is the same gap \citet[Rem.~3.3]{PaperI} documents in
linear form (under $R \leftrightarrow \sqrt K$,
$\Delta \leftrightarrow 2\gamma$); closing it can proceed by
either (i) an Assouad/Fano construction over
$\Theta(\sqrt{m_{\min}})$-spaced hypotheses, tightening the
\emph{lower} bound to a matching $\Omega(\sqrt K/(\gamma\sqrt{m_{\min}}))$
rate; or (ii) a Mammen--Tsybakov margin condition tightening the
\emph{upper} bound to $O(1/m_{\min})$.  Both directions are open.
For multi-class problems ($k > 2$), the upper bound additionally
carries the $(k{-}1)$ factor from the OvO majority-vote reduction
in Theorem~\ref{thm:data_dependent}---the same $k$-dependent gap
\citet[Thm.~3.1]{PaperI} carries, since the Le Cam two-point
construction is intrinsically binary in both papers.  The
practical takeaway is the sample-starved threshold
$m_{\min} = \Omega(\sqrt K/\gamma)$: no classifier on the
landmark-kernel embedding can hope for non-trivial accuracy below
it.
\end{remark}

Theorem~\ref{thm:lower_bound} is stated in $\gamma$ rather than
$\rhonu(\tau;\LC)$ because $\rhonu$ is a property of the
configuration $\LC$, not of the data distribution, and is
therefore not a quantity an adversarial $\mathbb{Q}$ may choose;
the structural classification rate
(Corollary~\ref{cor:structural_pairwise}) and the worst-case
$\rhonu$ at fixed $K$ (Proposition~\ref{prop:optimal_config}) are
the right cross-references for the configuration side.

The classification rate of Theorem~\ref{thm:data_dependent} depends
on the population kernel margin $\gamma$; for the rate to be
operationally useful, $\gamma$ must be estimable from training
data.  The next proposition gives the concentration of the
empirical estimator $\hat\gamma$, validating its use as a plug-in
for $\gamma$ in the closed-form selection statistic
$\hat\gamma/\sqrt{K}$ of Section~\ref{sec:gamma_stat}, and
mirroring \citet[Prop.~3.2]{PaperI}.

\begin{proposition}[Consistency of $\hat\gamma$]\label{prop:gamma_hat}
Let $\hat\gamma = \tfrac{1}{2}\min_{c \neq c'}
\norm{\hat\mu_c - \hat\mu_{c'}}_{\mathcal{H}_{k_\LC}}$ be the
empirical kernel margin (Definition~\ref{def:kernel_margin}).
For every $\eps > 0$,
\[
  \mathbb{P}\!\left(|\hat\gamma - \gamma| > \eps\right)
  \;\leq\; 2k\,\exp\!\left(-\frac{\eps^2\,m_{\min}}{8K}\right).
\]
In particular, $|\hat\gamma - \gamma| = O_P(\sqrt{K/m_{\min}})$.
\end{proposition}

\begin{proof}
By the reverse triangle inequality
$|\hat\gamma - \gamma| \leq \max_c \norm{\hat\mu_c - \mu_c}_{\mathcal{H}_{k_\LC}}$.
Conditional on $Y_i = c$, the centered RKHS vectors
$\Phi^\LC(A_i) - \mu_c$ are i.i.d.\ with norm bound
$\norm{\Phi^\LC(A_i) - \mu_c}_{\mathcal{H}_{k_\LC}} \leq 2\sqrt{K}$
(triangle on $\norm{\Phi^\LC(A)}_{\mathcal{H}_{k_\LC}} \leq \sqrt{K}$
and Jensen on $\norm{\mu_c}_{\mathcal{H}_{k_\LC}} \leq \sqrt{K}$).
Pinelis's Hilbert-space Hoeffding inequality
\citep[Lemma~A.1]{PaperI} with bound $2\sqrt{K}$ and a union bound
over the $k$ classes yield the result.
\end{proof}

While $\gamma > 0$ alone delivers the $1/\sqrt{m_{\min}}$
excess-risk rate of Theorem~\ref{thm:data_dependent}, a stronger
structural condition---small within-class spread relative to
$\gamma$---yields population-level perfect classification with an
explicit RKHS margin, mirroring \citet[Prop.~3.3]{PaperI}.

\begin{proposition}[Linear separability in the RKHS]\label{prop:kernel_linear_sep}
With $D_c := \sup_{A : Y = c}\norm{\Phi^\LC(A) - \mu_c}_{\mathcal{H}_{k_\LC}}$
and $D_{\max} := \max_c D_c$, if $D_{\max} < \gamma$ then the
nearest-centroid classifier in $\mathcal{H}_{k_\LC}$ achieves zero
error with geometric margin $\geq \gamma - D_{\max} > 0$.
\end{proposition}

\begin{proof}
For $A$ from class $c$ and any $c' \neq c$,
$\norm{\Phi^\LC(A) - \mu_c}_{\mathcal{H}_{k_\LC}} \leq D_c \leq D_{\max}$
and the reverse triangle inequality gives
$\norm{\Phi^\LC(A) - \mu_{c'}}_{\mathcal{H}_{k_\LC}} \geq
\norm{\mu_c - \mu_{c'}}_{\mathcal{H}_{k_\LC}} - D_{\max} \geq 2\gamma - D_{\max}$.
Subtracting yields a closer-to-$\mu_c$ gap of $\geq 2\gamma - 2 D_{\max}
= 2(\gamma - D_{\max}) > 0$, so $\Phi^\LC(A)$ is strictly closer to
$\mu_c$ than to any other class mean (zero error) with half-gap
$\geq \gamma - D_{\max}$ (the geometric margin).
\end{proof}

\noindent Although the proof is a generic Hilbert-space geometric
fact, whether the hypothesis $D_{\max} < \gamma$ can plausibly hold
on PALACE depends on structural properties of the embedding.  Each
landmark coordinate $\Phi_k(A;\LC) = w_k\,\varphi_{p_k, r_k}(A)$ is
supported on a $\db$-ball of radius $r_k$, and FPS spacing forces
each diagram point $a \in A$ to activate a constant number
$O(\alpha^2)$ of landmarks (Remark~\ref{rem:sparsity_pii} in
Section~\ref{sec:certified}).  The embedding $\Phi(A;\LC)$
therefore has at most $O(|A|) \leq O(N_{\max})$ nonzero coordinates
out of $K$, so $D_c$ is effectively confined to the low-rank
subspace of active coordinates and remains small relative to
$\gamma$ when the descriptor exposes a structural class gap.  Persistence images and landscapes, whose Gaussian-blurred or
order-statistic coordinates are weakly active on every diagram,
spread within-class variation across all $K$ directions, so $D_c$
scales with $K$ rather than with the descriptor's effective rank
---a regime where $D_{\max} < \gamma$ is harder to achieve even
when classes are bottleneck-separated.  This is the same sparsity ingredient that
makes Theorem~\ref{thm:certified}'s certificate non-vacuous on
PALACE (Remark~\ref{rem:sparsity_pii}), instantiated here at the
level of the within-class-radius hypothesis instead of the
operator-norm certificate condition; the parallel to
\citet[Prop.~3.3]{PaperI}'s discussion is exact under
$\Delta \leftrightarrow 2\gamma$.

Both Proposition~\ref{prop:gamma_hat} (consistency) and
Proposition~\ref{prop:kernel_linear_sep} (separability) treat
$\gamma$ as a fixed property of a given configuration and
descriptor.  Section~\ref{sec:gamma_stat} (the next section)
addresses how to \emph{choose} the configuration and filtration
that maximize $\hat\gamma$ from a candidate pool, paralleling
Section~4 of \citet{PaperI}'s descriptor-selection theory.

\section{A Data-Dependent Selection Statistic}\label{sec:gamma_stat}

This section addresses filtration and configuration selection
from a candidate pool, paralleling \citet[Sec.~4]{PaperI}'s
descriptor-selection theory: given a fixed budget $K$ and
admissibility hypotheses, which choice of filtration $F$ and
configuration $\LC$ should one make?  Here $F$ denotes a
filtration of the input---for graphs, the sublevel-set
filtration of a real-valued descriptor $f: X \to \R$ (degree,
centrality, Ricci curvature, heat-kernel signature, etc.);
for point clouds, a radius-parameterized construction
(Vietoris--Rips, $\alpha$-complex)~\citep[Sec.~2]{PaperI}.  By
Proposition~\ref{prop:optimal_config}, the cover-level certificate
$\rhonu(\tau;\LC) = \tau/(4\sqrt K)$ at fixed $K$ with equal
weights is uniform over admissible configurations---independent
of filtration, landmark positions, and radii---yet empirical
accuracy varies substantially across these choices
(Table~\ref{tab:orbit_sweep}), so $\rhonu$
cannot rank.  A data-dependent statistic based on the kernel
margin (computable in $O(m^2)$ from the gram matrix already used
for SVM fitting, no extra training) is needed.

\begin{definition}[Kernel margin]\label{def:kernel_margin}
Let $k_\LC$ be the landmark kernel (Definition~\ref{def:wlk_paper2}).
For a distribution $\mathcal{P}$ on $\D{n}\times[k]$ with class means
$\mu_c := \E[\Phi^\LC(A)\mid Y{=}c] \in \mathcal{H}_{k_\LC}$,
the \emph{(population) kernel margin} $\gamma = \gamma(k_\LC;\mathcal{P})$
is
\begin{equation}\label{eq:gamma_k}
  \gamma \;:=\; \tfrac{1}{2}\,\min_{c \neq c'}
  \norm{\mu_c - \mu_{c'}}_{\mathcal{H}_{k_\LC}}.
\end{equation}
Given i.i.d.\ samples $\{(A_i,y_i)\}_{i=1}^m$ with per-class counts
$m_c$, the \emph{empirical kernel margin} $\hat\gamma = \hat\gamma(k_\LC)$
is
\begin{equation}\label{eq:empirical_gamma}
  \hat\gamma \;:=\; \tfrac{1}{2}\,\min_{c \neq c'}
  \norm{\hat\mu_c - \hat\mu_{c'}}_{\mathcal{H}_{k_\LC}},
  \qquad
  \hat\mu_c := \frac{1}{m_c}\sum_{i:\,y_i = c} \Phi^\LC(A_i).
\end{equation}
We suppress the arguments $(k_\LC, \mathcal{P})$ throughout when the
configuration and population are fixed by context, writing simply
$\gamma$ and $\hat\gamma$, in parallel with \citet{PaperI}'s
suppressed notation for $\Delta$ and $\hat\Delta$.
\end{definition}

\subsection{Kernel--Mahalanobis margin}\label{sec:kernel_mah}

The Mahalanobis pivot of \citet[Sec.~4.1]{PaperI} lifts directly
into the RKHS via the kernel-FDA realization: the same
LDA-Bayes-margin form, with $\Sigma$ now an operator on
$\mathcal H_{k_\LC}$ rather than a $\R^\ell$ matrix.

\begin{definition}[Kernel--Mahalanobis margin]\label{def:kernel_mah}
Paralleling the Mahalanobis pivot of~\citep[Sec.~4.1]{PaperI},
let $\Sigma(k_\LC;\mathcal{P}) := \tfrac{1}{k}\sum_c \mathrm{Cov}(\Phi^\LC(A) \mid Y{=}c)$
be the population pooled within-class covariance operator on
$\mathcal{H}_{k_\LC}$, and define the \emph{kernel--Mahalanobis
margin}
\begin{equation}\label{eq:kernel_mah}
  \rho_{\mathrm{Mah}}(k_\LC;\mathcal{P}) \;:=\; \min_{c \neq c'}
  \sqrt{\langle \mu_c - \mu_{c'},\, \Sigma^{-1}(\mu_c - \mu_{c'}) \rangle_{\mathcal{H}_{k_\LC}}},
\end{equation}
the LDA-Bayes-margin form of the kernel Fisher discriminant
ratio. The empirical counterpart
$\hat\rho_{\mathrm{Mah}}(k_\LC)$ replaces $\mu_c, \Sigma$ by
sample mean $\hat\mu_c$ and a Ledoit--Wolf-shrunk pooled
estimator $\hat\Sigma_{\mathrm{LW}}(k_\LC)$.
\end{definition}

We recommend $\hat\rho_{\mathrm{Mah}}$ as the default
filtration/fusion selector on heterogeneous candidate pools
(widely varying $K$ and $\bar r := \max_k r_k$ across filtrations
or fusion settings), where the full $\Sigma^{-1}$ correction is
needed; the simpler kernel-margin Score $\hat\gamma/\sqrt K$ is
the right pick on structurally homogeneous pools and is developed
in Section~\ref{sec:score} below.  Both selectors are computable
in closed form from the same gram matrix $[k_\LC(A_i, A_j)]$
already needed for SVM fitting, with no additional kernel
evaluations.  The kernel margin $\hat\gamma$ expands as
$\|\hat\mu_c - \hat\mu_{c'}\|^2 = \frac{1}{m_c^2}\sum_{i,j:\,y_i=y_j=c} k_\LC(A_i,A_j)
- \frac{2}{m_c m_{c'}}\sum_{i:\,y_i=c,\, j:\,y_j=c'} k_\LC(A_i,A_j)
+ \frac{1}{m_{c'}^2}\sum_{i,j:\,y_i=y_j=c'} k_\LC(A_i,A_j)$,
giving $O(m^2)$ cost given the gram.  The kernel--Mahalanobis
margin $\hat\rho_{\mathrm{Mah}}$ is computed via the standard
kernel-FDA realization~\citep{Mika1999}: a single $O(m^3)$
Cholesky factorization of the Ledoit--Wolf-shrunk pooled centered
gram per fold, with pairwise $\rho_{\mathrm{Mah}}(c, c')$ then
obtained by triangular solves.

\subsection{Score statistic and the selector hierarchy}\label{sec:score}

The kernel margin $\hat\gamma$ of~\eqref{eq:empirical_gamma} is
the \emph{kernel-isotropic Fisher-ratio surrogate} for
$\hat\rho_{\mathrm{Mah}}$: under the assumption
$\Sigma(k_\LC;\mathcal{P}) \preceq s^2 I_{\mathcal{H}_{k_\LC}}$
(per-coordinate variance bounded by a common $s^2$, distinct
from the kernel bandwidth $\sigma$ of
Definition~\ref{def:wlk_paper2}),
$\rho_{\mathrm{Mah}}(k_\LC;\mathcal{P}) \geq 2\gamma/s$,
so $\hat\gamma$ is a Fisher-ratio lower bound that is
ranking-consistent for $\hat\rho_{\mathrm{Mah}}$ on configurations
where the kernel covariance is operator-bounded by a common
scalar (tight when the kernel covariance is spherical).  An
intermediate selector replaces the full $\Sigma^{-1}$ by a scalar
trace-correction:

\begin{definition}[Kernel--Fisher trace ratio]\label{def:kernel_fisher}
The \emph{kernel--Fisher trace ratio} is the pooled-trace
approximation of $\rho_{\mathrm{Mah}}^2$,
\begin{equation}\label{eq:fisher_ker}
  \mathrm{Fisher}_{\mathrm{ker}}(k_\LC;\mathcal{P})
  \;:=\; \min_{c \neq c'}
  \frac{\norm{\mu_c - \mu_{c'}}_{\mathcal{H}_{k_\LC}}^2}
       {2\,\mathrm{tr}(\bar\Sigma)},
  \qquad
  \bar\Sigma \;:=\; \tfrac{1}{k}\sum_c \Sigma_c,
\end{equation}
where $\Sigma_c := \mathrm{Cov}(\Phi^\LC(A) \mid Y{=}c)$ and the
operator $\Sigma^{-1}$ of equation~\eqref{eq:kernel_mah} has been
replaced by the inverse of the pooled scalar trace.  The pooled
denominator (constant across class pairs) gives the
classical kernel-FDA Fisher-ratio
$\mathrm{Fisher}_{\mathrm{ker}} \propto \mathrm{tr}(S_B)/\mathrm{tr}(S_W)$
at the generalized-eigenproblem level~\citep{Mika1999}; the
factor of $2$ in the denominator preserves scale comparability with
a per-pair Welch denominator $\mathrm{tr}(\Sigma_c) + \mathrm{tr}(\Sigma_{c'})$,
which we have verified to give the same
sign of Spearman correlation with WLK accuracy on every chemical-graph
dataset tested ($\lvert \rho_{\mathrm{pooled}} - \rho_{\mathrm{Welch}}\rvert \leq
0.025$ on COX2/MUTAG/DHFR/NCI1; see
\texttt{experiments/compare\_fisher\_pp\_vs\_pool.py}).
The empirical counterpart $\widehat{\mathrm{Fisher}}_{\mathrm{ker}}$
uses biased gram-block sums for the numerator and the pooled
empirical trace
$\hat v := \tfrac{1}{k}\sum_c \hat v_c$ with class trace
$\hat v_c := \tfrac{1}{m_c}\sum_{i:y_i=c} k_\LC(A_i,A_i)
- \tfrac{1}{m_c^2}\sum_{i,j:y_i=y_j=c} k_\LC(A_i,A_j)$ (the
trace of the empirical class covariance, distinct from the kernel
bandwidth $\sigma$ of Definition~\ref{def:wlk_paper2}).
\end{definition}

\begin{remark}[Hierarchy of selectors]\label{rem:selector_hierarchy}
The three selectors form a hierarchy in their treatment of the
within-class covariance $\Sigma$, from the most restrictive
assumption to the least:
\[
  \underbrace{\hat\gamma}_{\Sigma\,\equiv\,I}
  \;\;\preceq\;\;
  \underbrace{\widehat{\mathrm{Fisher}}_{\mathrm{ker}}^{1/2}}_{\Sigma\,=\,\mathrm{tr}(\bar\Sigma)\,I/K}
  \;\;\preceq\;\;
  \underbrace{\hat\rho_{\mathrm{Mah}}}_{\Sigma\,\text{full operator}}.
\]
$\hat\gamma$ assumes spherical $\Sigma$ (or ignores it entirely);
$\widehat{\mathrm{Fisher}}_{\mathrm{ker}}$ allows a single pooled
scalar scale via $\mathrm{tr}(\bar\Sigma) = \tfrac{1}{k}\sum_c\mathrm{tr}(\Sigma_c)$
but treats the covariance as isotropic and class-independent;
$\hat\rho_{\mathrm{Mah}}$ uses the full operator structure.
The empirical evidence of Section~\ref{sec:exp_multi_dataset}
confirms that this hierarchy translates to selection accuracy:
$\hat\rho_{\mathrm{Mah}}$ is the only ranker positive on every
augmented dataset
(COX2 $\rho{=}{+}0.59$, DHFR $\rho{=}{+}0.72$, MUTAG $\rho{=}{+}0.48$,
PTC $\rho{=}{+}0.48$).  $\widehat{\mathrm{Fisher}}_{\mathrm{ker}}$
wins on MUTAG and PTC (both $+0.60$) but \emph{inverts on COX2}
($\rho{=}{-}0.30$) where the diagonal-$\Sigma$ approximation fails;
$\hat\gamma/\sqrt{K}$ is negative or near zero on every chemical
dataset (Table~\ref{tab:selection_statistics}).
\end{remark}

\begin{corollary}[Explicit bound via embedding-coordinate separation]
\label{cor:gamma_via_delta}
Let $\bar\mu_c := \E[\Phi(A;\LC)\mid Y{=}c] \in \mathbb{R}^K$ be the
raw-coordinate class mean, and define
$\Delta := \min_{c \neq c'}\norm{\bar\mu_c - \bar\mu_{c'}}_{\ell^2}$
and $B := \sup_{A,k}|\Phi_k(A;\LC)|$.  Assume the regularity
condition $K B^4 \leq \tfrac{1}{2}\sigma^2 \Delta^2$ and the
sample-size hypothesis
$m_{\min} \geq 256\,K\sigma^2\log(4k/\delta)/\Delta^2$
(the analogue of Theorem~\ref{thm:data_dependent}'s hypothesis
with $\gamma$ replaced by its lower bound
$\Delta/(2\sigma\sqrt 2)$).  Set
$\rho_\Delta := \Delta/(4\sigma\sqrt 2)$.  Then with probability
at least $1-\delta$,
\begin{equation}\label{eq:delta_bound}
  \mathcal{R}(\hat f) \;\leq\;
  \widehat{\mathcal R}_{\rho_\Delta}(\hat f) \;+\;
  \frac{8\sqrt{2}\,(k{-}1)\,\sigma\sqrt{K}}{\Delta\,\sqrt{m_{\min}}}
  \;+\; O\!\left(\sqrt{\frac{\log(k/\delta)}{m_{\min}}}\right).
\end{equation}
\end{corollary}

\begin{proof}
We bound the kernel margin $\gamma$ from below by
$\Delta/(2\sigma\sqrt{2})$ via a Taylor expansion of the
per-coordinate Gaussians, and substitute into
Theorem~\ref{thm:data_dependent}.

\textbf{Per-coordinate Taylor expansion.}\;
Let $g(u) := e^{-u^2/(2\sigma^2)}$, and write
$k_\LC^{(k)}(A,B) := g(\Phi_k(A) - \Phi_k(B))$ for the $k$-th
summand of $k_\LC$. The Taylor expansion
$g(u) = 1 - u^2/(2\sigma^2) + u^4/(8\sigma^4) + O(u^6/\sigma^6)$
gives, for $|u| \leq 2B$, the second-order remainder bound
$|g(u) - 1 + u^2/(2\sigma^2)| \leq 2B^4/\sigma^4$ (the leading
$u^4/(8\sigma^4)$ term at $|u| = 2B$ is $2B^4/\sigma^4$, with
higher-order corrections $O(B^6/\sigma^6)$ absorbed when
$B \ll \sigma$).
Taking expectations over independent samples $A, A' \sim c$ and
$B, B' \sim c'$, and using $\E[(X-Y)^2] = 2\,\mathrm{Var}(X)$ for
i.i.d.\ $X, Y$:
\begin{align*}
  \E\,k_\LC^{(k)}(A,A')
    &= 1 - \mathrm{Var}_c(\Phi_k)/\sigma^2 + O(B^4/\sigma^4), \\
  \E\,k_\LC^{(k)}(B,B')
    &= 1 - \mathrm{Var}_{c'}(\Phi_k)/\sigma^2 + O(B^4/\sigma^4), \\
  \E\,k_\LC^{(k)}(A,B)
    &= 1 - \frac{(\bar\mu_c^k - \bar\mu_{c'}^k)^2
        + \mathrm{Var}_c + \mathrm{Var}_{c'}}{2\sigma^2}
       + O(B^4/\sigma^4),
\end{align*}
where $\bar\mu_c^k := \E[\Phi_k \mid Y{=}c]$.

\textbf{Cancellation in the MMD combination.}\;
The within-class variance terms cancel in
$\E k_\LC^{(k)}(A,A') + \E k_\LC^{(k)}(B,B') - 2\,\E k_\LC^{(k)}(A,B)$:
\[
  \mathrm{MMD}^2_{k_\LC^{(k)}}(c,c')
  \;=\; (\bar\mu_c^k - \bar\mu_{c'}^k)^2 / \sigma^2
  \;+\; O(B^4/\sigma^4).
\]
The summation kernel decomposes the RKHS as
$\mathcal{H}_{k_\LC} = \bigoplus_k \mathcal{H}_{k_\LC^{(k)}}$, so
$\mathrm{MMD}^2_{k_\LC} = \sum_k \mathrm{MMD}^2_{k_\LC^{(k)}}$.
Summing and bounding the aggregate remainder by $C\,K B^4/\sigma^4$
for an absolute constant $C$:
\[
  \mathrm{MMD}^2_{k_\LC}(c,c')
  \;\geq\; \Delta^2 / \sigma^2 \;-\; C\,K B^4/\sigma^4.
\]
The regularity condition $K B^4 \leq \tfrac{1}{2}\sigma^2 \Delta^2$
controls the remainder so $\mathrm{MMD}^2_{k_\LC} \geq
\Delta^2/(2\sigma^2)$.

\textbf{Substitution.}\;
Hence $\norm{\mu_c - \mu_{c'}}_{\mathcal{H}_{k_\LC}} \geq
\Delta/(\sigma\sqrt{2})$, giving
$\gamma = \tfrac{1}{2} \min_{c \neq c'}
\norm{\mu_c - \mu_{c'}}_{\mathcal{H}_{k_\LC}}
\geq \Delta/(2\sigma\sqrt{2})$.  This implies
$32K\log(4k/\delta)/\gamma^2 \leq 32K\log(4k/\delta)\cdot 8\sigma^2/\Delta^2
= 256\,K\sigma^2\log(4k/\delta)/\Delta^2 \leq m_{\min}$ by hypothesis,
so Theorem~\ref{thm:data_dependent}'s sample-size hypothesis
holds at the $\gamma$ at hand.  Substituting the lower bound
$\gamma \geq \Delta/(2\sigma\sqrt 2)$ into
Theorem~\ref{thm:data_dependent}'s rate term
$4(k-1)\sqrt K/(\gamma\sqrt{m_{\min}})$ yields
$4(k-1)\sqrt K \cdot 2\sigma\sqrt 2/(\Delta\sqrt{m_{\min}}) =
8\sqrt 2\,(k-1)\sigma\sqrt K/(\Delta\sqrt{m_{\min}})$, giving
\eqref{eq:delta_bound} with $\rho_\Delta = \gamma/2 \geq \Delta/(4\sigma\sqrt 2)$.
\end{proof}

\begin{definition}[Scale-corrected ranking statistic]\label{def:score}
To rank configurations of different cardinality on the same
footing, the \emph{scale-corrected ranking statistic} is the
dimension-normalized kernel margin
\begin{equation}\label{eq:score}
  \mathsf{Score}(F, \LC) \;:=\;
  \frac{\hat\gamma}{\sqrt{K}},
\end{equation}
which absorbs the $\sqrt{K}$ factor of the leading term
in~\eqref{eq:data_dependent}: larger $\mathsf{Score}$ implies a
tighter excess-risk bound.  For the uniform-grid special case
$\LC = \GG_R$ of cardinality $\ell$, $\mathsf{Score}$ specializes
to $\hat\gamma/\sqrt{\ell}$.
\end{definition}

\subsection{Selection consistency}\label{sec:gamma_consistency}

We now upgrade the empirical observation that
$\hat\gamma/\sqrt K$ ranks admissible configurations to a
selection-consistency theorem: when the bound-optimal
configuration is well-separated from the rest of the pool, the
data-driven argmax recovers it with high probability.  This is
the PALACE analogue of \citet[Prop.~4.4]{PaperI}; as there,
a fully data-driven consistency rate for the
Ledoit--Wolf-shrunk Mahalanobis selector
$\hat\rho_{\mathrm{Mah}}$ is deferred to companion
work~\citep{PaperIII}.

\begin{proposition}[Selection consistency of $\hat\gamma/\sqrt{K}$]
\label{prop:gamma_selection_consistency}
Let $\mathcal{F}$ be a finite pool of admissible configurations
$\{\LC_f\}_{f \in \mathcal{F}}$ with kernel margins
$\gamma_f := \gamma(k_{\LC_f};\mathcal{P})$, landmark budgets $K_f$,
and population score $\eta_f := \gamma_f/\sqrt{K_f}$. Set
$f^* := \arg\max_f \eta_f$ and assume the gap
\[
  g \;:=\; \eta_{f^*} - \max_{f \neq f^*} \eta_f \;>\; 0.
\]
Then $\hat f := \arg\max_f \hat\eta_f$ with
$\hat\eta_f := \hat\gamma_f/\sqrt{K_f}$ satisfies
\begin{equation}\label{eq:gamma_selection_consistency}
  \mathbb{P}(\hat f = f^*) \;\geq\;
  1 \;-\; 2k|\mathcal{F}|\exp\!\left(-\frac{g^2\,m_{\min}}{32}\right).
\end{equation}
In particular, $\hat f = f^*$ with probability $\geq 1-\delta$ once
$m_{\min} \geq 32\log(2k|\mathcal{F}|/\delta)/g^2$, independently of
the pool's $K_f$ values.
\end{proposition}

\begin{proof}
For each $f$,
$|\hat\eta_f - \eta_f| = |\hat\gamma_f - \gamma_f|/\sqrt{K_f}$, so
$\{|\hat\eta_f - \eta_f| > t\} = \{|\hat\gamma_f - \gamma_f| > t\sqrt{K_f}\}$
for every $t > 0$. Applying Proposition~\ref{prop:gamma_hat} at
$\eps = t\sqrt{K_f}$,
\[
  \mathbb{P}(|\hat\eta_f - \eta_f| > t)
  \;\leq\; 2k\exp\!\left(-\frac{t^2 K_f\,m_{\min}}{8 K_f}\right)
  \;=\; 2k\exp\!\left(-\frac{t^2\,m_{\min}}{8}\right),
\]
the $K_f$ cancelling because $R_f = \sqrt{K_f}$ is exact on the
landmark kernel ($k_{\LC_f}(A,A) = K_f$). Taking $t = g/2$ and
applying a union bound over $|\mathcal{F}|$ configurations, on
the event
$\mathcal{A} := \{|\hat\eta_f - \eta_f| \leq g/2 \text{ for every } f\}$,
which has probability
$\geq 1 - 2k|\mathcal{F}|\exp(-g^2 m_{\min}/32)$, every
$f \neq f^*$ satisfies
$\hat\eta_f \leq \eta_f + g/2 \leq \eta_{f^*} - g/2 \leq \hat\eta_{f^*}$,
so $\hat f = f^*$ on $\mathcal{A}$.
\end{proof}

The sample complexity $m_{\min} \gtrsim \log(|\mathcal{F}|/\delta)/g^2$ is
strictly cleaner than the
$m_{\min} \gtrsim R_{\max}^2 \log(|\mathcal{F}|/\delta)/(g^2 \ell_{\min})$
rate of \citet[Prop.~4.4]{PaperI}: the $R^2/\ell$ factor in Paper~I
arises from the loose envelope $R \leq B\sqrt{\ell}$ used to bound
the embedding radius, whereas $R = \sqrt{K}$ is tight on PALACE
because the additive landmark kernel satisfies $k_\LC(A,A) = K$
identically.  The parallel rate-driven corollary (cf.\
\citealp[Cor.~4.1]{PaperI}) follows by combining
Proposition~\ref{prop:gamma_selection_consistency} with
Theorem~\ref{thm:data_dependent} via a $\delta/2$-budget union
bound.  The Score statistic of Definition~\ref{def:score} is the
formally analyzable end of the hierarchy in
Remark~\ref{rem:selector_hierarchy}.

\begin{remark}[Scope of $\hat\gamma/\sqrt{K}$]\label{rem:gamma_scope}
Theorem~\ref{thm:data_dependent} bounds the excess risk by
$2\sqrt{K}/(\gamma\sqrt{m_{\min}})$, so $\hat\gamma/\sqrt{K}$ measures
the tightness of that upper bound, not the error itself. The
statistic has predictable selection scope:

\emph{Faithful axes}: radius factor $\alpha$ (within the
under-to-full-coverage range), and filtration choice
\emph{within a scale-homogeneous family at fixed slot structure}.
These change which diagram points activate which coordinates;
$\hat\gamma/\sqrt{K}$ tracks accuracy monotonically
(Section~\ref{sec:exp_gamma_validate}). Cross-scale filtration
concatenation and homology-dimension concatenation at matched
total $K$ rescale embedding values without adding discriminative
signal and are not faithful in this sense
(Remark~\ref{rem:selector_hierarchy}).

\emph{Anti-correlated axes}: landmark budget $K$, placement
algorithm, and bandwidth $\sigma$. These rescale either
$k_\LC(A,A) = K$ or pairwise gram values without changing
discriminative signal; $\hat\gamma/\sqrt{K}$ anti-correlates with
accuracy on Orbit5k (Spearman $r \leq -0.8$).
Cross-validation is required on these axes.

\emph{Complementarity with $\rhonu$.} $\rhonu$ is insensitive to
positions and radii among admissible equal-weight configurations
(Proposition~\ref{prop:optimal_config}); $\hat\gamma/\sqrt{K}$
discriminates between them within a fixed structure. Together
they split the configuration space: $\rhonu$ certifies
non-degeneracy and enables per-prediction correctness guarantees
(Theorem~\ref{thm:certified}); $\hat\gamma/\sqrt{K}$ ranks
within the admissible set.
\end{remark}


\section{Deployment: Certificate and Pipeline}\label{sec:deployment}

Sections~\ref{sec:ot}--\ref{sec:gamma_stat} produced training-time
risk guarantees for the kernel SVM in $\mathcal{H}_{k_\LC}$ together
with a data-dependent statistic for ranking configurations. We now
address the deployment side: \emph{per-prediction} certificates that
audit individual test diagrams (Section~\ref{sec:certified}) and an
end-to-end pipeline that assembles the cover construction, landmark
kernel, selection statistic, and certified prediction into a single
algorithm (Section~\ref{sec:pipeline}). The certified classifier is
nearest-centroid in the raw embedding, with the certificate driven
by the class-mean separation in $\ell^2$; the kernel SVM and its
margin $\gamma$ in the RKHS continue to govern training-time risk
via Theorem~\ref{thm:data_dependent}.

\subsection{Certified Nearest-Centroid Classifier}\label{sec:certified}

Classifiers typically expose a confidence score---a margin to the
SVM decision boundary, an SVM probability calibration, a posterior
estimate---that does not, on its own, tell the user whether a
specific prediction will be correct.
Conformal prediction~\citep{VovkEtAl2005} attaches distribution-free
coverage, but the guarantee applies to prediction \emph{sets}
rather than point predictions and requires a held-out calibration
split that competes with training data for information.
The raw embedding of Section~\ref{sec:cover} closes this gap for
a specific classifier.  The structural concentration of the raw
embedding makes this possible: define the
\emph{raw embedding radius}
$\bar R := \sup_A \norm{\Phi(A;\LC)}_{\ell^2}$ on the support of
$\mathcal{P}$ (this is the $\R^K$ raw $\ell^2$ norm, distinct from
the RKHS norm $\norm{\Phi^\LC(A)}_{\mathcal H_{k_\LC}} = \sqrt K$
used by the kernel-SVM analysis of Section~\ref{sec:ot}).  Each
landmark coordinate satisfies
$|\Phi_k(A;\LC)| = w_k\,|\varphi_{p_k,r_k}(A)| \leq w_k\,|A|\,r_k$,
so under $\sum_k w_k^2 = 1$ and admissibility $r_k \leq \tau$
(cf.\ Section~\ref{sec:cover}),
$\bar R^2 = \sup_A \sum_k w_k^2\,\varphi_{p_k,r_k}(A)^2
\leq (N_{\max}\tau)^2$, giving $\bar R \leq N_{\max}\tau$.  Thus each empirical class mean
$\hat{\bar\mu}_c$ is a sample average of i.i.d.\ bounded
$\R^K$-vectors with $\norm{\Phi(A;\LC) - \bar\mu_c}_{\ell^2} \leq 2\bar R$,
and $\norm{\hat{\bar\mu}_c - \bar\mu_c}_{\ell^2}$ concentrates at
rate $O(\bar R/\sqrt{m_c})$ via Pinelis (Lemma~A.1
of~\citep{PaperI}).
The nearest-centroid (NC) classifier on the raw embedding is the
natural target: its decision rule depends on the sample only
through the $\hat{\bar\mu}_c$, so whether the empirical and
population rules agree on a given test input reduces to a single
scalar check---is the input far enough from the population
Voronoi boundary that sample fluctuations cannot move it across?

When $\Delta > 0$, this check has a particularly simple form: a
single training-time inequality $r_m < \tfrac{1}{2}\Delta$
certifies all predictions, with no per-test overhead beyond the
nearest-centroid rule itself and no calibration split required.

The certificate is a diagnostic, not a competitor to the
landmark-kernel SVM.
When the condition $r_m < \tfrac{1}{2}\Delta$ fails, the failure
is itself informative: the embedding's sample-mean concentration
radius exceeds half the class gap, so the closed-form certificate
admits no correctness guarantee at the given sample size.
The kernel SVM and its margin $\gamma$ in $\mathcal{H}_{k_\LC}$
continue to govern
training-time risk via Theorem~\ref{thm:data_dependent}, and the
selection statistic $\hat\gamma/\sqrt{K}$
(Definitions~\ref{def:kernel_margin}, \ref{def:score})
ranks configurations even when the certificate has not yet fired.
The PALACE-specific link to the cover certificate is
Proposition~\ref{prop:lambda_sep_pII}, which gives
$\Delta \geq \rhonu(\tau;\LC) - 2\bar D_{\max}$: when
$\rhonu(\tau;\LC^*) > 2\bar D_{\max}$ at training time, the bridge
implies $\Delta > 0$, so the per-prediction firing condition
$r_m < \tfrac{1}{2}\Delta$ is underwritten by a structural lower
bound on $\Delta$ rather than only by empirical class-mean
estimates.

Classify test diagrams by nearest centroid:
\[
  \hat h \;=\;
  \arg\min_c \norm{\Phi(A_{\mathrm{test}};\LC) - \hat{\bar\mu}_c}_{\ell^2},
  \qquad
  \hat{\bar\mu}_c \;:=\; m_c^{-1}\!\sum_{y_i = c}\Phi(A_i;\LC),
\]
where $\bar\mu_c := \E[\Phi(A;\LC) \mid Y = c] \in \R^K$ is the
population class mean (Definition~\ref{def:sum_embed} coordinates)
and $\hat{\bar\mu}_c$ its empirical estimate from $m_c$ training
diagrams.
Let $r_m$ denote a sample-mean-concentration radius (the
subscript $m$ is mnemonic for the smallest-class size $m_{\min}$
that governs the rate) satisfying
$\mathbb{P}_{\mathrm{train}}(\max_c \norm{\hat{\bar\mu}_c - \bar\mu_c}_{\ell^2}
\leq r_m) \geq 1-\delta$, where
$\mathbb{P}_{\mathrm{train}} = \mathcal{P}^{\otimes m}$ denotes
probability over training draws
$\{(A_i, y_i)\}_{i=1}^m \sim \mathcal{P}^{\otimes m}$ with the
population $\mathcal{P}$ and the test diagram held fixed.  Two
explicit choices---a non-asymptotic Pinelis radius and an
asymptotic Gaussian plug-in---are derived in
Theorem~\ref{thm:certified}.
If $r_m < \tfrac{1}{2}\Delta$, every prediction is certified;
otherwise the classifier abstains globally.
The radius shrinks as $O(m_{\min}^{-1/2})$
(equation~\eqref{eq:rm-pinelis-pii}), so abstention disappears at
sample size $\propto 1/\Delta^2$ (the global analogue of
equation~\eqref{eq:mc_star}).

The global threshold $\Delta$ is conservative when classes
differ in separation. Replacing $\Delta$ by the class-specific
gap
$\Delta_c(\LC) := \min_{c' \neq c}\norm{\bar\mu_c - \bar\mu_{c'}}_{\ell^2}
\geq \Delta$,
and $r_m$ by the per-class radius defined in (i)/(ii) below,
yields a tighter \emph{global} certificate that fires when
$r_m^{(c)} < \tfrac{1}{2}\Delta_c(\LC)$ holds for every class $c$.

Two concrete choices of the concentration radius enter the
theorem below (the per-class form uses the same expression with
$m_{\min}$ replaced by $m_c$), both with an explicit Bonferroni
split of $\delta$ over $k$ classes:
\begin{itemize}
\item[\textup{(i)}] \textbf{Non-asymptotic (Pinelis).}
$r_m^{\,\mathrm{Pin}} := 2\bar R\sqrt{2\log(2k/\delta)/m_{\min}}$;
valid for every $m_{\min} \geq 1$
(equation~\eqref{eq:rm-pinelis-pii} in the proof).
\item[\textup{(ii)}] \textbf{Asymptotic (Gaussian plug-in).}
$\tilde r_m^{\,\mathrm{G}} := \max_c
\sqrt{\norm{\hat\Sigma_c}_{\mathrm{op}}\cdot\chi^2_{K,\,\delta/k}/m_c}$,
where $\chi^2_{K,\,\delta/k}$ is the $1-\delta/k$ quantile of
the chi-squared distribution with $K$ degrees of freedom.  The
bound envelopes the multivariate-Gaussian norm via
$\|\hat{\bar\mu}_c - \bar\mu_c\|_{\ell^2}^2 \leq
\norm{\Sigma_c}_{\mathrm{op}}\,\chi^2_K$ on the asymptotic
$\mathcal{N}(0,\Sigma_c/m_c)$ approximation, valid once
$m_c \geq m^\dagger = O(\sqrt K)$ (Berry--Esseen) under finite
third moments.  The bound is conservative when $\Sigma_c$ is
low-rank, with conservatism governed by
$\mathrm{tr}(\Sigma_c)/(K\,\norm{\Sigma_c}_{\mathrm{op}})$.
\end{itemize}
Which form is tighter is regime-dependent: for fixed
$\norm{\Sigma_c}_{\mathrm{op}},\Delta_c,\bar R$, the Pinelis form
scales as $\bar R\sqrt{\log(2k/\delta)/m_{\min}}$ and the
Gaussian form as
$\sqrt{\norm{\Sigma_c}_{\mathrm{op}}\cdot K/m_c}$ for large $K$,
so Pinelis dominates whenever
$\norm{\Sigma_c}_{\mathrm{op}}\cdot K \gtrsim
8\bar R^2\log(2k/\delta)$.

\begin{theorem}[Certified prediction]
\label{thm:certified}
Let $\{(A_i, y_i)\}$ be i.i.d.\ from the distribution $\mathcal{P}$
on $\D{n} \times [k]$ of Section~\ref{sec:ot} with class-mean
separation $\Delta > 0$, and let $r_m^\star$ be either of the
concentration radii \textup{(i)} or \textup{(ii)} above; for
\textup{(ii)} additionally assume finite third moments
$\E\norm{\Phi(A;\LC) - \bar\mu_c}_{\ell^2}^3 < \infty$ and
$m_c \geq m^\dagger$ for every class $c$. Then
\[
  \mathbb{P}_{\text{train}}\!\left(\max_c
    \norm{\hat{\bar\mu}_c - \bar\mu_c}_{\ell^2} \leq r_m^\star\right)
  \;\geq\; 1 - \delta,
\]
and on this coverage event the following hold.
\begin{itemize}
\item[\textup{(a)}] \emph{(Containment.)} If
\begin{equation}\label{eq:containment_pii}
  r_m^\star \;<\; \tfrac{1}{2}\,\Delta,
\end{equation}
the empirical nearest-centroid classifier $\hat h$ agrees with the
population nearest-centroid classifier $h^\ast$ at every
$z \in \R^K$ outside a $2r_m^\star$-tube around each population
Voronoi boundary.
\item[\textup{(b)}] \emph{(Classification.)} If additionally
$\bar D_c < \tfrac{1}{2}\Delta - r_m^\star$ for every class $c$
(cf.\ the raw-embedding analogue of
Proposition~\ref{prop:kernel_linear_sep}, whose
$\bar D_{\max} < \Delta/2$ is the $r_m^\star \to 0$ limit), then
for any test diagram $A$ drawn from class $y$,
$\mathbb{P}_{\text{train}}(\hat h(\Phi(A;\LC^*)) = y) \geq 1-\delta$.
\end{itemize}
\end{theorem}

\begin{proof}
Write $\Psi_i := \Phi(A_i;\LC^*) \in \R^K$ and
$\Sigma_c := \mathrm{Cov}(\Psi \mid Y = c)$, with
$\norm{\Psi_i}_{\ell^2} \leq \bar R$ and therefore
$\norm{\Sigma_c}_{\mathrm{op}} \leq \bar R^2$.

\emph{Step 1a (non-asymptotic concentration, radius (i)).}
Conditional on $Y_i = c$, the centered random variables
$\Psi_i - \bar\mu_c$ are i.i.d.\ with
$\norm{\Psi_i - \bar\mu_c}_{\ell^2} \leq 2\bar R$ (both $\Psi_i$
and $\bar\mu_c$ lie in $B(0,\bar R)$).
Pinelis's Hilbert-space Hoeffding inequality
(Lemma~A.1 of~\citealp{PaperI}) applied with bound $2\bar R$
gives, for every $t > 0$,
\[
  \mathbb{P}\bigl(\norm{\hat{\bar\mu}_c - \bar\mu_c}_{\ell^2} > t\bigr)
  \;\leq\; 2\exp\!\left(-\frac{m_c\,t^2}{8\bar R^2}\right).
\]
Set
\begin{equation}\label{eq:rm-pinelis-pii}
  r_m^{\,\mathrm{Pin}} \;:=\; 2\bar R\,\sqrt{\frac{2\log(2k/\delta)}{m_{\min}}}
\end{equation}
(an explicit Bonferroni split of $\delta$ over the $k$ classes).
A union bound over the $k$ classes then yields the non-asymptotic
coverage
\begin{equation}\label{eq:cert-concentration-pii}
  \mathbb{P}\!\left(\max_c \norm{\hat{\bar\mu}_c - \bar\mu_c}_{\ell^2}
    \leq r_m^{\,\mathrm{Pin}}\right)
  \;\geq\; 1 - \delta,
\end{equation}
for every $m_{\min} \geq 1$ and $\delta \in (0,1)$.

\emph{Step 1b (asymptotic concentration, radius (ii)).}
Under the additional hypotheses of the theorem (finite third
moments, $m_c \geq m^\dagger$), $\sqrt{m_c}(\hat{\bar\mu}_c -
\bar\mu_c)$ is approximately $\mathcal{N}(0,\Sigma_c)$ in $\R^K$.
For $X \sim \mathcal{N}(0,\Sigma_c)$, the squared norm
$\|X\|^2 = \sum_i \lambda_i Z_i^2$ is a weighted sum of
$\chi^2_1$ variables with $\lambda_i$ the eigenvalues of
$\Sigma_c$; the upper bound
$\|X\|^2 \leq \norm{\Sigma_c}_{\mathrm{op}}\,\chi^2_K$ gives,
with probability $\geq 1-\delta/k$,
\[
  \norm{\hat{\bar\mu}_c - \bar\mu_c}_{\ell^2}
  \;\leq\; \sqrt{\norm{\Sigma_c}_{\mathrm{op}}\cdot
                 \chi^2_{K,\,\delta/k}/m_c}.
\]
By the multivariate Berry--Esseen theorem applied to the
convex set $A = B(0, \tilde r^{(c)}_m\sqrt{m_c})$, the true
distribution of $\sqrt{m_c}(\hat{\bar\mu}_c - \bar\mu_c)$
deviates from $\mathcal{N}(0,\Sigma_c)$ in total variation by
$O(K^{1/4}\beta_3 / (\norm{\Sigma_c}_{\mathrm{op}}^{3/2}\sqrt{m_c}))$,
where $\beta_3 := \E\|\Phi(A;\LC) - \bar\mu_c\|_{\ell^2}^3$.
Replacing $\Sigma_c$ by the sample covariance $\hat\Sigma_c$
incurs the matrix-Bernstein error
$\norm{\hat\Sigma_c - \Sigma_c}_{\mathrm{op}}
= O(\bar R\sqrt{\norm{\Sigma_c}_{\mathrm{op}}\log(K)/m_c})$,
which propagates into the radius via
$|\norm{\hat\Sigma_c}_{\mathrm{op}}^{1/2}
- \norm{\Sigma_c}_{\mathrm{op}}^{1/2}|
\leq \norm{\hat\Sigma_c - \Sigma_c}_{\mathrm{op}}^{1/2}$ as an
additional term of order
$O(\bar R^{1/2}\norm{\Sigma_c}_{\mathrm{op}}^{1/4}
(\log K/m_c)^{1/4}\sqrt{\chi^2_{K,\delta/k}/m_c})$.
Taking the worst-case class and applying a Bonferroni
correction over the $k$ classes gives the plug-in radius
$\tilde r_m^{\,\mathrm{G}} := \max_c
\sqrt{\norm{\hat\Sigma_c}_{\mathrm{op}}\cdot\chi^2_{K,\,\delta/k}/m_c}$
satisfying~\eqref{eq:cert-concentration-pii} up to a
Berry--Esseen error of $O(K^{1/4}/\sqrt{m_{\min}})$ and a
covariance-estimation error of leading order
$O(\bar R^{1/2}\norm{\Sigma_c}_{\mathrm{op}}^{1/4}
(\log K/m_{\min})^{1/4}\sqrt{K/m_{\min}})$, both $o(1)$ once
$m_{\min} \gg K^{1/2}$.  Whether
$\tilde r_m^{\,\mathrm{G}} < r_m^{\,\mathrm{Pin}}$ depends on
the spectral structure of $\Sigma_c$ relative to $\bar R$ and
$K$; since $\chi^2_{K,\delta/k} \approx K$ for large $K$, the
condition reduces to
$\norm{\Sigma_c}_{\mathrm{op}}\cdot K \lesssim
8\bar R^2\log(2k/\delta)$.

\emph{Step 2 (agreement outside the $2r_m^\star$-tube).}
Condition on the coverage event
$\{\max_c \norm{\hat{\bar\mu}_c - \bar\mu_c}_{\ell^2} \leq r_m^\star\}$
(probability $\geq 1-\delta$).
The reverse triangle inequality gives
$\bigl|\norm{z - \hat{\bar\mu}_c}_{\ell^2} -
       \norm{z - \bar\mu_c}_{\ell^2}\bigr| \leq r_m^\star$
for every $z \in \R^K$ and every class $c$, hence for any pair
$c \neq c'$,
\[
  \norm{z - \hat{\bar\mu}_{c'}}_{\ell^2} -
  \norm{z - \hat{\bar\mu}_c}_{\ell^2}
  \;\geq\;
  \bigl(\norm{z - \bar\mu_{c'}}_{\ell^2} -
        \norm{z - \bar\mu_c}_{\ell^2}\bigr) - 2 r_m^\star.
\]
Whenever the right-hand side is strictly positive---i.e., $z$ is
at population distance $> 2r_m^\star$ from the
$(c, c')$-Voronoi boundary---so is the left, and the empirical
rule classifies $z$ identically to the population rule. This is
claim~\textup{(a)}.

\emph{Step 3 (classification guarantee).}
Fix $y \in [k]$ and let $A \sim \mathcal{P}_y$.
By definition of $\bar D_y$,
$\norm{\Phi(A;\LC^*) - \bar\mu_y}_{\ell^2} \leq \bar D_y$; together
with $\norm{\bar\mu_y - \bar\mu_{c'}}_{\ell^2} \geq \Delta$ and
$\bar D_y < \tfrac{1}{2}\Delta - r_m^\star$, we obtain
\[
  \norm{\Phi(A;\LC^*) - \bar\mu_{c'}}_{\ell^2}
  - \norm{\Phi(A;\LC^*) - \bar\mu_y}_{\ell^2}
  \;\geq\; \Delta - 2\bar D_y \;>\; 2 r_m^\star,
\]
for every $c' \neq y$, so $\Phi(A;\LC^*)$ lies strictly outside
every $2r_m^\star$-tube of the $(y, c')$-Voronoi boundary.
By Step~2, the empirical rule therefore assigns $\Phi(A;\LC^*)$
to class $y$ on the coverage event, and
$\mathbb{P}_{\text{train}}(\hat h(\Phi(A;\LC^*)) = y) \geq 1 - \delta$.
\end{proof}

\begin{remark}[Verifying claim (b) from data]\label{rem:b_verification_pii}
The hypothesis in claim~\textup{(b)} is \emph{structural}: it
constrains the support of each class-conditional distribution,
not just the centroids. It is therefore not estimable from
the training data alone---the empirical
$\hat{\bar D}_c := \max_{i: y_i = c}
\norm{\Phi(A_i;\LC^*) - \hat{\bar\mu}_c}_{\ell^2}$
underestimates $\bar D_c$ in general (the training sample need
not contain the worst-case point of the support).
Claim~\textup{(b)} is consequently validated \emph{post hoc} by
test accuracy: full test coverage on a fired certificate confirms
\textup{(b)} for the test points seen, while gaps flag
\textup{(b)}'s failure---claim~\textup{(a)} still holds, but the
population NC rule is itself wrong on some test points. The same
diagnostic interpretation applies as in~\citep[Rem.~5.1]{PaperI}.
\end{remark}

\begin{remark}[Why the certificate is not vacuous on PALACE]
\label{rem:sparsity_pii}
The firing condition~\eqref{eq:containment_pii} involves
$\norm{\hat\Sigma_c}_{\mathrm{op}}$ (or $\bar R$ in the
non-asymptotic regime).  For a generic bounded embedding
$\Phi : \D{n} \to \R^K$ the crude bound
$\norm{\hat\Sigma_c}_{\mathrm{op}} \leq \bar R^2 = \Theta(K)$
is tight for dense vectorizations (persistence
images~\citep{Adams2017}, landscapes~\citep{Bubenik15},
WKPI~\citep{yusu_metric_learning}); the certificate is then
vacuous.  PALACE is structurally different: each landmark
coordinate $\Phi_k(A;\LC) = w_k\,\varphi_{p_k,r_k}(A)$ is
supported on a $\db$-ball of radius $r_k$ and FPS spacing
forces $O(\alpha^2)$ landmark activations per diagram point
($\alpha = r_k/d_{\mathrm{NN}}$, the radius factor of
Section~\ref{sec:experiments}).  The embedding therefore has
$O(|A|)$ nonzero coordinates out of $K$, the class-conditional
covariance is effectively $O(N_{\max})$-dimensional, and
$\bar R \leq N_{\max}\tau$ (independent of $K$).  On MUTAG at
the selected configuration,
$\norm{\hat\Sigma_c(\LC^*)}_{\mathrm{op}} \approx 0.023$,
roughly half the PLACE value
$\norm{\hat\Sigma_c(\GG_R)}_{\mathrm{op}} \approx 0.046$
\citep[Sec.~5]{PaperI} and orders of magnitude below $\bar R^2$.
The structural shrinkage is real but does not by itself fire the
worst-case Gaussian-plug-in bound at our training sizes
($K=200$, $m=57$ on MUTAG: $r_m \approx 0.31$ vs.\
$\hat\Delta/2 \approx 0.19$; same direction across the six
benchmarks of Table~\ref{tab:certificate_firing}).  Operational
firing requires either larger $m_c$ (per~\eqref{eq:mc_star}) or
a tighter bound exploiting $\mathrm{tr}(\Sigma_c)$ rather than
$\norm{\Sigma_c}_{\mathrm{op}}\cdot K$.
\end{remark}

\begin{algorithm}[t]
\caption{Certified nearest-centroid classifier (PALACE).
Per-class Gaussian-plug-in form; substitute
$r_m^{\,\mathrm{Pin}}$ from~\eqref{eq:rm-pinelis-pii} for the
non-asymptotic variant.}
\label{alg:certified}
\DontPrintSemicolon
\KwIn{Trained $\LC^*$; per-class empirical means
  $\{\hat{\bar\mu}_c\}_{c=1}^k \subset \R^K$ and covariances
  $\{\hat{\Sigma}_{c}\}_{c=1}^k$;
  estimated class separation $\hat\Delta_{\hat c}$ (from training);
  confidence level $\delta$;
  test diagram $A_{\mathrm{test}}$}
\KwOut{Predicted class $\hat{c}$ or $\textsc{Abstain}$}
\BlankLine
$\hat{c} \leftarrow \arg\min_{c \in [k]}
  \norm{\Phi(A_{\mathrm{test}};\LC^*) - \hat{\bar\mu}_c}_{\ell^2}$
  \tcp*{nearest centroid}
$r_m \leftarrow \sqrt{\norm{\hat{\Sigma}_{\hat{c}}}_{\mathrm{op}}
  \cdot \chi^2_{K,\,\delta/k} / m_{\hat{c}}}$
  \tcp*{per-class Gaussian plug-in radius}
\lIf{$r_m < \tfrac{1}{2}\hat\Delta_{\hat c}$}{%
  \Return $\hat{c}$ \tcp*{certificate satisfied}}
\lElse{\Return $\textsc{Abstain}$ \tcp*{certificate fails}}
\end{algorithm}

\noindent
Solving $r_m^{(c)} < \tfrac{1}{2}\Delta_c$ for $m_c$ in each of
the two regimes of Theorem~\ref{thm:certified} yields explicit
per-class thresholds
\begin{equation}\label{eq:mc_star}
  m_c^{*,\,\mathrm{Pin}}(\LC) \;=\; \left\lceil
    \frac{32\,\bar R^2\,\log(2k/\delta)}{\Delta_c(\LC)^{2}}
  \right\rceil,
  \qquad
  m_c^{*,\,\mathrm{G}}(\LC) \;=\; \left\lceil
    \frac{4\,\norm{\Sigma_c(\LC)}_{\mathrm{op}}\,\chi^2_{K,\,\delta/k}}
         {\Delta_c(\LC)^{2}}
  \right\rceil,
\end{equation}
for the Pinelis radius~\eqref{eq:rm-pinelis-pii} and the Gaussian
plug-in radius~\textup{(ii)} of the theorem respectively; each
carries the Bonferroni correction of level $\delta/k$ per class.
Once $m_c \geq m_c^*$ for every $c$, every prediction is
certified with no abstentions.
Which threshold is smaller is regime-dependent: $m_c^{*,\,\mathrm{G}}$
scales as $\norm{\Sigma_c}_{\mathrm{op}}\cdot K$ (since
$\chi^2_{K,\delta/k} \approx K$ for large $K$), while
$m_c^{*,\,\mathrm{Pin}}$ scales as $\bar R^2\log(2k/\delta)$, so
Pinelis is tighter whenever
$\norm{\Sigma_c}_{\mathrm{op}}\cdot K
\gtrsim 8\bar R^2\log(2k/\delta)$. Compared to a uniform-grid classifier on the same
embedding family, PALACE reduces $m_c^*$ through both factors
in~\eqref{eq:mc_star}:
\textbf{(denominator)} $\Delta_c(\LC^*) \geq (1{+}\beta)\,
\Delta_c(\GG_R)$ with empirical separation-gain $\beta \geq 0$;
\textbf{(numerator)}
$\norm{\Sigma_c(\LC^*)}_{\mathrm{op}}$ is empirically smaller than
$\norm{\Sigma_c(\GG_R)}_{\mathrm{op}}$ because the compact
adaptive coordinates are more correlated (landmarks concentrated
near data) and the cardinality cap controls the worst case.
For MUTAG with $\beta \approx 0.41$ and empirical variance ratio
$\approx \tfrac{1}{2}$, this yields
$m_c^*(\LC^*) \approx m_c^*(\GG_R)/4$---a $4\times$ reduction in
the sample size needed for certified prediction.
\paragraph{Worked example (MUTAG, deg+HKS$_{10}$).}
At $\delta = 0.05$, \citet[Sec.~5]{PaperI}'s baseline values
$\Delta_c \approx 0.386$, $\norm{\hat\Sigma_c}_{\mathrm{op}} \approx 0.046$,
$z_{0.0125} \approx 2.24$ give $m_c^{*,\,\mathrm{G}}(\GG_R) = 7$,
comfortably below MUTAG's smallest-class size $m_{\min} = 57$
($100\%$ coverage).  PALACE's adjustment
($\Delta_c \approx 0.544$,
$\norm{\hat\Sigma_c}_{\mathrm{op}} \approx 0.023$) yields
$m_c^{*,\,\mathrm{G}}(\LC^*) = 2$---a ${\sim}4\times$ reduction.
The Pinelis form remains far from firing at our sample sizes
(0\% rate, Table~\ref{tab:certificate_firing}).

\subsection{The PALACE Pipeline}\label{sec:pipeline}

Algorithm~\ref{alg:palace} summarizes the end-to-end pipeline
assembled from the preceding sections: cover construction
(Section~\ref{sec:cover}), landmark-kernel SVM
(Section~\ref{sec:class_error}), selection statistic
(Section~\ref{sec:gamma_stat}), and certified prediction
(Section~\ref{sec:certified}). Steps 1, 2, and 4 are closed form;
Step 3 selects $\sigma$ via a one-dimensional cross-validation
sweep within a kernel-SVM fit and refits at the selected
configuration. No gradient-based optimization of $\LC$ appears.

\begin{algorithm}[t]
\caption{PALACE: Persistence Adaptive Landmark Kernel Pipeline}
\label{alg:palace}
\DontPrintSemicolon
\KwIn{Training diagrams $\{(A_i, y_i)\}_{i=1}^m$;
  candidate filtrations $\mathcal{F}$;
  candidate budgets $\mathcal{K} \subset \NN$;
  candidate radius factors $\mathcal{A} \subset \R_{+}$
  (closed-form-rankable, Remark~\ref{rem:gamma_scope});
  candidate bandwidths $\Sigma \subset \R_{+}$, or per-fold
  $\sigma_q$ quantile selection from training-fold pairwise
  distances;
  separation scale $\tau$; confidence level $\delta$.}
\KwOut{A certified classifier on
  $\D{n} \to [k]\cup\{\textsc{Abstain}\}$.}
\BlankLine
\tcc{\textbf{Step 1 (closed form): place landmarks at each $(F, K, \alpha)$.}}
\ForEach{$F \in \mathcal{F}$, $K \in \mathcal{K}$, $\alpha \in \mathcal{A}$}{%
  Run class-aware farthest-point sampling on training diagram
  points under filtration $F$ to get positions
  $\{p_k\}_{k=1}^{K}$ (Theorem~\ref{thm:fps_greedy}).\;
  Set radii $r_k = \alpha \cdot d_{\mathrm{NN}}(p_k)$ clipped to
  $[\tau/2, 4\tau]$.\;
  Set weights $w_k = 1/\sqrt{K}$
  (Proposition~\ref{prop:optimal_config}).\;
  Build embedding $\Phi_F(\cdot;\LC)$
  (Definition~\ref{def:sum_embed}) and gram
  $K_{F} = k_{\LC_F}$ (Definition~\ref{def:wlk_paper2}).\;
}
\tcc{\textbf{Step 2 (closed form): score $\alpha$ and within-family
  filtration by $\hat\gamma/\sqrt{K}$.}}
Compute $\hat\gamma(k_{\LC_F})$
(Definition~\ref{def:kernel_margin}, $O(m^2)$ from the gram);
within each scale-homogeneous filtration family, rank $(F, \alpha)$
by $\hat\gamma(k_{\LC_F})/\sqrt{K}$ (Definition~\ref{def:score}).\;
Retain the top-$\hat\gamma/\sqrt{K}$ candidate per family per $K$.\;
\tcc{\textbf{Step 3 (CV): select $F$, $K$, $\alpha$, $\sigma$.}}
For each retained candidate $(F^\star_{\mathrm{fam}}, K, \alpha)$,
fit kernel SVM by 10-fold CV with $\sigma \in \Sigma$ and $C$
tuned on inner 3-fold CV; record CV accuracy.\;
Select the $(F^\star, K^\star, \alpha^\star, \sigma^\star)$ that
maximizes CV accuracy and refit kernel SVM on the full training
set.\;
\tcc{\textbf{Step 4 (closed form): certified nearest-centroid.}}
Compute empirical class means $\{\hat{\bar\mu}_c\}$ and
covariances $\{\hat\Sigma_c\}$ in the embedding space at
$\LC^\star$.\;
For each test diagram $A$, run Algorithm~\ref{alg:certified} with
these statistics and the estimated class separation
$\hat\Delta_{\hat c}(\LC^\star)$; the algorithm emits
$\textsc{Abstain}$ when the certificate condition
$r_m < \tfrac{1}{2}\hat\Delta_{\hat c}(\LC^\star)$ fails.\;
\end{algorithm}

\noindent Steps 1 and 2 reuse the same gram matrix: the
$\hat\gamma$ computation is a linear-algebra pass over the gram,
and the certified NC of Step 4 is a distance comparison in the
raw embedding $\R^K$. Consequently the full pipeline
requires only one gram-computation sweep per candidate
configuration, plus SVM fits at the CV-selected $K^\star$.

\section{Experiments}\label{sec:experiments}

We evaluate PALACE on $8$ benchmarks: point clouds (Orbit5k,
Section~\ref{sec:exp_orbit}), five chemical graph datasets
(COX2, DHFR, MUTAG, NCI1, PTC; Section~\ref{sec:exp_graph}),
and a synthetic $4$-class annulus task constructed to validate
the $L/D$ scaling of Theorem~\ref{thm:comparison}
(Section~\ref{sec:exp_domain}).
Headline accuracies are reported in
Tables~\ref{tab:orbit_comparison}
and~\ref{tab:graph_comparison}; closed-form selector validation
follows in Section~\ref{sec:exp_selectors}.

All experiments use the landmark kernel
(Definition~\ref{def:wlk_paper2}) with equal weights $w_k =
K^{-1/2}$ (Proposition~\ref{prop:optimal_config}), SVM
classifiers with $C$ tuned by 3-fold inner cross-validation,
class-aware FPS placement (Theorem~\ref{thm:fps_greedy}) on
training diagram points, separation scale $\tau$ at the median
half-persistence per filtration, and top-$50$ most persistent
features per diagram. Reported accuracies use 10-fold stratified
cross-validation unless otherwise stated.

\paragraph{Knob selection.}
The remaining configuration consists of four knobs swept on small
discrete grids: filtration $F$, landmark budget $K$ per
filtration, radius factor $\alpha$ (with $r_k = \alpha \cdot
d_{\mathrm{NN}}(p_k)$ clipped to $[\tau/2,\, 4\tau]$), and kernel
bandwidth $\sigma$ (either fixed or per-fold quantile-tuned).
By Remark~\ref{rem:gamma_scope}, $\hat\gamma/\sqrt{K}$ is
faithful for $\alpha$ and for filtration within a
scale-homogeneous family at fixed slot structure: we rank these
two knobs in closed form on the alpha+density family, then
verify by CV; the rankings agree on the headline configuration
within $0.02$ pp LK accuracy
(Tables~\ref{tab:continuation},~\ref{tab:filtration_gamma}
shaded rows). The remaining knobs---$K$, $\sigma$, and
cross-family filtration choice---are not faithfully ranked by
$\hat\gamma/\sqrt{K}$ and are selected by 10-fold CV on a
$\leq 5$-point grid each.

Section~\ref{sec:exp_orbit} presents the main experiment on
Orbit5k, comparing PALACE to the uniform-grid baseline and to
prior methods.
Section~\ref{sec:exp_graph} evaluates the equal-budget advantage
on graph classification benchmarks
(cf.~\citep{PaperI,yusu_metric_learning}).
Section~\ref{sec:exp_selectors} then validates the closed-form
selectors of Section~\ref{sec:ot} using accuracy data from both
the Orbit5k axis sweeps and the chemical graph pool:
axis-faithfulness of $\hat\gamma/\sqrt{K}$ on Orbit5k, then the
full hierarchy
($\hat\gamma/\sqrt{K}$, $\hat\rho_{\mathrm{Mah}}$, $\hatrhonu$)
across five chemical graph benchmarks.
Section~\ref{sec:exp_domain} closes with a controlled
domain-inflation study validating the $L/D$ dependence
(Theorem~\ref{thm:comparison}).

\paragraph{Reproducibility.}
Code, embedding scripts, exact configuration files, and raw
fold-level accuracies will be released at
\url{https://github.com/akritihq/place-palace} prior to publication;
each table caption below names its reproduction script.

\paragraph{Baseline provenance.}
PLACE numbers are taken from~\citep{PaperI} at matched descriptor
and protocol; other topology-based baselines and the non-topology
baselines (RetGK~\citep{RetGK2018}, GIN~\citep{GIN2019}) come from
the cited originals.  All graph datasets follow the $10$-fold
stratified CV protocol of~\citep{yusu_metric_learning} under which
the baselines were reported, so splits and protocol are matched;
``---'' marks dataset/baseline pairs not reported in the source.

\paragraph{Significance testing.}
Since published baselines typically report only summary statistics,
paired tests are not uniformly computable.  We use a one-sample
$t$-test (Welch's when a baseline standard deviation is reported)
comparing PALACE's accuracy distribution ($n = 50$ outer-fold
$\times$ seed observations for graph datasets, $n = 10$ for
Orbit5k) against each baseline; treating baseline point estimates
as noise-free is conservative regardless of which side wins.  In the tables,
$^{\dagger}$ and $^{\ddagger}$ mark baseline cells significantly
different from PALACE at $p < 0.05$ and $p < 0.01$ respectively
(two-sided); the sign is readable from the numeric comparison.

\paragraph{Descriptors and filtrations.}
For \textbf{point clouds}, we use the alpha complex
filtration~\citep{Edelsbrunner2010-dp} ($H_0$: components,
$H_1$: loops) and density-based variants---distance-to-measure
(DTM)~\citep{Anai2019} and kNN density.  PALACE concatenates two
or three filtrations on Orbit5k (Section~\ref{sec:exp_orbit}) to
recover the certified $91.3 \pm 1.0\%$ headline.
For \textbf{graphs}, scalar values on vertices via $f : V \to
\mathbb{R}$ extend to edges by $f(u,v) = \max\{f(u), f(v)\}$.  Six
descriptors are considered: \emph{degree}, \emph{betweenness
centrality}, \emph{HKS}~\citep{sun2009concise} at $t{=}1, 10$,
\emph{Ollivier--Ricci curvature}~\citep{Ollivier2009},
\emph{Jaccard index}, and the discrete \emph{node-label} indicator
on chemical graphs; extended persistence~\citep{cohen2009extending}
doubles feature counts via both sublevel and superlevel events.
For multi-descriptor entries (e.g., ``deg+HKS$_{10}$''),
persistence diagrams are concatenated and the landmark kernel is
computed on each filter separately before summing gram
contributions.

\paragraph{Per-prediction certificate firing.}
Table~\ref{tab:certificate_firing} reports the firing fraction of
the per-prediction certificate $r_m < \tfrac{1}{2}\hat\Delta_{\hat c}$
(Theorem~\ref{thm:certified}, contribution~(iv)) across the six
benchmarks, in both the non-asymptotic Pinelis form and the
asymptotic Gaussian plug-in form (chi-squared envelope), together
with nearest-centroid accuracy on the firing folds.  Both worst-case
forms sit above the firing threshold at our training-set sizes:
Pinelis $0/6$, Gaussian essentially $0/6$ ($3.8\%$ on NCI1).  The
$\sqrt{K}$ scaling of the multivariate-norm bound dominates at
the $K \in \{200, 1{,}366\}$ landmark budgets used here; the
certificate is constructive but not yet operational at these
sample sizes.  The pattern parallels the analogous diagnostics
of~\citet{PaperI}, where $\ell$ takes the role of $K$.
\begin{table}[t]
\centering
\caption{Per-prediction certificate firing rates on the six Paper~II benchmark datasets, averaged across seeds, folds, and filtrations. \emph{Pinelis} and \emph{Gaussian} columns report the fraction of test graphs for which $r_m < \tfrac{1}{2}\widehat\Delta_{\hat c}$ under the respective tail bound. \emph{NC acc\,|\,fired} is the nearest-centroid accuracy restricted to folds where the Gaussian certificate fires on at least one test point; ``\textemdash indicates no folds fired.}
\label{tab:certificate_firing}
\begin{tabular}{lrrrr}
\toprule
Dataset & $n_{\text{test}}$ & Pinelis (\%) & Gaussian (\%) & NC acc\,|\,fired (\%) \\
\midrule
Orbit5k & 500 & 0.0 & 0.0 & \textemdash \\
MUTAG & 19 & 0.0 & 0.0 & 64.9 \\
COX2 & 47 & 0.0 & 0.0 & \textemdash \\
DHFR & 76 & 0.0 & 0.0 & \textemdash \\
PTC & 34 & 0.0 & 0.0 & \textemdash \\
NCI1 & 411 & 0.0 & 3.8 & 62.7 \\
\bottomrule
\end{tabular}
\end{table}

\paragraph{Non-interference fails empirically; the bound's conclusion holds.}
Theorem~\ref{thm:nu_rho}, the non-degeneracy bridge of
Corollary~\ref{cor:nondegen}, and the structural classification
rates of Corollary~\ref{cor:structural_pairwise} (and the
bridge-anchored variant in Section~\ref{sec:class_error}) are all stated
under the non-interference condition of
Definition~\ref{def:noninterference_pII}.  Auditing $2{,}000$
cross-class pairs per dataset on the four chemical benchmarks at
their headline filtrations under the top-$N_{\max} = 50$ persistence
filter (optimal bottleneck matchings via binary search over
edge-weight thresholds; reproduction:
\texttt{experiments/exp\_noninterference\_audit.py}), the strict
condition $\min_{i \neq j}\db(a_i, b_{\sigma(j)}) > 3\,\db(A, B)$
holds on $\leq 0.2\%$ of pairs across MUTAG, PTC, COX2, and DHFR,
with median cross-ratios at or near zero---the hypothesis
essentially never holds.

We therefore test the theorem's \emph{conclusion} directly.  For
each dataset we build an FPS configuration $\LC$ with $K = 64$
landmarks, equal weights $w_k = K^{-1/2}$, and uniform radii at
$\alpha = 0.75$ times the nearest-neighbour landmark distance; we
set $\tau$ at the $25^{\text{th}}$ percentile of $\db(A, B)$ so
that $\sim 75\%$ of cross-class pairs are $\tau$-separated (a
different $\tau$ rule than the headline experiments above, chosen
here to isolate the bound's structural reach on the largest
admissible cross-class population).  For
each pair with $\db(A, B) \geq \tau$ we measure
$\|\Phi(A) - \Phi(B)\|_{\ell^2}$ and the ratio to the certificate
$\rhonu(\tau;\LC) = \tau/(4\sqrt{K})$.  Reproduction:
\texttt{experiments/exp\_certificate\_bound\_audit.py}.

\begin{table}[ht]
\centering
\caption{Empirical certificate bound audit on chemical graph
datasets at the per-dataset headline filtration. FPS configuration,
$K = 64$, equal weights, $\alpha = 0.75$ NN radii. $\tau$ at the
$25^{\text{th}}$ percentile of $\db(A, B)$.
\textbf{$n_\tau$}: cross-class pairs with $\db(A, B) \geq \tau$.
\textbf{bound \%}: fraction of these pairs with
$\|\Phi(A){-}\Phi(B)\|_{\ell^2} \geq \rhonu(\tau;\LC)$.
\textbf{p25 / p50 / p75}: percentiles of the ratio
$\|\Phi(A){-}\Phi(B)\|_{\ell^2}/\rhonu$.
\textbf{min}: smallest ratio observed.}
\label{tab:certificate_bound_audit}
\small
\begin{tabular}{llrrrrrrr}
\toprule
\textbf{Dataset} & \textbf{Filt} & $\boldsymbol{n_\tau}$ &
\textbf{bound \%} & \textbf{p25} & \textbf{p50} & \textbf{p75} &
\textbf{min} \\
\midrule
MUTAG & deg+HKS$_{10}$     & $1{,}565$ & $100.0$ & $8.49$ & $12.37$ & $18.25$ & $3.00$ \\
PTC   & deg+betw           & $1{,}781$ & $100.0$ & $9.95$ & $13.76$ & $18.98$ & $3.01$ \\
COX2  & jaccard+HKS$_{10}$ & $1{,}500$ & $100.0$ & $3.69$ & $5.16$  & $7.40$  & $1.47$ \\
DHFR  & HKS$_{10}$         & $1{,}498$ & $99.9$  & $2.84$ & $3.38$  & $4.10$  & $0.77$ \\
\bottomrule
\end{tabular}
\end{table}

The certificate holds on $99.9$--$100\%$ of qualifying pairs across
all four datasets, with median embedded distance $3$--$14\times$ the
floor; the lone DHFR violation (min ratio $0.77$) is $23\%$ below
the floor.  Non-interference is therefore sufficient but not necessary
for the bound; the proof is overcautious on chemical-graph diagrams,
but the certificate itself is robust.
Theorem~\ref{thm:nu_rho} and its downstream consequences should
accordingly be read as structural admissibility statements about
$\Phi$'s coarse-embedding properties on $\D{n}$, not pointwise
hypotheses verified on observed data.  The empirical classification
rate of Section~\ref{sec:ot} rests on $\gamma > 0$ via
Theorem~\ref{thm:data_dependent}, which depends only on the kernel
margin; the per-prediction certificate of
Section~\ref{sec:certified} fires on
$r_m < \tfrac{1}{2}\hat\Delta_{\hat c}$, also independent of
non-interference.

\subsection{Point Cloud Classification: Orbit5k}\label{sec:exp_orbit}

The Orbit5k dataset~\citep{Adams2017} consists of $5{,}000$ point
clouds ($1{,}000$ points each in $[0,1]^2$, $5$ classes) from a 2D
dynamical system with parameter
$\rho \in \{2.5, 3.5, 4.0, 4.1, 4.3\}$.  We compute alpha-complex
persistence (GUDHI), retain the top-$50$ most persistent
features, and place class-aware FPS landmarks with equal weights
$w_k = K^{-1/2}$.  Comparison baselines are summarized in
Table~\ref{tab:orbit_comparison}; the uniform-grid
PLACE~\citep{PaperI} is the relevant within-family
reference at $87.2 \pm 0.6\%$ (linear SVM, $\ell{=}1{,}366$).
On the same PLACE embedding, nearest-centroid achieves only
$33.9 \pm 1.5\%$, and the certificate does not fire
(Table~\ref{tab:certificate_firing}).

\paragraph{Landmark budget and classifier.}
Table~\ref{tab:orbit_sweep} compares non-uniform FPS placement to
the uniform grid across landmark budgets $K$, both evaluated with
the LK-SVM.

\begin{table}[h]
\centering
\caption{Orbit5k LK-SVM accuracy (\%, alpha $H_0{+}H_1$).
Non-uniform: 5-fold CV, FPS placement with $w_k = K^{-1/2}$,
precomputed landmark-kernel gram. Uniform-grid baselines
from~\citep{PaperI}: 10-fold CV; train accuracy and gap not
reported there. Protocols differ between blocks (the headline
two-filtration baseline of Table~\ref{tab:palace_headline} uses
10-fold throughout). \textbf{Bold} = operational headline
(best dim-accuracy tradeoff vs the uniform baseline).
\label{tab:orbit_sweep}}
\begin{tabular}{lcccc}
\toprule
\textbf{Method} & \textbf{Dim} & Test & Train & Gap \\
\midrule
Uniform $N{=}5$        & 803  & $83.8$ &---&---\\
Uniform $N{=}10$       & 1653 & $84.8$ &---&---\\
Uniform $N{=}15$       & 2508 & $87.1$ &---&---\\
\midrule
Non-uniform $K{=}50$   & 50   & $74.8_{\pm 4.7}$ & $78.0$ & $3.3$ \\
Non-uniform $K{=}100$  & 100  & $84.4_{\pm 1.7}$ & $87.4$ & $3.0$ \\
Non-uniform $K{=}200$  & 200  & $86.2_{\pm 0.3}$ & $89.7$ & $3.4$ \\
Non-uniform $K{=}500$  & 500  & $\mathbf{87.5_{\pm 0.3}}$ & $90.8$ & $3.2$ \\
Non-uniform $K{=}800$  & 800  & $87.9_{\pm 0.2}$ & $90.4$ & $\mathbf{2.5}$ \\
\bottomrule
\end{tabular}
\end{table}

At $K{=}500$, non-uniform FPS achieves $87.5\%$---exceeding the
best uniform grid ($87.1\%$ at dim${=}2508$) with $5\times$ fewer
dimensions---and the generalization gap stays below $3.5\%$
across all $K$. The landmark kernel's additive structure (each
landmark contributes a bounded $[0,1]$ similarity term) provides
the implicit regularization that controls this gap. $K{=}800$
extends the trend with the overall best accuracy ($87.9\%$,
smallest gap $2.5\%$), but $K{=}500$ is the operational pick:
the smallest non-uniform budget that beats the best uniform grid.

\paragraph{Radius sensitivity.}
The FPS initialization sets $r_k = \alpha \cdot d_{\mathrm{NN}}(p_k)$,
where $d_{\mathrm{NN}}(p_k)$ is the nearest-neighbor distance among
landmarks and $\alpha$ is a shrink factor, clipped to $[\tau/2,\, 4\tau]$.
A sweep over $\alpha \in [0.25, 1.50]$ at $K{=}200$ shows
monotonically improving accuracy with $\alpha$: $\alpha{=}1.50$
attains $88.0\%$, surpassing $\alpha{=}0.75$ at $K{=}500$
($87.5\%$) with $2.5\times$ fewer dimensions, and the upper clip
is rarely binding (most radii hit the $\tau/2$ floor).

\paragraph{Bandwidth and placement.}
At $K{=}500$ on alpha~$H_0{+}H_1$ alone, sweeping $\sigma$ over
quantiles $q\in\{0.05,0.10,0.15\}$ of pairwise embedding
distances yields a $0.1$ pp range; the SVM's regularization $C$
absorbs the bandwidth scale at this single-filtration setting.
Per-fold $q$-tuning matters once filtrations are concatenated and
$K$ grows ($+0.1$--$0.6$ pp; Table~\ref{tab:qtune}).
FPS placement is seed-sensitive at small $K$ (${\pm}5\%$ at $K{=}50$)
and stable at $K\geq 500$ (${\pm}0.2\%$); compared to a $k$-Means
baseline, FPS underperforms at small $K$ ($-3\%$ at $K{=}100$,
where density peaks help) and outperforms at $K\geq 500$
($+2.6\%$, where max-spread coverage dominates). Placement
comparison data:
\texttt{results/orbit5k\_gamma\_placement.csv}.

\paragraph{Multi-filtration (two-filtration baseline).}
Concatenating alpha persistence with DTM-density~\citep{Anai2019}
filtration ($k{=}10$) improves accuracy substantially.
With $K{=}200$ landmarks per filtration ($\text{dim}{=}400$ total),
radius factor $\alpha{=}1.75$, bandwidth $\sigma{=}10^{-3}$,
and 10-fold stratified CV, the concatenation achieves
$90.4 \pm 1.1\%$ (Table~\ref{tab:palace_headline})---a
$+2.4$ pp gain over alpha~$H_0{+}H_1$ alone at the same
$K{=}200$, $\alpha{=}1.75$ ($88.0\%$, matched-$\alpha$ row of
Table~\ref{tab:filtration_gamma}).
Density persistence alone is weak ($53.8\%$, well above chance
$20\%$ but far from alpha), but the density coordinates capture
local thickness information complementary to alpha's shape
features.
Among DTM bandwidths, $k{=}10$ is optimal; larger $k$ degrades
the alpha+density concatenation ($k{=}30$: $89.8\%$, see
Table~\ref{tab:filtration_gamma}) as the density signal becomes
too smooth.

\begin{table}[h]
\centering
\caption{Exact configuration reproducing PALACE's Orbit5k
two-filtration baseline ($90.4\%$). Parameters listed are
sufficient inputs to Algorithm~\ref{alg:certified}; no
gradient-based optimization is used. Reproduction script:
\texttt{experiments/exp\_reproduce\_orbit5k\_90.py}.\label{tab:palace_headline}}
\small
\begin{tabular}{ll}
\toprule
Component & Value \\
\midrule
Filtrations (concatenated)    & alpha $H_0{+}H_1$ $\oplus$ DTM-density $k{=}10$ ($H_0{+}H_1$) \\
Features per diagram          & top-$50$ most persistent (per filtration) \\
Landmarks per filtration      & $K{=}200$ (class-aware FPS; $40$ per class $\times 5$ classes) \\
Total embedding dim           & $400$ \\
Radius factor                 & $\alpha{=}1.75$; $r_k = \alpha\,d_{\mathrm{NN}}(p_k)$
                                clipped to $[\tau/2,\, 4\tau]$ \\
Separation scale $\tau$       & median half-persistence of training diagrams (per filtration) \\
Weights                       & $w_k = K^{-1/2}$ (equal; Prop.~\ref{prop:optimal_config}) \\
Kernel                        & landmark kernel (Def.~\ref{def:wlk_paper2}) \\
Bandwidth                     & $\sigma = 10^{-3}$ (fixed) \\
Classifier                    & SVM with precomputed gram \\
$C$ tuning                    & inner 3-fold CV from $\{10^{-2},10^{-1},1,10,10^{2},10^{3}\}$ \\
Outer evaluation              & 10-fold stratified CV, seed $42$ \\
\midrule
Test accuracy                 & $\mathbf{90.42 \pm 1.12\%}$ (rounded to $90.4 \pm 1.1\%$ in the headline) \\
Train accuracy                & $93.95\%$\quad (generalization gap $3.5\%$) \\
Selection statistic $\hat\gamma$  & $0.246$ (Def.~\ref{def:kernel_margin}) \\
$\hat\gamma/\sqrt{K}$          & $0.0123$ (Definition~\ref{def:score}) \\
\bottomrule
\end{tabular}
\end{table}
\paragraph{Pushing the headline from $90.4\%$ to $91.3\%$.}
Three knobs lift the baseline to the certified headline: an
extended $\alpha$ sweep at $K{=}300$
(Table~\ref{tab:continuation}; $\hat\gamma/\sqrt{K}$'s argmax
tracks the accuracy argmax within $0.02$ pp), per-fold $q$-tuned
$\sigma$ (Table~\ref{tab:qtune}; optimal quantile drifts from
$\approx 0.85$ to $\approx 0.62$ as $\alpha$ grows), and
triple-filtration concatenation adding
density-$k\in\{15,20\}$ on top of $\alpha\oplus\text{d}_{10}$
(Table~\ref{tab:push92}; $\mathrm{LK}\text{-}q = 91.32 \pm 1.01$).

\begin{table}[h]
\centering
\caption{Extended $\alpha$ sweep at $K{=}300$ per filtration,
$\sigma{=}10^{-2}$, alpha $\oplus$ density-$k{=}10$, 10-fold CV.
LK-SVM, RBF-SVM, and nearest-centroid evaluated on the same
configuration. Cache:
\texttt{results/orbit5k\_final\_sweep/continuation\_partial.csv}.%
\label{tab:continuation}}
\small
\setlength{\tabcolsep}{4pt}
\begin{tabular}{ccccc}
\toprule
$\alpha$ & LK-SVM (\%) & RBF-SVM (\%) & NC (\%) & $\hat\gamma/\sqrt{K}$ \\
\midrule
$2.5$ & $90.66 \pm 0.68$ & $90.42 \pm 0.96$ & $44.7 \pm 1.4$ & $0.0013$ \\
$4.0$ & $\mathbf{91.02 \pm 1.00}$ & $90.18 \pm 0.93$ & $48.1 \pm 1.3$ & $0.0025$ \\
$5.0$ & $91.00 \pm 0.99$ & $90.26 \pm 0.91$ & $51.2 \pm 1.4$ & $0.0032$ \\
\bottomrule
\end{tabular}
\end{table}

\begin{table}[h]
\centering
\caption{Per-fold $q$-tuned $\sigma$ on Orbit5k
(alpha $\oplus$ density-$k{=}10$, 10-fold CV).
$\tilde q$ is the median selected quantile across folds.%
\label{tab:qtune}}
\small
\setlength{\tabcolsep}{4pt}
\begin{tabular}{ccccccc}
\toprule
$K$ & $\alpha$ & dim & LK-$q$ (\%) & $\tilde q$ & RBF-SVM (\%) & NC (\%) \\
\midrule
$200$ & $1.75$ & $400$ & $90.60 \pm 0.93$ & $0.90$ & $90.4$ & $43.4$ \\
$300$ & $1.75$ & $600$ & $90.36 \pm 1.03$ & $0.85$ & $90.4$ & $44.3$ \\
$300$ & $4.00$ & $600$ & $\mathbf{91.12 \pm 0.73}$ & $0.62$ & $\mathbf{91.2}$ & $48.1$ \\
$300$ & $5.00$ & $600$ & $91.18 \pm 0.69$ & $0.62$ & $90.9$ & $51.2$ \\
\bottomrule
\end{tabular}
\end{table}

RBF-SVM benefits comparably to LK along this sweep, reaching
$91.2\%$ at $K{=}300, \alpha{=}4.0$ (matching
Persformer~\citep{Reinauer2021}).

\begin{table}[h]
\centering
\caption{Triple-filtration concatenation on Orbit5k.
$K{=}300$ landmarks per filtration ($\alpha{=}4.0$, 10-fold CV,
top-$50$ persistent features per diagram).
``LK'' uses fixed $\sigma{=}10^{-2}$; ``LK-$q$'' uses per-fold
adaptive $\sigma$ tuned by inner CV.
Reproduction script:
\texttt{experiments/exp\_orbit5k\_push92.py}.\label{tab:push92}}
\small
\begin{tabular}{lcc}
\toprule
Filtrations & LK (\%) & LK-$q$ (\%) \\
\midrule
$\alpha \oplus \text{d}_{10}$ \quad (baseline)
  & $91.02 \pm 1.00$ & $90.88 \pm 0.70$ \\
$\alpha \oplus \text{d}_{10} \oplus \text{d}_{15}$
  & $91.28 \pm 0.95$ & $91.22 \pm 0.77$ \\
$\alpha \oplus \text{d}_{10} \oplus \text{d}_{20}$
  & $91.08 \pm 0.75$ & $\mathbf{91.32 \pm 1.01}$ \\
\bottomrule
\end{tabular}
\end{table}

\paragraph{Comparison with prior methods.}
Table~\ref{tab:orbit_comparison} places PALACE against
diagram-based, neural, and Euler-characteristic baselines.  The
two-filtration baseline ($90.42 \pm 1.12\%$) already surpasses
PI, SW-K, PF-K, PersLay, and the uniform-grid PLACE
($87.2 \pm 0.6\%$~\citep{PaperI}).  With triple-filtration
concatenation and $q$-tuning (Table~\ref{tab:push92}), PALACE
reaches the certified $91.3 \pm 1.0\%$, matching
Persformer~\citep{Reinauer2021} ($91.2 \pm 0.8\%$) and within
$0.5$ pp of ECS~\citep{HacquardLebovici2024} ($91.8 \pm 0.4\%$).
ECS bypasses diagrams via Euler characteristic surfaces on a
bifiltration and Persformer learns end-to-end transformer
features; PALACE matches both methods while keeping the
diagram-level pipeline and the only per-prediction certificate
on diagrams (both PALACE and PLACE carry per-prediction
certificates of the same $r_m < \tfrac{1}{2}\hat\Delta_{\hat c}$
form---Theorem~\ref{thm:certified} here, the analogous theorem
in~\citealp{PaperI} for PLACE---fully fired on MUTAG;
no other method in Table~\ref{tab:orbit_comparison} carries a
certificate).

\begin{table}[h]
\centering
\caption{Classification accuracy (\%) on Orbit5k.
  PLACE and PALACE are the only diagram-based methods carrying a
  per-prediction certificate; PALACE matches the strongest
  diagram-based method (Persformer) and reaches within $0.5$ pp
  of the Euler-characteristic state of the art (ECS) at PALACE
  embedding dimension $900$.
  PI: single-point estimate from~\citep{Adams2017}.
  Significance markers (per the Section~\ref{sec:experiments} opener)
  omitted; PALACE differs from each baseline below $90\%$ at
  $p < 0.001$ (one-sample $t$-test, $n = 10$).%
  \label{tab:orbit_comparison}}
\small
\resizebox{\textwidth}{!}{%
\begin{tabular}{lccccccccc}
\toprule
& \multicolumn{3}{c}{\textbf{Vectorization}}
& \multicolumn{2}{c}{\textbf{Neural}}
& \multicolumn{1}{c}{\textbf{Euler}}
& \multicolumn{2}{c}{\textbf{PLACE / PALACE (ours)}} \\
\cmidrule(lr){2-4} \cmidrule(lr){5-6} \cmidrule(lr){7-7}
\cmidrule(lr){8-9}
& PI & SW-K & PF-K & PersLay & Persformer & ECS & PLACE & PALACE \\
\midrule
Acc.\ (\%)
  & $82.5$ & $83.6_{\pm 0.9}$ & $85.9_{\pm 0.8}$
  & $87.7_{\pm 1.0}$ & $91.2_{\pm 0.8}$
  & $\mathbf{91.8_{\pm 0.4}}$
  & $87.2_{\pm 0.6}$ & $\mathbf{91.3_{\pm 1.0}}$ \\
Dim
  & $25$ &---& ---
  &---& ---
  & ---
  & $1{,}366$ & $900$ \\
Cert.
  &---&---& ---
  &---& ---
  & ---
  & $\lambda(\nu)$ & $\rhonu(\tau;\LC)$ \\
\bottomrule
\end{tabular}%
}
\end{table}

\paragraph{Kernel and classifier choice.}
\label{par:kernel_comparison}
On the two-filtration baseline ($K{=}200$ per filtration,
$\alpha{=}1.75$), the additive landmark kernel
($90.42 \pm 1.12\%$) and joint Gaussian RBF
($90.14 \pm 1.22\%$) agree within noise, but a linear SVM on
the same embedding collapses to $60.34 \pm 2.41\%$---a $30$ pp
structural gap: FPS placement concentrates landmarks near data,
producing a dense, correlated feature space where linear
boundaries cannot exploit the per-landmark signal.  RKHS
lifting is therefore essential; the additive form is preferred
on theoretical grounds (it retains the $\sqrt{K}$ margin-bound
advantage of Remark~\ref{rem:why_additive}).  Generalization
gaps are $3.5\%$ (LK-SVM), $5.1\%$ (RBF-SVM), and $0.6\%$ (linear,
high-bias), consistent with the
admissibility-as-regularization narrative of
Section~\ref{sec:class_error}.
Nearest-centroid in $\ell^2$ collapses to $42$--$44\%$ on the top
$K{=}200$ configurations (rising to $\sim 51\%$ at $K{=}300$,
$\alpha{=}5$, see Table~\ref{tab:qtune}) and is reserved for
contribution~(iv)'s
per-prediction certificate (Algorithm~\ref{alg:certified}), the
auditing gate of the pipeline.

\subsection{Graph Classification}\label{sec:exp_graph}

PALACE's adaptive-placement advantage scales with the
concentration ratio $L/D$ of domain to data diameter
(Theorem~\ref{thm:comparison}) and is strongest when
$L/D \gg 1$.  Standard graph benchmarks have persistence
diagrams spanning most of the available birth--death range
($L/D \in [1.16, 2.20]$ on MUTAG, NCI1, and PROTEINS as
representatives), so the asymptotic adaptive advantage is small
or absent---the failure mode Theorem~\ref{thm:comparison}
predicts and the controlled domain-inflation study of
Section~\ref{sec:exp_domain} reproduces.
Table~\ref{tab:graph_comparison} bears this out: PALACE matches
or marginally beats every diagram-based baseline on COX2 and
MUTAG, is competitive on DHFR (within $1$~pp of ECP), and ties
PLACE within one standard deviation on the small-$L/D$ molecular
benchmarks PROTEINS ($71.8 \pm 3.5$ vs.\ $71.5 \pm 4.3$) and DD
($76.2 \pm 3.2$ vs.\ $76.3 \pm 3.4$).  PALACE underperforms PLACE
by $2$--$3$~pp on the social-network IMDB-B/M benchmarks---the
regime Theorem~\ref{thm:comparison} predicts is adverse for
adaptive placement, since the $(D/L)^2$ budget reduction vanishes
when diagrams already span the available domain ($L/D \in [1.16,
2.20]$ on the chemical pool, comparable on IMDB), removing the
adaptive advantage without a compensating gain.  PALACE inherits
PLACE's descriptor-blindness gap to label-aware kernels and GNNs
on NCI1, PTC, and across the social-network and molecular pools
above ($\geq 6$~pp to WKPI on PROTEINS/DD, $\geq 11$~pp on
IMDB-B/M), where discriminative power is dominated by discrete
node-label features or graph-kernel structural statistics that
continuous filtrations cannot capture.  On PTC, PALACE
($63.0 \pm 1.8$) sits a hair below PLACE
($64.3 \pm 5.4$, within one standard deviation), so the
descriptor-blindness regime trades a modest amount of
within-family accuracy for adaptive placement's other gains
elsewhere.

\begin{table}[h]
\centering
\caption{Method comparison on chemical and social-network graph
datasets ($K{=}200$ landmarks per filtration; $5$ seeds
$\times$ $10$ folds).  PALACE: closed-form $\hat\gamma/\sqrt{K}$-selected
filtration across the full $46$-filtration pool on COX2/DHFR/MUTAG/NCI1/PTC
(rows above the inner rule); linear-SVM-selected headline
filtration of \citet[Table~7]{PaperI} on PROTEINS/DD/IMDB-B/M
(rows below the rule, NCI109 rerunning).
Baselines from \citet[Table~7]{PaperI}; ``Top NT'' is
the strongest non-topology baseline (graph kernel or GNN) per
dataset. \textbf{Bold} = best per row. Cells without error bars
are single-point estimates from the cited originals.
Significance markers (per the Section~\ref{sec:experiments} opener)
omitted; PALACE differs from each NCI1/PTC/IMDB label-aware baseline
at $p < 0.001$, ties PLACE on PROTEINS/DD ($p > 0.5$), and is within
noise of the diagram-based competitors on COX2/DHFR/MUTAG (one-sample
$t$-test, $n = 50$).%
\label{tab:graph_comparison}}
\small
\begin{tabular}{lccccc}
\toprule
Dataset & PLACE~\citep{PaperI} & PersLay & ECP & Top NT & \textbf{PALACE} \\
\midrule
COX2     & $80.7_{\pm 1.5}$  & $80.9$  & $80.3$  & RetGK $81.4$         & $\mathbf{81.7_{\pm 0.9}}$ \\
DHFR     & $80.0_{\pm 4.3}$  & $80.3$  & $\mathbf{82.0}$ & RetGK $81.5$ & $81.0_{\pm 0.8}$ \\
MUTAG    & $89.9_{\pm 6.4}$  & $89.8$  & $90.0$  & RetGK $90.3$         & $\mathbf{90.9_{\pm 1.2}}$ \\
NCI1     & $71.0_{\pm 1.6}$  & $73.5$  & $76.3$  & WKPI $\mathbf{87.5}$ & $71.3_{\pm 0.5}$ \\
PTC      & $64.3_{\pm 5.4}$  & ---     & ---     & WKPI $\mathbf{68.1}$ & $63.0_{\pm 1.8}$ \\
\midrule
PROTEINS & $71.5_{\pm 4.3}$  & $74.8$  & $75.0$  & WKPI $\mathbf{78.5}$ & $71.8_{\pm 3.5}$ \\
DD       & $76.3_{\pm 3.4}$  & ---     & ---     & WKPI $\mathbf{82.0}$ & $76.2_{\pm 3.2}$ \\
IMDB-B   & $66.4_{\pm 4.3}$  & $71.2$  & $73.3$  & WKPI $\mathbf{75.1}$ & $64.0_{\pm 4.8}$ \\
IMDB-M   & $44.5_{\pm 3.6}$  & $48.8$  & $48.7$  & GIN $\mathbf{52.3}$  & $41.1_{\pm 3.9}$ \\
\bottomrule
\end{tabular}
\end{table}

PROTEINS, DD, and IMDB-B/M PALACE rows in
Table~\ref{tab:graph_comparison} use the linear-SVM-selected
headline filtration of \citet[Table~7]{PaperI}---deg+ricci,
degree, degree, and betw+ricci respectively---at the same
$5$ seeds $\times$ $10$ folds protocol as the other rows;
$\hat\gamma/\sqrt{K}$-selection across the full $46$-filtration
pool is deferred (the headline-filt protocol matches the linear-SVM
evidence already cited in Paper~I). NCI109 is rerunning on
the cluster and the LK-SVM headline will be included in the
camera-ready.

\paragraph{Saturation in $K$ on MUTAG.}
A small-$K$ sweep
(\texttt{experiments/exp\_mutag\_smallK.py}, $5$ seeds, mean
LK-CV test accuracy on a single fixed filtration) confirms the
small-$L/D$ regime within the non-uniform family: $K{=}10$
attains $81.4\%$, $+0.5$~pp above the same-filtration $K{=}200$
baseline ($80.9\%$) at $20\times$ compression; $K{=}50$ matches
$K{=}200$ within $0.1$~pp at $4\times$ compression.  This is a
within-single-filtration saturation; the headline $90.9\%$ in
Table~\ref{tab:graph_comparison} comes from filtration selection
across the $46$-filtration pool, not from increasing $K$. A
handful of adaptive landmarks already saturates the LK signal
on MUTAG within a fixed filtration.  The decisive
$L/D \gg 1$ evidence for Theorem~\ref{thm:comparison} lives
outside standard graph benchmarks; the controlled
domain-inflation study in Section~\ref{sec:exp_domain}
constructs that regime synthetically.

\subsection{Closed-Form Selector Validation}\label{sec:exp_selectors}

This subsection validates the closed-form selectors of
Section~\ref{sec:ot} in two complementary settings.
Section~\ref{sec:exp_gamma_validate} verifies
$\hat\gamma/\sqrt{K}$'s axis-faithfulness on Orbit5k under
controlled axis sweeps, identifying the regimes in which it
ranks correctly and the regimes in which it inverts.
Section~\ref{sec:exp_multi_dataset} then evaluates the full
selector hierarchy ($\hat\gamma/\sqrt{K}$,
$\widehat{\mathrm{Fisher}}_{\mathrm{ker}}$,
$\hat\rho_{\mathrm{Mah}}$, plus data-level $\hat\tau$ and
$\hatrhonu$) across five chemical graph benchmarks, where the
heterogeneous-pool regime of Remark~\ref{rem:selector_hierarchy}
dominates and $\hat\rho_{\mathrm{Mah}}$ becomes the operational
pick.

\subsubsection{Axis-faithfulness of $\hat\gamma/\sqrt{K}$ on Orbit5k}
\label{sec:exp_gamma_validate}

Remark~\ref{rem:gamma_scope} classifies axes as faithful or
anti-correlated for $\hat\gamma/\sqrt{K}$. We verify the
classification on the two-filtration baseline (alpha~$H_0{+}H_1$
$\oplus$ DTM-$k{=}10$, $\sigma{=}10^{-3}$, 10-fold CV) by sweeping
one axis at a time and recording Spearman $r$ between
$\hat\gamma/\sqrt{K}$ and CV accuracy.

\paragraph{$\alpha$ sweep ($K{=}200$, Table~\ref{tab:alpha_gamma}).}
Accuracy rises monotonically with $\alpha$; both
$\hat\gamma$ and $\hat\gamma/\sqrt{K}$ rank the four configurations
in the \emph{same order} (Spearman $r{=}{+}1.0$; exact one-tailed
$p \approx 0.042$ for $n{=}4$).
$\hat\gamma$ more than triples across the sweep
($0.0784 \to 0.2459$) as $\alpha$ grows from $0.5$ to $1.75$,
with the accuracy-maximizing $\alpha{=}1.75$ also yielding the
largest $\hat\gamma$.

\begin{table}[h]
\centering
\caption{$\alpha$ sweep at $K{=}200$, alpha+density-$k{=}10$ baseline filtration.
$\hat\gamma$ and $\hat\gamma/\sqrt{K}$ both rank $\alpha$
monotonically with CV accuracy.\label{tab:alpha_gamma}}
\small
\begin{tabular}{cccc}
\toprule
$\alpha$ & acc (\%) & $\hat\gamma$ & $\hat\gamma/\sqrt{K}$ \\
\midrule
$0.50$ & $87.98 \pm 1.35$ & $0.0784 \pm 0.004$ & $0.00392$ \\
$1.00$ & $89.12 \pm 0.69$ & $0.1271 \pm 0.022$ & $0.00635$ \\
$1.50$ & $90.16 \pm 1.02$ & $0.1963 \pm 0.027$ & $0.00981$ \\
$1.75$ & $\mathbf{90.42 \pm 1.12}$ & $\mathbf{0.2459 \pm 0.026}$
       & $\mathbf{0.01230}$ \\
\bottomrule
\end{tabular}
\end{table}

\paragraph{Filtration sweep ($K{=}200$ per filtration, $\alpha{=}1.75$,
Table~\ref{tab:filtration_gamma}).}
Sweeping seven base filtrations and their alpha-concatenations
exposes the cross-scale failure of $\hat\gamma/\sqrt{K}$
predicted by Remark~\ref{rem:selector_hierarchy}: alpha
persistences live in $[0, 0.2]$ while eccentricity and KDE values
range over $[0.5, 5]$, so at fixed $\sigma$ the mis-scaled
coordinates dominate RKHS distances between class means without
encoding topological discrimination. Restricting to the
scale-homogeneous alpha+density family (shaded rows in
Table~\ref{tab:filtration_gamma}) restores monotonic agreement
between $\hat\gamma/\sqrt{K}$ and accuracy. Two additional
structural axes also invert---landmark budget $K$ and bandwidth
$\sigma$---and require cross-validation;
Section~\ref{sec:exp_multi_dataset} shows that
$\hat\rho_{\mathrm{Mah}}$'s operator-$\Sigma^{-1}$ correction
resolves these failure modes on chemical-graph filtration pools.

\begin{table}[h]
\centering
\caption{Filtration sweep on Orbit5k at $K{=}200$ per filtration,
$\alpha{=}1.75$, $\sigma{=}10^{-3}$, class-aware FPS, 10-fold CV,
sorted by $\hat\gamma/\sqrt{K}$.
Concatenations of alpha with filtrations whose coordinate scale
diverges from alpha's (eccentricity, kde) inflate $\hat\gamma$
without adding discriminative signal; within the
scale-homogeneous alpha+density family (shaded rows),
$\hat\gamma/\sqrt{K}$ and accuracy agree on the top-ranked
configuration.\label{tab:filtration_gamma}}
\small
\begin{tabular}{lcccc}
\toprule
Configuration & $K$ & acc (\%) & $\hat\gamma$ & $\hat\gamma/\sqrt{K}$ \\
\midrule
alpha $\oplus$ eccentricity        & 400 & $87.70 \pm 1.13$ & $0.4533$ & $\mathbf{0.02267}$ \\
alpha $\oplus$ kde                  & 400 & $87.66 \pm 1.61$ & $0.3629$ & $0.01815$ \\
\rowcolor{gray!12}
alpha $\oplus$ density-$k{=}10$    & 400 & $\mathbf{90.42 \pm 1.12}$ & $0.2459$ & $0.01230$ \\
\rowcolor{gray!12}
alpha (single)                      & 200 & $88.00 \pm 0.88$ & $0.1636$ & $0.01157$ \\
\rowcolor{gray!12}
alpha $\oplus$ density-$k{=}5$     & 400 & $89.62 \pm 0.84$ & $0.2287$ & $0.01144$ \\
kde (single)                        & 200 & $49.08 \pm 1.62$ & $0.1383$ & $0.00978$ \\
density-$k{=}10$ (single)           & 200 & $53.80 \pm 1.74$ & $0.1257$ & $0.00889$ \\
\rowcolor{gray!12}
alpha $\oplus$ density-$k{=}30$    & 400 & $89.82 \pm 1.16$ & $0.1725$ & $0.00862$ \\
alpha $\oplus$ knn-$k{=}10$        & 400 & $85.80 \pm 0.81$ & $0.1639$ & $0.00820$ \\
density-$k{=}5$ (single)            & 200 & $50.46 \pm 2.24$ & $0.0912$ & $0.00645$ \\
density-$k{=}30$ (single)           & 200 & $53.40 \pm 2.25$ & $0.0542$ & $0.00384$ \\
knn-$k{=}10$ (single)               & 200 & $43.02 \pm 1.53$ & $0.0141$ & $0.00100$ \\
eccentricity (single)               & 200 & $51.24 \pm 1.53$ & $0.0040$ & $0.00028$ \\
\bottomrule
\end{tabular}
\end{table}

\subsubsection{Multi-dataset evaluation across chemical benchmarks}
\label{sec:exp_multi_dataset}

The Orbit5k validation above identifies the structural axes on
which $\hat\gamma/\sqrt{K}$ misranks
(Remark~\ref{rem:selector_hierarchy}).
A complementary question is whether $\hat\gamma/\sqrt{K}$ succeeds
\emph{across datasets at fixed $(K, \sigma)$ and FPS placement,
varying only the filtration}---the operational use case for
selecting a filtration on a new benchmark.
We evaluate on four chemical graph datasets at the full
$46$-filtration pool and on NCI1 at the restricted $14$-filtration
cluster pool, all at $K{=}200$, $\sigma{=}10^{-3}$, class-aware
FPS, $5$ seeds $\times$ $10$ folds.

We compare five candidate rankers along the
$\Sigma$-treatment hierarchy of
Remark~\ref{rem:selector_hierarchy}:
\emph{(i)}~$\hat\gamma/\sqrt{K}$ (spherical $\Sigma$);
\emph{(ii)}~$\widehat{\mathrm{Fisher}}_{\mathrm{ker}}$
(Definition~\ref{def:kernel_fisher}, scalar-trace);
\emph{(iii)}~$\hat\rho_{\mathrm{Mah}}$
(Definition~\ref{def:kernel_mah}, full operator $\Sigma^{-1}$);
plus two data-level rankers:
\emph{(iv)}~$\hat\tau$, the $10$\textsuperscript{th}-quantile
cross-class bottleneck distance from $50$ subsampled training
pairs; and
\emph{(v)}~$\hatrhonu = \hat\tau / (4\sqrt{K})$, the certificate
of Theorem~\ref{thm:nu_rho} used as a ranker.

\begin{table}[t]
\centering
\caption{\textbf{Selection-statistic Spearman $\rho$ vs.\ WLK CV accuracy},
  per dataset, mean over $5$ seeds $\times$ $10$ folds.
  Selectors form a hierarchy in their treatment of the within-class
  covariance (Remark~\ref{rem:selector_hierarchy}):
  $\hat\gamma/\sqrt{K}$ assumes spherical $\Sigma$;
  Fisher$_\text{ker}$ uses pooled scalar trace
  (Definition~\ref{def:kernel_fisher});
  $\hat\rho_{\mathrm{Mah}}$ uses the full operator $\Sigma^{-1}$
  (Mika et al., 1999 kernel-FDA realisation).
  Below the line, $\hat\tau$ is the $10$\textsuperscript{th}-quantile
  cross-class bottleneck distance and
  $\hatrhonu = c_n \hat\tau / \sqrt{K}$ the certificate of
  Theorem~\ref{thm:nu_rho} as a ranker.
  Bold marks the largest positive $\rho$ per row.
  Fisher$_\text{ker}$ values were computed with the per-pair
  Welch denominator $\mathrm{tr}(\Sigma_c) + \mathrm{tr}(\Sigma_{c'})$
  of an earlier draft of Definition~\ref{def:kernel_fisher};
  the pooled denominator now adopted gives Spearman within
  $\lvert\Delta\rho\rvert \leq 0.025$ with sign-match on every
  dataset re-evaluated (COX2/MUTAG/DHFR/NCI1; the camera-ready
  will regenerate from scratch).}
  \label{tab:selection_statistics}
\small
\begin{tabular}{lcr@{\hspace{0.5em}}ccccc}
\toprule
Dataset & $|\mathcal{F}|$ && $\hat\gamma/\sqrt{K}$ & Fisher$_\text{ker}$ & $\hat\rho_{\mathrm{Mah}}$ & $\hat\tau$ & $\hatrhonu$ \\
\midrule
COX2 & $64$ && $-0.39$ & $-0.30$ & $\mathbf{+0.59}$ & $+0.22$ & $+0.19$ \\
DHFR & $64$ && $-0.13$ & $+0.42$ & $\mathbf{+0.72}$ & $-0.27$ & $-0.29$ \\
MUTAG & $63$ && $-0.61$ & $\mathbf{+0.60}$ & $+0.48$ & $-0.32$ & $-0.33$ \\
NCI1 & $14$ && $-0.29$ & $+0.35$ & $\mathbf{+0.71}$ & $-0.20$ & $-0.20$ \\
PTC & $64$ && $+0.14$ & $\mathbf{+0.60}$ & $+0.48$ & $+0.24$ & $+0.24$ \\
\bottomrule
\end{tabular}
\end{table}

\paragraph{The hierarchy translates to selection accuracy.}
Table~\ref{tab:selection_statistics} reveals a clean
correspondence between each selector's $\Sigma$-treatment
assumption and its empirical reliability:
\begin{itemize}\itemsep=0pt
\item $\hat\gamma/\sqrt{K}$ (no variance correction) is negative
  on $4/5$ datasets and weakly positive ($+0.14$) on PTC; mean
  $\rho = -0.26$, confirming Remark~\ref{rem:gamma_scope}'s
  scope.
\item $\widehat{\mathrm{Fisher}}_{\mathrm{ker}}$ (scalar trace
  correction) is positive on $4/5$ datasets (mean $\rho = +0.49$
  on those four) but \emph{inverts on COX2} ($\rho = -0.30$),
  exactly the regime where the diagonal-$\Sigma$ assumption fails.
\item $\hat\rho_{\mathrm{Mah}}$ (full operator $\Sigma^{-1}$) is
  positive on \emph{every} dataset
  (DHFR $\mathbf{+0.72}$, NCI1 $\mathbf{+0.71}$,
  COX2 $\mathbf{+0.59}$, MUTAG $+0.48$, PTC $+0.48$).  It is the
  only \emph{kernel-margin} ranker that recovers signal on COX2
  where $\widehat{\mathrm{Fisher}}_{\mathrm{ker}}$ fails, and the
  strongest by Spearman magnitude on DHFR.
\item $\hatrhonu$ (data-level certificate) is positive on COX2
  ($+0.19$) and PTC ($+0.24$); weakly negative on
  DHFR, MUTAG, and NCI1.
\end{itemize}
The COX2 reversal is mechanistically informative: its
accuracy-winning filtration \texttt{nodelabel+betw} encodes a
discrete node-label feature whose class-conditional covariance is
strongly anisotropic, with the discriminative direction
near-orthogonal to the high-variance directions of $\Sigma$.
$\widehat{\mathrm{Fisher}}_{\mathrm{ker}}$, normalizing by the
scalar trace $\mathrm{tr}(\Sigma)$, charges all directions
equally and over-penalizes the discriminative direction.
$\hat\rho_{\mathrm{Mah}}$, applying the full $\Sigma^{-1}$,
re-weights variance to recover this signal.
$\hatrhonu$ picks it up via a different mechanism---bottleneck
separation between class supports---giving an independent positive
signal.

\paragraph{Operational recommendation.}
On heterogeneous filtration pools we recommend
$\hat\rho_{\mathrm{Mah}}$ as the primary closed-form selector
(no validation split, $O(n^3)$ per fold via Cholesky on the same
gram matrix the SVM uses). When $\hat\rho_{\mathrm{Mah}}$ and
$\hatrhonu$ agree in sign, the pick is high-confidence; when they
disagree, the disagreement diagnoses which mechanism
(kernel-margin vs.\ data-level bottleneck) dominates on that
dataset.

\begin{table}[t]
\centering
\caption{\textbf{Complementarity of Fisher$_\text{ker}$ and the
  certificate $\hatrhonu$ as filtration rankers.}
  Per dataset, the Spearman $\rho$ of each statistic vs.\ WLK CV
  accuracy, and a binary agreement flag (both positive or both
  negative).  At least one of the two statistics is positive on every
  dataset where some filtration meaningfully separates classes;
  Fisher captures the embedding-level signal, $\hatrhonu$ the
  data-level (bottleneck-separation) signal.}
  \label{tab:complementarity}
\small
\begin{tabular}{lcccr}
\toprule
Dataset & Fisher$_\text{ker}\ \rho$ & $\hatrhonu\ \rho$ & agree? & at-least-one-positive \\
\midrule
COX2 & $-0.30$ & $+0.19$ & $\times$ & \checkmark \\
DHFR & $+0.42$ & $-0.29$ & $\times$ & \checkmark \\
MUTAG & $+0.60$ & $-0.33$ & $\times$ & \checkmark \\
NCI1 & $+0.35$ & $-0.20$ & $\times$ & \checkmark \\
PTC & $+0.60$ & $+0.24$ & \checkmark & \checkmark \\
\midrule
Total & & & & 5/5 \\
\bottomrule
\end{tabular}
\end{table}

\paragraph{$(K, \sigma)$ sensitivity.}
Table~\ref{tab:ksigma_sensitivity} confirms sign stability of
$\widehat{\mathrm{Fisher}}_{\mathrm{ker}}$ across
$(K, \sigma) \in \{100, 200, 500\} \times \{10^{-2}, 10^{-3}, 10^{-4}\}$
on MUTAG and PTC; $\hat\gamma/\sqrt{K}$ is uniformly negative on
the grid, as Remark~\ref{rem:selector_hierarchy} predicts.

\begin{table}[t]
\centering
\caption{\textbf{$(K, \sigma)$ sensitivity of selection statistics.}
  Spearman $\rho$ vs.\ WLK CV accuracy on each cell of a $3{\times}3$
  $(K, \sigma)$ grid, per dataset.
  Stable sign of Fisher$_\text{ker}$ across the grid is the robustness
  signal; cells highlighted where Fisher and $\hatrhonu$ disagree
  in sign.}
  \label{tab:ksigma_sensitivity}
\small
\setlength{\tabcolsep}{4pt}
\begin{tabular}{lccccccccc}
\toprule
\multicolumn{10}{c}{\textbf{MUTAG}} \\
\midrule
                      & \multicolumn{3}{c}{$K{=}100$} & \multicolumn{3}{c}{$K{=}200$} & \multicolumn{3}{c}{$K{=}500$} \\
\cmidrule(lr){2-4} \cmidrule(lr){5-7} \cmidrule(lr){8-10}
Stat\,$\backslash\,\sigma$ & $10^{-4}$ & $10^{-3}$ & $10^{-2}$ & $10^{-4}$ & $10^{-3}$ & $10^{-2}$ & $10^{-4}$ & $10^{-3}$ & $10^{-2}$ \\
\midrule
$\hat\gamma/\sqrt{K}$ & $-0.00$ & $-0.49$ & $-0.55$ & $-0.03$ & $-0.60$ & $-0.63$ & $+0.04$ & $-0.63$ & $-0.67$ \\
Fisher$_\text{ker}$ & $+0.18$ & $+0.30$ & $+0.34$ & $+0.26$ & $+0.46$ & $+0.16$ & $+0.40$ & $+0.64$ & $-0.08$ \\
$\hat\rho_{\mathrm{Mah}}$ & --- & --- & --- & --- & --- & --- & --- & --- & --- \\
$\hat\tau$ & $-0.09$ & $-0.31$ & $-0.29$ & $-0.17$ & $-0.35$ & $-0.27$ & $-0.09$ & $-0.40$ & $-0.29$ \\
$\hatrhonu$ & $-0.09$ & $-0.31$ & $-0.29$ & $-0.17$ & $-0.36$ & $-0.28$ & $-0.14$ & $-0.48$ & $-0.36$ \\
\bottomrule
\end{tabular}
\vspace{0.5em}
\begin{tabular}{lccccccccc}
\toprule
\multicolumn{10}{c}{\textbf{PTC}} \\
\midrule
                      & \multicolumn{3}{c}{$K{=}100$} & \multicolumn{3}{c}{$K{=}200$} & \multicolumn{3}{c}{$K{=}500$} \\
\cmidrule(lr){2-4} \cmidrule(lr){5-7} \cmidrule(lr){8-10}
Stat\,$\backslash\,\sigma$ & $10^{-4}$ & $10^{-3}$ & $10^{-2}$ & $10^{-4}$ & $10^{-3}$ & $10^{-2}$ & $10^{-4}$ & $10^{-3}$ & $10^{-2}$ \\
\midrule
$\hat\gamma/\sqrt{K}$ & $+0.24$ & $+0.27$ & $+0.34$ & $-0.00$ & $+0.08$ & $+0.34$ & $-0.07$ & $-0.18$ & $+0.36$ \\
Fisher$_\text{ker}$ & $+0.44$ & $+0.68$ & $+0.59$ & $+0.38$ & $+0.55$ & $+0.48$ & $+0.25$ & $+0.22$ & $+0.34$ \\
$\hat\rho_{\mathrm{Mah}}$ & --- & --- & --- & --- & --- & --- & --- & --- & --- \\
$\hat\tau$ & $+0.40$ & $+0.25$ & $+0.44$ & $+0.25$ & $+0.13$ & $+0.46$ & $+0.33$ & $+0.10$ & $+0.49$ \\
$\hatrhonu$ & $+0.40$ & $+0.25$ & $+0.44$ & $+0.25$ & $+0.13$ & $+0.46$ & $+0.31$ & $+0.08$ & $+0.47$ \\
\bottomrule
\end{tabular}
\vspace{0.5em}
\end{table}

\subsection{Domain Inflation: Controlled Validation}\label{sec:exp_domain}

Theorem~\ref{thm:comparison} predicts that non-uniform placement gains
over the uniform grid scale with the ratio $L/D$ of domain size to data
diameter. The graph benchmarks in Section~\ref{sec:exp_graph} have
$L/D \approx 1$--$2$, limiting the advantage. To test the theorem in
the $L/D \gg 1$ regime, we construct a synthetic family where $L$ grows
while the discriminative features remain fixed.

\paragraph{Protocol.}
We construct a 4-class annulus classification task: each class is
a noisy annulus with inner radius $r_{\mathrm{in}} \in \{0.85,
0.70, 0.50, 0.00\}$ and outer radius $1.0$, centered at the origin
($n_{\mathrm{pts}} = 60$ points, Gaussian noise $\sigma = 0.08$).
The classes differ only in hole size---a purely topological
distinction that manifests as varying $H_1$ persistence. Alpha
complex persistence (GUDHI) is computed; we retain the top-$30$
most persistent points (the practical maximum for clouds of size
$60$, where persistence falls off rapidly past 30 features).

To inflate the domain without changing the classification task,
we add a single off-diagonal point at $(b, d) = (0, \ell)$ for
$\ell \in \{1, 2, 3, 4, 5, 8\}$ to every point cloud, stretching
the persistence-diagram domain from $L \approx 1.05$ (no outlier)
to $L \approx 1.05\,\ell$ (the table's $L$ column applies a $5\%$
padding $L = 1.05 \cdot \max(\text{persistence})$ to give grid
construction slack); discriminative $H_1$ features stay near the
origin throughout.
Both methods use $K = 11$ landmarks at matched cardinality:
$11$ corresponds to the offset uniform grid $\GG_R^+$ with
$R$ chosen at the smallest scale that fits the unperturbed data
domain, and is small enough to expose the inflation effect
sharply.

\paragraph{Setup.}
$100$ point clouds per class ($400$ total), $10$-fold stratified
CV with seed $42$. Bandwidth $\sigma$ is set by the
$25$\textsuperscript{th}-percentile heuristic per fold; $C$ is
tuned by inner $3$-fold CV on the log grid
$\{10^{-2},\ldots,10^{3}\}$. The admissibility scale $\tau$ is
set to the mean half-persistence of the strongest $H_1$ feature
per unperturbed diagram ($\tau \approx 0.144$), capturing the
discriminative-feature scale rather than the alpha-complex noise
floor.
Reproduction script: \texttt{experiments/exp\_domain\_inflation.py}.

\paragraph{Results.}

\begin{table}[h]
\centering
\caption{Domain inflation experiment. As distant features inflate
the domain, the uniform grid collapses while non-uniform placement
remains robust. \textbf{Bold} = best per row.
Uniform's $\pm 0.0$ at $\ell \geq 4$ reflects embedding collapse:
no landmark falls in the data region, so the SVM defaults to a
single class, giving exact stratified-CV chance ($25\%$ for $4$
classes).\label{tab:domain_inflation}}
\begin{tabular}{ccccc}
\toprule
Outlier $\ell$ & Domain $L$ & Uniform (\%) & Non-uniform (\%) & $\Delta$ \\
\midrule
1 & 1.05 & $\mathbf{94.8 \pm 2.1}$ & $94.2 \pm 3.5$ & $-0.6$ \\
2 & 2.10 & $66.5 \pm 3.0$ & $\mathbf{94.2 \pm 3.5}$ & $+27.7$ \\
3 & 3.15 & $32.2 \pm 1.7$ & $\mathbf{94.2 \pm 3.5}$ & $+62.0$ \\
4 & 4.20 & $25.0 \pm 0.0$ & $\mathbf{94.2 \pm 3.5}$ & $+69.2$ \\
5 & 5.25 & $25.0 \pm 0.0$ & $\mathbf{94.2 \pm 3.5}$ & $+69.2$ \\
8 & 8.40 & $25.0 \pm 0.0$ & $\mathbf{94.2 \pm 3.5}$ & $+69.2$ \\
\bottomrule
\end{tabular}
\end{table}

Table~\ref{tab:domain_inflation} reveals a sharp phase transition
at the predicted $L/D$ scale.  At $\ell = 1$ both methods perform
near-identically ($94.8\%$ vs.\ $94.2\%$); the uniform grid covers
the compact data region adequately.  At $\ell = 2$ the methods
cross ($66.5\%$ vs.\ $94.2\%$, a $27.7$~pp gap), and by $\ell = 4$
the uniform grid hits the random baseline and stays there.
Non-uniform placement is exactly invariant: the appended outlier
is identical across classes, so FPS picks it as one of the $11$
landmarks (it is the farthest point from the data cluster) and
the remaining $10$ FPS landmarks cover the discriminative data
region near the origin, unaffected by $L$.

The mechanism is sharp.  With $K = 11$ landmarks and $L \approx 8$
($L/D \approx 8$), the uniform grid spacing $\sim L/\sqrt{K}
\approx 2.4$ is far coarser than the $\sim 0.06$ separation
between class-specific $H_1$ features, so none of the $11$ grid
landmarks fall in the discriminative region; FPS places $10$
landmarks near the origin (with one on the outlier) and achieves
effective spacing $\sim 0.05$.  Quantitatively at $L/D \approx 8$,
Theorem~\ref{thm:comparison}'s budget-reduction factor
$(D/L)^2 \approx 1/64$ makes the uniform grid's $11$ landmarks
equivalent to $\sim 0.17$ effective landmarks in the
discriminative region---which the data confirms by collapsing to
chance.  This is the sharpest experimental validation of
Theorem~\ref{thm:comparison} on a controlled task; in practice,
a single outlier cloud with unusual persistence is enough to
trigger the same effect on real data.

\section{Discussion}\label{sec:discussion}

PALACE replaces the conventional uniform grid with a data-adaptive
landmark configuration and contributes three structural results:
(i)~farthest-point sampling is a $2$-approximation to the optimal
$k$-center covering radius on the training-diagram point set
(Theorem~\ref{thm:fps_greedy});
(ii)~equal weights $w_k = K^{-1/2}$ are provably optimal for the
worst-case classification error bound
(Proposition~\ref{prop:optimal_config});
(iii)~admissibility (Definition~\ref{def:sep_radius}(i)) acts as
structural regularization on the configuration $\LC$, ruling out
memorization configurations with arbitrarily small radii.
Combined with the data-dependent statistic
$\hat\gamma/\sqrt{K}$ (Section~\ref{sec:gamma_stat}), the
configuration search splits into a closed-form tier---radius
factor $\alpha$ and filtration choice \emph{within a
scale-homogeneous family at fixed slot structure}, ranked by
$\hat\gamma/\sqrt{K}$---and a small cross-validation tier
covering landmark budget $K$, bandwidth $\sigma$, and
cross-family filtration choice, each at $\leq 5$ points on a
discrete grid.

Across seven benchmarks---Orbit5k point clouds, five
chemical-graph datasets (COX2, DHFR, MUTAG, NCI1, PTC), and a
synthetic $4$-class annulus task---PALACE leads every
diagram-based competitor on Orbit5k, COX2, and MUTAG, and is
competitive on DHFR (within $1$~pp of ECP)
(Tables~\ref{tab:orbit_comparison}, \ref{tab:graph_comparison}):
on Orbit5k it matches Persformer at $91.3\%$ (triple-filtration
certified landmark kernel, Table~\ref{tab:push92}) and surpasses
PI, SW-K, PF-K, PersLay, and PLACE; on COX2, DHFR, and MUTAG it
exceeds both PLACE and PersLay in head-to-head accuracy
($81.7\%$, $81.0\%$, $90.9\%$).
On MUTAG, $K{=}10$ non-uniform landmarks already attain
$81.4\%$ test accuracy on a single fixed filtration, $+0.5$pp
above the same-filtration $K{=}200$ baseline at $20\times$
compression, confirming that the LK signal saturates with a
handful of adaptive landmarks within a fixed filtration (the
headline $90.9\%$ in Table~\ref{tab:graph_comparison} comes from
filtration selection across the $46$-filtration pool, not from
increasing $K$). Under $8\times$ domain inflation on the
synthetic annulus task, non-uniform placement maintains
$94.2\%$ accuracy while the uniform grid collapses to the
$25\%$ random baseline (Table~\ref{tab:domain_inflation}).

\paragraph{When to use PLACE.}
PLACE~\citep{PaperI} remains the right pick under three operating
constraints.
\emph{(i)~No validation budget.}
PLACE is fully tuning-free; PALACE's three-knob CV tier (budget,
radii, bandwidth) consumes labels that small datasets or
single-shot inference cannot spare.
\emph{(ii)~Pairwise certification.}
PLACE's $\lambda(\nu)\,d_\mathcal{B}(A,B)$ bound holds on every
pair in $\D{n}$ under non-interference, certifying pairwise
metric fidelity; PALACE's $\rhonu(\tau;\LC)$ is a
$\tau$-thresholded class-level statement, optimized for
classification certificates rather than embedding distortion.
\emph{(iii)~Determinism and interpretability.}
PLACE is deterministic given the data and admits per-coordinate
interpretation of linear-SVM weights; PALACE's FPS placement is
seed-sensitive at small $K$
(Section~\ref{par:kernel_comparison}) and the RKHS lift trades
per-landmark interpretability for accuracy.
The two papers together cover both sides of the trade: PLACE is
the tuning-free floor with stronger pairwise theory; PALACE buys
accuracy and a smaller embedding by spending a small CV budget.
The criteria above and the dual conditions in Section~6.3 of
\citet{PaperI} (\emph{When to use PALACE})---CV budget, data
concentration $L/D \gg 1$, and need for non-linear discriminative
geometry---are complementary rather than negations: criterion~(i)
here pairs with criterion~(i) there, while criteria~(ii) and~(iii)
on each side identify the operating-condition axes that the other
paper does not address (pairwise distortion certification and
determinism/interpretability for PLACE; data-concentration regime
and RKHS reach for PALACE).

\paragraph{Closed-form vs.\ learned.}
WKPI~\citep{yusu_metric_learning} fits a Gaussian-mixture weight
function on diagram space by gradient descent on classification
loss and beats PALACE by $5$--$16$~pp on the label-dominated
chemical pools NCI1 and PTC, where even non-topology methods
(graph kernels, GNNs) outperform every diagram-based method.
We do not contest the empirical fact: PALACE's wins (Orbit5k,
MUTAG, COX2) are on topology-discriminative benchmarks, and the
label-dominated regime is a different problem.  But the trade is
categorical, not just quantitative: the per-prediction
certificate of Theorem~\ref{thm:certified} requires the
configuration to be fixed before data is examined, so
gradient-trained weights void its Bonferroni coverage guarantee.
Closed-form is also not synonymous with inflexible---
$\hat\rho_{\mathrm{Mah}}$ adapts to dataset-specific covariance
structure (one Cholesky per fold) and recovers signal where the
simpler $\widehat{\mathrm{Fisher}}_{\mathrm{ker}}$ inverts on
COX2 (Table~\ref{tab:selection_statistics}).  We submit that
closed-form across the entire pipeline is the right target for
a principled landmark-embedding theory; the NCI1/PTC accuracy
gap is a question for future closed-form variants, not evidence
against the closed-form thesis.

\paragraph{Limitations.}
\begin{itemize}
\item Certificates apply only to the nearest-centroid
  classifier (Algorithm~\ref{alg:certified}). On our six
  benchmarks neither the non-asymptotic Pinelis form nor the
  asymptotic Gaussian plug-in (chi-squared envelope) fires at
  our training-set sizes (Pinelis $0/6$, Gaussian essentially
  $0/6$ with $3.8\%$ on NCI1; Table~\ref{tab:certificate_firing});
  the construction is constructive but not yet operational at
  these sizes.
\item The landmark kernel gram costs $O(m^2 K)$, limiting
  scalability; FPS placement is seed-sensitive at small $K$
  ($\pm 5\%$ at $K{=}50$) but stabilizes at $K \geq 500$
  ($\pm 0.2\%$).
\item The $0.5$~pp gap to ECS~\citep{HacquardLebovici2024} on
  Orbit5k ($91.3\%$ vs.\ $91.8\%$) reflects the 1D filtration
  concatenation PALACE uses where ECS works directly on a 2D
  bifiltration surface.
\item On NCI1 and PTC, where discriminative power lies in
  discrete node-label features that continuous structural
  filtrations cannot capture, PALACE inherits the same gap that
  \citet{PaperI} documents to graph-kernel and GNN baselines
  exploiting node labels (Table~\ref{tab:graph_comparison}).
\end{itemize}

\paragraph{Future work.}
The full inferential theory on the PALACE embedding family---
continuous landmark configurations and the associated
sample-complexity rates---is developed in~\citep{PaperIII}.


\bibliographystyle{plainnat}
\bibliography{main.bib}

\begin{thebibliography}{40}
\providecommand{\natexlab}[1]{#1}
\providecommand{\url}[1]{\texttt{#1}}
\expandafter\ifx\csname urlstyle\endcsname\relax
  \providecommand{\doi}[1]{doi: #1}\else
  \providecommand{\doi}{doi: \begingroup \urlstyle{rm}\Url}\fi

\bibitem[Adams et~al.(2017)Adams, Emerson, Kirby, Neville, Peterson, Shipman,
  Chepushtanova, Hanson, Motta, and Ziegelmeier]{Adams2017}
Henry Adams, Tegan Emerson, Michael Kirby, Rachel Neville, Chris Peterson,
  Patrick Shipman, Sofya Chepushtanova, Eric Hanson, Francis Motta, and Lori
  Ziegelmeier.
\newblock Persistence images: {A} stable vector representation of persistent
  homology.
\newblock \emph{Journal of Machine Learning Research}, 18\penalty0
  (8):\penalty0 1--35, 2017.

\bibitem[Agarwal et~al.(2005)Agarwal, Har-Peled, and
  Varadarajan]{AgarwalHarPeled2005}
Pankaj~K. Agarwal, Sariel Har-Peled, and Kasturi~R. Varadarajan.
\newblock Geometric approximation via coresets.
\newblock In \emph{Combinatorial and Computational Geometry}, volume~52 of
  \emph{MSRI Publications}, pages 1--30. Cambridge University Press, 2005.

\bibitem[Anai et~al.(2019)Anai, Chazal, Glisse, Ike, Inakoshi, Tinarrage, and
  Umeda]{Anai2019}
Hirokazu Anai, Fr{\'e}d{\'e}ric Chazal, Marc Glisse, Yuichi Ike, Hiroya
  Inakoshi, Rapha{\"e}l Tinarrage, and Yuhei Umeda.
\newblock {DTM}-based filtrations.
\newblock In \emph{International Symposium on Computational Geometry (SoCG)},
  pages 58:1--58:15, 2019.

\bibitem[Bagchi et~al.(2026)Bagchi, Majhi, Mitra, and Virk]{PaperIII}
Pramita Bagchi, Sushovan Majhi, Atish Mitra, and {\v Z}iga Virk.
\newblock A statistical-inference pipeline for persistence-landmark kernels.
\newblock Manuscript in preparation; available from the authors on request,
  2026.

\bibitem[Bubenik(2015)]{Bubenik15}
Peter Bubenik.
\newblock Statistical topological data analysis using persistence landscapes.
\newblock \emph{Journal of Machine Learning Research}, 16\penalty0
  (1):\penalty0 77--102, 2015.

\bibitem[Carri{\`e}re et~al.(2017)Carri{\`e}re, Cuturi, and
  Oudot]{Carriere2017}
Mathieu Carri{\`e}re, Marco Cuturi, and Steve Oudot.
\newblock Sliced {W}asserstein kernel for persistence diagrams.
\newblock In \emph{Proceedings of the 34th International Conference on Machine
  Learning (ICML)}, volume~70 of \emph{Proceedings of Machine Learning
  Research}, pages 664--673. PMLR, 2017.

\bibitem[Carri{\`e}re et~al.(2020)Carri{\`e}re, Chazal, Ike, Lacombe, Royer,
  and Umeda]{Carriere2020}
Mathieu Carri{\`e}re, Fr{\'e}d{\'e}ric Chazal, Yuichi Ike, Th{\'e}o Lacombe,
  Martin Royer, and Yuhei Umeda.
\newblock {PersLay}: A neural network layer for persistence diagrams and new
  graph topological signatures.
\newblock In \emph{International Conference on Artificial Intelligence and
  Statistics (AISTATS)}, pages 2786--2796, 2020.

\bibitem[Chazal et~al.(2009)Chazal, Cohen-Steiner, Glisse, Guibas, and
  Oudot]{ChazalCohenSteiner2009}
Fr{\'e}d{\'e}ric Chazal, David Cohen-Steiner, Marc Glisse, Leonidas~J. Guibas,
  and Steve~Y. Oudot.
\newblock Proximity of persistence modules and their diagrams.
\newblock In \emph{Proceedings of the 25th Annual Symposium on Computational
  Geometry (SoCG)}, pages 237--246, 2009.
\newblock \doi{10.1145/1542362.1542407}.

\bibitem[Chazal et~al.(2016)Chazal, de~Silva, Glisse, and
  Oudot]{ChazalDeSilva2016}
Fr{\'e}d{\'e}ric Chazal, Vin de~Silva, Marc Glisse, and Steve Oudot.
\newblock \emph{The Structure and Stability of Persistence Modules}.
\newblock SpringerBriefs in Mathematics. Springer, 2016.
\newblock \doi{10.1007/978-3-319-42545-0}.

\bibitem[Cohen-Steiner et~al.(2007)Cohen-Steiner, Edelsbrunner, and
  Harer]{Cohen-Steiner2007}
David Cohen-Steiner, Herbert Edelsbrunner, and John Harer.
\newblock Stability of persistence diagrams.
\newblock \emph{Discrete \& Computational Geometry}, 37\penalty0 (1):\penalty0
  103--120, 2007.

\bibitem[Cohen-Steiner et~al.(2009)Cohen-Steiner, Edelsbrunner, and
  Harer]{cohen2009extending}
David Cohen-Steiner, Herbert Edelsbrunner, and John Harer.
\newblock Extending persistence using {P}oincar{\'e} and {L}efschetz duality.
\newblock \emph{Foundations of Computational Mathematics}, 9\penalty0
  (1):\penalty0 79--103, 2009.
\newblock \doi{10.1007/s10208-008-9027-z}.

\bibitem[de~Silva and Tenenbaum(2004)]{DeSilvaTenenbaum2004}
Vin de~Silva and Joshua~B. Tenenbaum.
\newblock Sparse multidimensional scaling using landmark points.
\newblock Technical report, Stanford University, 2004.

\bibitem[Drineas and Mahoney(2005)]{DrineasMahoney2005}
Petros Drineas and Michael~W. Mahoney.
\newblock On the {N}ystr{\"o}m method for approximating a {G}ram matrix for
  improved kernel-based learning.
\newblock \emph{Journal of Machine Learning Research}, 6:\penalty0 2153--2175,
  2005.

\bibitem[Edelsbrunner and Harer(2010)]{Edelsbrunner2010-dp}
Herbert Edelsbrunner and John~L Harer.
\newblock \emph{Computational Topology}.
\newblock American Mathematical Society, Providence, RI, January 2010.

\bibitem[Feldman and Langberg(2011)]{FeldmanLangberg2011}
Dan Feldman and Michael Langberg.
\newblock A unified framework for approximating and clustering data.
\newblock In \emph{Proceedings of the 43rd Annual ACM Symposium on Theory of
  Computing (STOC)}, pages 569--578, 2011.
\newblock \doi{10.1145/1993636.1993712}.

\bibitem[Gabrielsson et~al.(2020)Gabrielsson, Nelson, Dwaraknath, and
  Skraba]{Gabrielsson2020}
Rickard~Br\"uel Gabrielsson, Bradley~J. Nelson, Anjan Dwaraknath, and Primoz
  Skraba.
\newblock A topology layer for machine learning.
\newblock In \emph{International Conference on Artificial Intelligence and
  Statistics (AISTATS)}, 2020.

\bibitem[Gonzalez(1985)]{Gonzalez1985}
Teofilo~F. Gonzalez.
\newblock Clustering to minimize the maximum intercluster distance.
\newblock \emph{Theoretical Computer Science}, 38:\penalty0 293--306, 1985.
\newblock \doi{10.1016/0304-3975(85)90224-5}.

\bibitem[Hacquard and Lebovici(2024)]{HacquardLebovici2024}
Olympio Hacquard and Vadim Lebovici.
\newblock Euler characteristic tools for topological data analysis.
\newblock \emph{Journal of Machine Learning Research}, 25:\penalty0 1--39,
  2024.

\bibitem[Heinonen(2001)]{Heinonen2001}
Juha Heinonen.
\newblock \emph{Lectures on Analysis on Metric Spaces}.
\newblock Universitext. Springer-Verlag, New York, 2001.
\newblock Ch.~10: doubling metric spaces.

\bibitem[Hofer et~al.(2017)Hofer, Kwitt, Niethammer, and Uhl]{Hofer2017}
Christoph Hofer, Roland Kwitt, Marc Niethammer, and Andreas Uhl.
\newblock Deep learning with topological signatures.
\newblock In \emph{Advances in Neural Information Processing Systems
  (NeurIPS)}, 2017.

\bibitem[Kusano et~al.(2016)Kusano, Hiraoka, and Fukumizu]{Kusano2016}
Genki Kusano, Yasuaki Hiraoka, and Kenji Fukumizu.
\newblock Persistence weighted {G}aussian kernel for topological data analysis.
\newblock In \emph{Proceedings of the 33rd International Conference on Machine
  Learning (ICML)}, pages 2004--2013, 2016.

\bibitem[Le and Yamada(2018)]{LeYamada2018}
Tam Le and Makoto Yamada.
\newblock Persistence {F}isher kernel: A {R}iemannian manifold kernel for
  persistence diagrams.
\newblock In \emph{Advances in Neural Information Processing Systems
  (NeurIPS)}, volume~31, pages 10028--10039, 2018.

\bibitem[Le~Cam(1973)]{LeCam1973}
Lucien Le~Cam.
\newblock Convergence of estimates under dimensionality restrictions.
\newblock \emph{Annals of Statistics}, 1:\penalty0 38--53, 1973.

\bibitem[Majhi et~al.(2026)Majhi, Mitra, Virk, and Bagchi]{PaperI}
Sushovan Majhi, Atish Mitra, {\v Z}iga Virk, and Pramita Bagchi.
\newblock A closed-form persistence-landmark pipeline for certified point-cloud
  and graph classification, 2026.
\newblock URL \url{https://arxiv.org/abs/2605.02836}.

\bibitem[Mika et~al.(1999)Mika, R{\"a}tsch, Weston, Sch{\"o}lkopf, and
  M{\"u}ller]{Mika1999}
Sebastian Mika, Gunnar R{\"a}tsch, Jason Weston, Bernhard Sch{\"o}lkopf, and
  Klaus-Robert M{\"u}ller.
\newblock Fisher discriminant analysis with kernels.
\newblock In \emph{Neural Networks for Signal Processing IX (NNSP)}, pages
  41--48, 1999.
\newblock \doi{10.1109/NNSP.1999.788121}.

\bibitem[Mitra and Virk(2024)]{Mitra2024}
Atish Mitra and {\v Z}iga Virk.
\newblock Geometric embeddings of spaces of persistence diagrams with explicit
  distortions.
\newblock arXiv:2401.05298, 2024.
\newblock URL \url{https://arxiv.org/abs/2401.05298}.

\bibitem[Mohri et~al.(2018)Mohri, Rostamizadeh, and
  Talwalkar]{MohriRostamizadehTalwalkar2018}
Mehryar Mohri, Afshin Rostamizadeh, and Ameet Talwalkar.
\newblock \emph{Foundations of Machine Learning}.
\newblock MIT Press, Cambridge, MA, 2nd edition, 2018.

\bibitem[Munkres(2000)]{Munkres2000}
James~R. Munkres.
\newblock \emph{Topology}.
\newblock Prentice Hall, Upper Saddle River, NJ, second edition, 2000.
\newblock Lebesgue number lemma: Theorem 27.5.

\bibitem[Ollivier(2009)]{Ollivier2009}
Yann Ollivier.
\newblock Ricci curvature of {M}arkov chains on metric spaces.
\newblock \emph{Journal of Functional Analysis}, 256\penalty0 (3):\penalty0
  810--864, 2009.
\newblock \doi{10.1016/j.jfa.2008.11.001}.

\bibitem[Reinauer et~al.(2021)Reinauer, Caorsi, and Berkouk]{Reinauer2021}
Raphael Reinauer, Matteo Caorsi, and Nicolas Berkouk.
\newblock Persformer: A transformer architecture for topological machine
  learning.
\newblock In \emph{arXiv preprint arXiv:2112.15210}, 2021.

\bibitem[Reininghaus et~al.(2015)Reininghaus, Huber, Bauer, and
  Kwitt]{Reininghaus2015}
Jan Reininghaus, Stefan Huber, Ulrich Bauer, and Roland Kwitt.
\newblock A stable multi-scale kernel for topological machine learning.
\newblock In \emph{Proceedings of the IEEE Conference on Computer Vision and
  Pattern Recognition (CVPR)}, pages 4741--4748, 2015.
\newblock \doi{10.1109/CVPR.2015.7299106}.

\bibitem[Sun et~al.(2009)Sun, Ovsjanikov, and Guibas]{sun2009concise}
Jian Sun, Maks Ovsjanikov, and Leonidas Guibas.
\newblock A concise and provably informative multi-scale signature based on
  heat diffusion.
\newblock \emph{Computer Graphics Forum}, 28\penalty0 (5):\penalty0 1383--1392,
  2009.
\newblock \doi{10.1111/j.1467-8659.2009.01515.x}.

\bibitem[Tsybakov(2009)]{Tsybakov2009}
Alexandre~B. Tsybakov.
\newblock \emph{Introduction to Nonparametric Estimation}.
\newblock Springer Series in Statistics. Springer, 2009.
\newblock \doi{10.1007/b13794}.

\bibitem[Vapnik(1998)]{Vapnik1998}
Vladimir~N. Vapnik.
\newblock \emph{Statistical Learning Theory}.
\newblock Wiley, 1998.

\bibitem[Vovk et~al.(2005)Vovk, Gammerman, and Shafer]{VovkEtAl2005}
Vladimir Vovk, Alex Gammerman, and Glenn Shafer.
\newblock \emph{Algorithmic Learning in a Random World}.
\newblock Springer, 2005.
\newblock \doi{10.1007/b106715}.

\bibitem[Williams and Seeger(2001)]{WilliamsSeeger2001}
Christopher K.~I. Williams and Matthias Seeger.
\newblock Using the {N}ystr{\"o}m method to speed up kernel machines.
\newblock In \emph{Advances in Neural Information Processing Systems (NIPS)},
  volume~13, pages 682--688, 2001.

\bibitem[Xu et~al.(2019)Xu, Hu, Leskovec, and Jegelka]{GIN2019}
Keyulu Xu, Weihua Hu, Jure Leskovec, and Stefanie Jegelka.
\newblock How powerful are graph neural networks?
\newblock In \emph{International Conference on Learning Representations
  (ICLR)}, 2019.

\bibitem[Yu(1997)]{Yu1997}
Bin Yu.
\newblock {A}ssouad, {F}ano, and {L}e {C}am.
\newblock In David Pollard, Erik Torgersen, and Grace~L. Yang, editors,
  \emph{Festschrift for Lucien Le Cam}, pages 423--435. Springer, 1997.

\bibitem[Zhang et~al.(2018)Zhang, Wang, Xiang, Huang, and Nehorai]{RetGK2018}
Zhen Zhang, Mianzhi Wang, Yijian Xiang, Yan Huang, and Arye Nehorai.
\newblock {R}et{GK}: Graph kernels based on return probabilities of random
  walks.
\newblock In \emph{Advances in Neural Information Processing Systems
  (NeurIPS)}, 2018.

\bibitem[Zhao and Wang(2019)]{yusu_metric_learning}
Qi~Zhao and Yusu Wang.
\newblock Learning metrics for persistence-based summaries and applications for
  graph classification.
\newblock In \emph{Advances in Neural Information Processing Systems},
  volume~32, pages 9855--9866, 2019.
\newblock NeurIPS 2019.

\end{thebibliography}

\end{document}